\documentclass[11pt]{article}
\usepackage{fullpage}
\usepackage{amsfonts}
\usepackage{amssymb}
\usepackage{fancyhdr}
\usepackage{comment}
\usepackage{soul}

\usepackage[margin=1in]{geometry}  % set the margins to 1in on all sides
\usepackage{amsmath}
\usepackage{amsthm}
\usepackage{dsfont}
\usepackage{graphicx}
\usepackage{enumerate}
\usepackage[hypcap]{caption}
\usepackage{color, soul}
\usepackage{float}
\usepackage[ruled]{algorithm2e}
\usepackage[toc,page]{appendix}
\newtheorem{theorem}{Theorem}

\newtheorem{lemma}{Lemma}

\newcommand{\naturals}{{\mathbb{N}}}

\newcommand{\Ld}{\gamma}

\hypersetup{%dvips,
            CJKbookmarks=true,
            bookmarksnumbered=true,
            bookmarksopen=true,
            colorlinks=true,
            citecolor=red,
            linkcolor=blue,
            anchorcolor=red,
            urlcolor=blue,
            }

\title{Optimal query complexity for private sequential learning against eavesdropping}

\author{Jiaming Xu, Kuang Xu, and Dana Yang\thanks{
J.\ Xu and D.\ Yang are with The Fuqua School of Business, Duke University, Durham NC, USA, \texttt{\{jx77,xiaoqian.yang\}@duke.edu}.
K.\ Xu is with the Stanford Graduate School of Business, Stanford University, Stanford CA, USA, \texttt{kuangxu@stanford.edu}.
This research is supported by the NSF Grants IIS-1838124, CCF-1850743, and CCF-1856424.
}}

\date{\today}

\begin{document}

\maketitle

\begin{abstract}
We study the query complexity of a learner-private sequential learning problem, motivated by the privacy and security concerns due to eavesdropping that arise in practical applications such as pricing and Federated Learning. A learner 
tries to estimate an unknown scalar value, by sequentially
querying an external database and receiving binary responses; 
meanwhile, a third-party adversary observes 
the learner's queries but not the responses. The learner's goal
is to design a querying strategy with the minimum number of queries (optimal query complexity)
so that she can accurately estimate the true value, while the eavesdropping adversary even
with the complete knowledge of her querying strategy cannot.  

We develop new querying strategies  and analytical techniques
and use them to prove tight upper and lower bounds on the  optimal query complexity. %The bounds match up to an additive constant \nbr{JX. In Bayesian setting, 
%we still have an additive gap about $4L$} across nearly the entire parameter range,
The bounds almost match across the entire parameter range,
 substantially improving upon existing results. We thus obtain a complete picture of the optimal query complexity as a function of the 
estimation accuracy and the desired levels of privacy. 
We also extend the results to sequential learning models in higher dimensions,
and where the binary responses are noisy. Our analysis leverages a  crucial insight into the nature of private learning problem, which suggests that the query trajectory of an optimal learner can be divided into distinct phases that focus on pure learning versus learning and obfuscation, respectively.  
%\nbr{JX. This insight seems to only apply in Bayesian setting. For Deterministic setting, the optimal strategy works differently. }
\end{abstract}

\section{Introduction}

Rapid developments in machine learning and data science have compelled organizations and individuals to increasingly rely on data to solve inference and decision problems. It quickly became clear, however, that collecting and disseminating data in bulk can expose data owners to serious privacy breaches \cite{Dwork08}. To address the privacy concerns of data owners, researchers and practitioners have been advocating a new learning framework, known as \emph{learning with external workers} \cite{konevcny2015federated}. %\nbr{JX. I am not sure if this is the right reference. I double checked the paper and couldn't find the paper proposes this new framework.
%If we are speaking about Federated Learning, I think the earliest reference who proposes it is \cite{konevcny2015federated}. Please double check this.} 
Under this framework, instead of allowing a learner to possess the entire data set and conduct analysis in an offline manner, data sets are kept secure by their owners, and the learner must interact with data owners by submitting queries and receiving responses.

While substantial progress has been achieved in protecting data owners' privacy in such systems \cite{Dwork08, geyer2017differentially, song2013stochastic, agarwal2018cpsgd}, the   \emph{learner's} privacy has largely been overlooked. Because a learner has to communicate frequently with data owners in order to perform analysis, their queries can be subject to eavesdropping by a third-party adversary. That adversary, in turn, could use the observed queries to reconstruct the learned model, thus allowing them to free-ride at the learner's expense, or worse, leverage such information in future sabotages. 

In this paper, we focus on understanding how to protect the learner's privacy against eavesdropping attacks, and precisely quantifying the fundamental privacy-complexity trade-offs in such an interactive learning system.  We base our analysis on the Private Sequential Learning model proposed by~\cite{tsitsiklis2018private}. Suppose that a learner is trying to estimate an unknown target value $X^* \in [0,1]$, by submitting 
$n$ queries sequentially, $(q_1, \ldots, q_n) \in [0,1]^n$, for some $n \in \naturals$.
For each query $q_i$, the learner receives a binary response $r_i=\mathds{1}\{X^*\geq q_i\}$, indicating the position of $X^*$ relative to the query, where  $\mathds{1}\{\cdot\}$ denotes the indicator function. Meanwhile, there is an adversary who observes all of the learner's queries
$(q_1, \ldots, q_n)$, but not the responses $(r_1, \ldots, r_n)$. The adversary then tries to estimate $X^*$. The learner's goal is to design a querying strategy with a minimal $n$ (optimal query complexity) so that she can estimate $X^*$ up to an additive error of $\epsilon/2$ with probability $1$ (accuracy), while   \emph{no adversary} can estimate $X^*$ up to an additive error of $\delta/2$ with probability larger than $1/L$ 
for some integer $L\ge 2$ (privacy), even if they are equipped with the complete knowledge of the learner's querying strategy. The parameter $L$ thus captures the learner's privacy level. In the special case of $L=1$ (corresponding to having no privacy constraint) this learning model
reduces to the classical problem of sequential search with binary feedback, with numerous applications such as data transmission with feedback~\cite{horstein1963sequential}, finding the roots of a continuous function~\cite{waeber2011bayesian}, and even the game of ``twenty questions''~\cite{jedynak2012twenty}. 
\begin{figure}[h]
\centering
\includegraphics[scale=0.5]{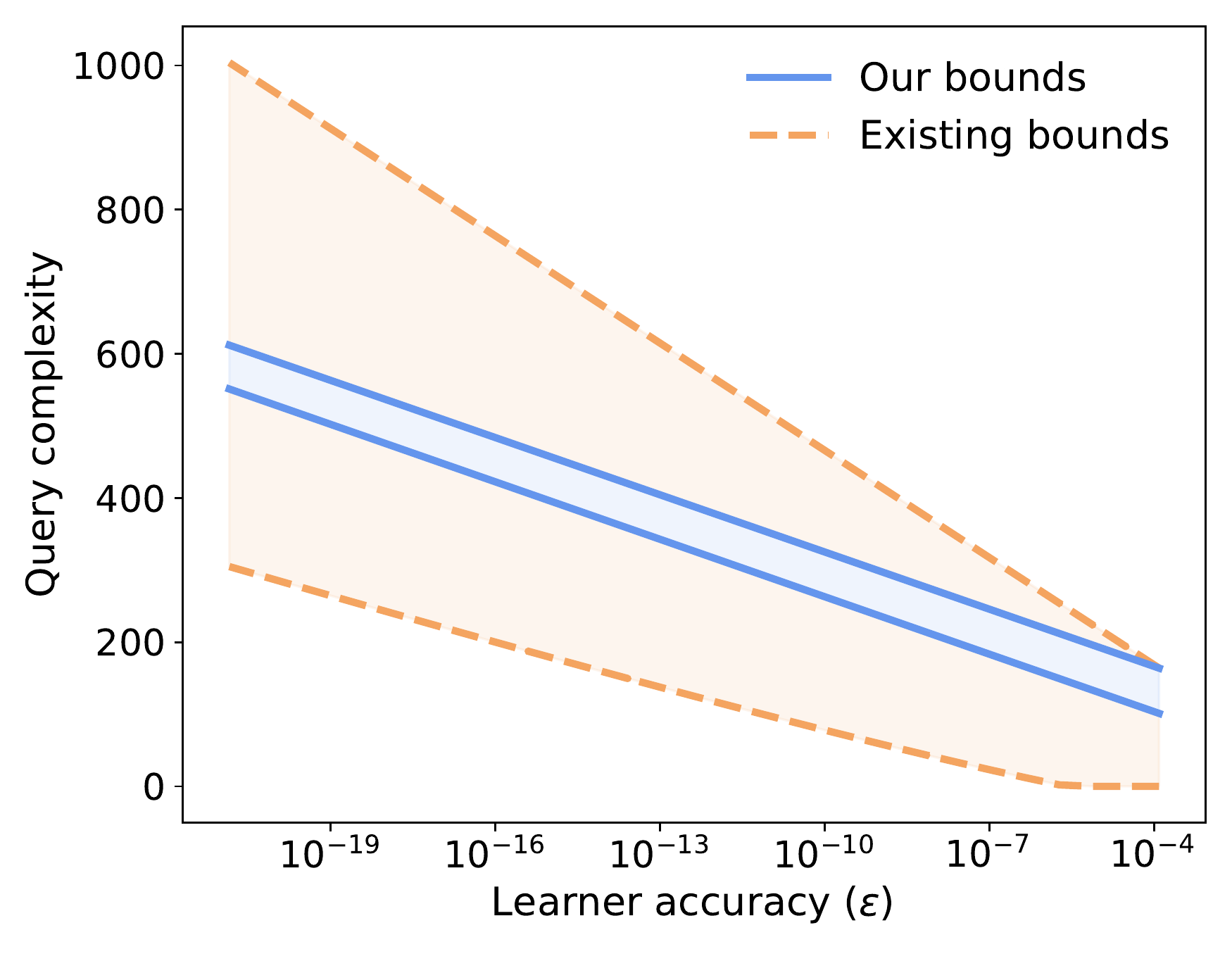}
\caption{Our results  in Theorem \ref{thm:bayes} (solid) versus the best known bounds (dashed, upper bound: \cite{tsitsiklis2018private}, lower bound: \cite{xu2018query}) under the noiseless Bayesian setting, with $L=15$ and $\delta= 4\epsilon^{0.5}$. The figure is cut off to the right at the point where $\delta$ hits the upper limit, $1/L$, beyond which it is easy to show that no learning strategy can be private.}
\label{fig:old_new_bayes}
\end{figure}

%\nbr{JX. Is the noiseless Bayesian setting defined? Another minor comment:
% I know Fig.~\ref{fig:old_new_bayes} is
%eye-catching and clearly shows the improvement of our results.
%But we haven't explicitly mentioned our results and the existing bounds yet, so
%I am feeling a bit uneasy about having Fig.~\ref{fig:old_new_bayes} here.
%Also, it seems a bit unnatural to leave Fig.~\ref{fig:old_new_det} behind.
%}

{\bf Our contributions.} The primary contributions of our paper are two fold: 

%\kx{Please verify if this summary is reasonable. }

\begin{enumerate}
\item We settle the optimal query complexity of the Private Sequential Learning problem in both the Bayesian setting (where $X^*$ is random) and  deterministic setting (where $X^*$ is fixed but arbitrary). We do so by establishing query complexity upper and lower bounds that almost match in the entire parameter range,
%match up to an additive constant in nearly the entire parameter range 
%\dy{This is not entirely true under the Bayesian setting, when $L$ is not a constant. I slightly prefer ``... bounds that almost match in the entire parameter range"},
%\kx{sounds good, please go ahead !}
 thus obtaining a complete picture of the optimal query complexity as a function of the estimation accuracy and the learner's privacy level. Our results substantially improve upon the best known upper and lower bounds \cite{tsitsiklis2018private, xu2018query}, and the improvements are most drastic over an important range of parameters, where both the adversary and the learner aim to locate $X^*$ within small errors
 % the adversary and learner's accuracy requirements are comparable
%\dy{since $\epsilon$ and $\delta$ are of different order, I suggest we reword this to ``where the accuracy requirement of the adversary grows with that of the learner"} \kx{good point, but I think the results are most drastic when both are small, and "grow" might give the opposite impression. How about "...accuracy requirements are both small''? I'm also okay with your proposal.}
%\nbr{JX. I guess it is a bit unnatural to say ``accuracy requirements are small''? maybe "where both the adversary and learner aim to locate $X^*$ within small errors}
; see Figure \ref{fig:old_new_bayes} for an illustrative example. 

\item We propose and analyze an important variant of the private learning model with noisy responses, a feature that is especially salient in real-world operations and machine learning applications where the functional evaluations are often stochastic. In this setting, we prove upper and lower bounds on the optimal query complexity which match up to multiplicative constants that only depend on the level of noise in the responses. This mirrors the best known characterizations available in the  non-private version of the problem, which also has a dependency on the noise level \cite{ waeber2013bisection}. 
\end{enumerate}

\emph{Methodological contribution}. Our results are rooted in new insights into  the nature of learner-private sequential learning problems, which could have broader implications. For instance, in the Bayesian setting, one driving insight is that the portion of the learning process that demands the most privacy protection is \emph{after} the learner has already obtained a reasonably accurate estimate of $X^*$. We further demonstrate that, as an implication of this observation, the querying trajectory of an optimal learner can be roughly divided into two phases: 
\begin{enumerate}
\item A \emph{pure-learning} phase, where the primary objective is to narrow the search down to a smaller interval that contains $X^*$ and privacy is not a top priority. 
\item A \emph{private-refinement} phase, where the learner  refines her estimate of $X^*$ within the said interval, while allocating significantly more querying budget towards obfuscation. 
\end{enumerate}
We develop new learning strategies and analytical techniques to make this intuition precise. The algorithmic implications are significant. On the one hand, it suggests more efficient learner strategies that allocate more obfuscation budget towards the latter stages of learning. On the other hand, it can be used to design more powerful adversary strategies that focus on latter stages of a learner's query trajectory, and obtain a more accurate estimator. In analogous manner, our analysis of the deterministic setting also leverages a two-phase approach, although the obfuscation budget is now skewed towards the earlier stages of learning.

\subsection{Motivating Examples}\label{sec:motivation}
%\nbr{In view of ICML reviews, maybe we want to further clarify data owner's privacy and 
%learner's privacy. I like the pricing example that you used in the response. I think the pricing example
%is much simpler and easier to understand than FL. Maybe we can expand that example as one motivation
%before jumping to FL directly.}
%\nb{The pricing example is added below.}

We examine in this sub-section several motivating applications of the Private Sequential Learning model. To be clear, the private learning model we study is highly stylized, and as such applying our algorithms in an application would require caution and necessary modifications. However, it appears likely that the rigorous study we carry out and the structural properties it unveils would,  at the very least, yield valuable insights and help guild policy design in real-world systems.

\paragraph*{Learning the optimal price:} As discussed in \cite{tsitsiklis2018private}, dynamics similar to those in the Private Sequential Learning model arise in the domain of dynamic pricing. Suppose a company is conducting market experiments to determine the release price of a product. The goal is to learn a global parameter about the entire consumer base, e.g., $X^*$ equals the highest price to charge so that at least 50\% of the consumers would purchase the product. At each epoch of the experiment, the company samples a subset of the consumers and experiment on a test price (query). Under the sequential learning model, the response $r_i$ corresponds to the indicator function of whether at least 50\% of the sampled consumers would purchase the product at price $q_i$. Note that, due to individual differences and the sampling process, the response is a noisy version of its population variant, which can be captured by the noisy variant of the model we study in this paper. In this example, a competitor (adversary) can easily access the sequence of test prices by participating in the experiments, but does not observe
the responses. The optimal query complexity refers to the minimum number of epochs the company takes to estimate $X^*$ accurately, while making sure the eavesdropping adversary cannot infer the final release price. Notice the distinction between our privacy incentive and the incentive to protect the data owners' privacy. The latter aims to ensure that the query sequence does not reveal the price each individual participant is willing to pay, which varies from person to person and can be far from $X^*$.

\paragraph*{Federated Learning:}  
%\kx{(to do: add discussion on: 1) model theft is real / prevalent  2) query eavesdropping is a valid vector of attack.)}
As mentioned in the Introduction, our study of Private Sequential Learning is motivated in part by wanting to protect the privacy of the learner, rather than data owners, in a learning-with-external-workers system. 
Among these systems,  Federated Learning is an emerging machine learning model training paradigm that has been gaining traction over the past few years~\cite{konevcny2015federated,konevcny2016federated,ChenSuXu17,su2019securing,kairouz2019advances,bagdasaryan2020backdoor,yin2018byzantine,bhagoji2019analyzing,alistarh2018byzantine,yang2020federated,yin2019defending,xie2018generalized,li2019rsa,ding2019distributed,wu2019federated}, and has been deployed in products by companies such as Apple \cite{hao_2019} and Google \cite{FL_google2017}. 

%\kx{I'd generally try to minimize acronyms, so will use Federated Learning in full when possible.}

Before explaining how the Private Sequential Learning model applies in this context,  we briefly review the basic mechanisms of Federated Learning. 
A typical Federated Learning training process is sequential in nature and works as follows (see e.g.~\cite[Algorithm~1]{mcmahan2016communication}). 
A central learner trains a global model by aggregating local model updates across a large number of users
on mobile devices. At each iteration $t$, the learner broadcasts the current model parameter $w_t$ to all (or a subset) of the users.
Using their local data, each user $i$ then trains an individual model update $w_t^i$ starting from $w_t$, and sends it back to the central learner.
The learner aggregates all the model updates $w_t^i$ across all user $i$ to produce $w_{t+1}$, the model parameter for the next iteration. 
Importantly, like most systems in its category, Federated Learning is designed with the  goal of protecting the privacy of data owners (the users) 
\cite{geyer2017differentially,mcmahan2017learning,agarwal2018cpsgd}, while offering no explicit privacy guarantees on the central learner's learned model parameters under eavesdropping attacks.
\footnote{Although in Federated Learning the final trained model is usually released for all users to access,
the learner often chooses to keep the model parameters in the central server and only allows
users to perform evaluation tasks.} Our focus, in contrast, would be on the privacy concern of the latter. 

Specifically, when training with thousands of users, as 
the learner lacks enough administrative power over those external workers,  the Federated Learning system is highly vulnerable to eavesdropping attacks~\cite{kairouz2019advances}.
An \emph{honest-but-curious} adversary can participate in the training stage by pretending to be an user, and
eavesdrop on the sequence of broadcasted model parameters. Simply by taking the last set of model parameters, the adversary can
approximate the learner's final model fairly well. Sophisticated models can be worth millions. 
%By stealing the models, 
The eavesdropper can use the stolen model to profit 
%in the same way that the companies behind the algorithms do, 
or even leverage them for illicit purposes~\cite{juuti2019prada}. It not only saves
 the eavesdropper from investing the tremendous amounts of funding into training the model, 
 but it could also devalue the learner's model.
 Therefore, it is of paramount importance to protect the learner's privacy from eavesdropping attacks. 

There are several potential techniques to conceal the model parameters from the users in Federated Learning,
such as restricting each user to run the local computation inside a Trusted Executation Envrionments (TEE)~\cite{subramanyan2017formal} or encrypting the model parameters
under a homomorphic encryption scheme before broadcasting it to the users~\cite{mohassel2017secureml}. Unfortunately, as pointed out by the recent survey~\cite[Section 4.3.3]{kairouz2019advances},
TEEs may not be generally available across all workers especially when these workers represent end-devices such as smartphones. 
Moreover, TEEs and homomorphic encryption are often costly to implement and incurs large overhead. There is an emerging line of research
on preventing model theft in the evaluation stage, where an adversary attempts to extract the deployed model by repeatedly querying the model and obtaining estimation on the input feature vectors~\cite{tramer2016stealing,papernot2017practical,shi2017steal,wang2018stealing,juuti2019prada,orekondy2019prediction,orekondy2019knockoff,kariyappa2020defending}.
However, this line of work does not address the unique challenge of concealing the model parameters from the users during the training stage in Federated Learning. 

%TEEs may not be generally available across all workers
%especially when these workers represent end-devices such as smartphones. Moreover, TEEs and homomorphic encryption are often costly to implement and incurs large
%overhead. Last but not least, recent studies show that even if the model parameters themselves are successfully hidden, they can still be reconstructed by
%an adversary who only has access to  query interfaces based on these parameters~\cite{tramer2016stealing}. 
%This  consideration prompts us to investigate whether we can achieve information-theoretic security against the eavesdropping attack.

This consideration prompts us to investigate whether we can offer provable guarantees on the learner's privacy against the eavesdropping attack in Federated learning. 
%immediately leads to 
In particular, we aim to address the following two natural but fundamental questions:
\begin{enumerate}
\item 
%(1) 
Can the learner arrive at an accurate model, while ensuring that the eavesdropping adversary cannot learn the same model with a high level of accuracy? 
\item 
%(2) 
What is the minimal number of iterations needed in the training process, for accurate and private learning? 
\end{enumerate}

%Note that we focus on private learning from the learner's perspective, that is privatizing the final trained model parameters.  This is different from protecting data owners' privacy, which privatizes the individual users' local data.

The sequential learning problem we study can be viewed as a special case of the problem faced by the central learner in a Federated Learning framework, where the true model parameter $X^*$ is in one dimension. In particular, at iteration $t$ the learner broadcasts the current model parameter $w_t$ (viewed as the query). 
Then instead of the local model update $w_t^i$, each user $i$ sends back to the learner only the directional information $\text{sign}(w_t^i-w_t)$ of the update. 
The majority vote of $\text{sign}(w_t^i-w_t)$ is viewed as the response, which is often noisy due to noise in the users' local data.
%Suppose there are enough non-adversarial users having sufficient amount of data in total,
%so that the majority vote of $\text{sign}(w_t^i -w_t)$ across all user $i$  is effectively noiseless and 
%equal to $\text{sign}(X^*-w_t)$ (viewed as the response). 
%\nbr{Not sure whether we still need to emphasize
%the noiseless responses, as we also studied the noisy responses}
%\nb{The sentence is now rewritten to emphasize that the responses can be noisy.}
% , the learner
%can infer the position of the optimal (the optimal model or the true model?) model parameter $X^*$ relative to $w_t$, a.k.a. the response variables in the stylized model.
All users, including the adversarial ones, observe the broadcasted queries.
Moreover, it is reasonable to assume that the adversary does not observe the responses.
That is because in order to observe the responses, 
the adversary would have to access the updates generated by all users in the system, a formidable task that is not realistic for an adversary that only controls up to a small subset of the users.

Note that in Federated Learning, communication bandwidth is a scarce resource, as the data transmission between the external workers and the learner typically suffers from high latency and low throughput. Thus, determining the optimal query complexity (i.e.\ the minimum communication rounds) is of fundamental importance in both theory and practice.

\paragraph*{Other applications:} There are many other applications in which a learner would naturally suffer from privacy breaches if the learning process can be eavesdropped by a third party, for example in crowdsourced learning~\cite{howe2006rise}~\cite{vukovic2009crowdsourcing}. 
Similar scenarios also arise in conducting surveys or medical inspections, where we hope to adaptively collect information, while preventing the adaptive questionnaires from giving away respondents' private information. While the Private Sequential Learning model may not yet directly apply, the fundamental tradeoff between privacy and complexity of learning unveiled by our investigation is likely to speak to similar strategic considerations arising  in these applications. 

\subsection{Comparison with prior work}

%\kx{A fair bit of re-arrangement in this section too...}

The Private Sequential Learning model with noiseless responses was proposed by~\cite{tsitsiklis2018private} and further studied by~\cite{xu2018query}. While simple, this model
already captures a core tension between learning and privacy. To see why a naive learning strategy can put the learner's privacy at risk,   consider the vanilla bisection strategy which recursively 
queries the mid-point of the interval that the learner knows to contain the target $X^*$.
It can be shown that the bisection strategy achieves the minimum sample complexity of $\log(1/\epsilon)$ queries\footnote{Here and subsequently $\log$ refers to logarithm with base 2.}, but it almost entirely compromises the learner's privacy. If the adversary sets their estimator to be equal to the learner's last query, 
then $X^*$ is within a distance of at most $\epsilon.$ 
On the other extreme, consider the non-adaptive grid search strategy, where
the learner partitions
the interval $[0,1]$ into subintervals of length $\epsilon$
and queries the endpoints of all the subintervals. 
While this non-adaptive strategy offers the best protection for the learner's privacy (the adversary can learn nothing about $X^*$ by observing the locations of the queries),  it suffers a very high query complexity of $1/\epsilon$. 
This observation suggests that an optimal learning strategy would have to strike a delicate balance between adaptive querying and willful obfuscation in order to  achieve the optimal trade-off between query complexity and privacy. 

Prior works~\cite{tsitsiklis2018private} and~\cite{xu2018query} developed upper and lower bounds for the optimal query complexity, and quantify the impact of the privacy level $L$ under two formulations of the Private Sequential Learning problem:  a \emph{deterministic setting} where $X^* \in [0,1]$ is fixed but arbitrary,  and a \emph{Bayesian setting} where $X^*$ is uniformly distributed over $[0,1]$. These results, however, mostly focus on the subset of parameters where the adversary's accuracy requirement, $\delta$, is substantially larger than that of the learner's, $\epsilon$. Under the deterministic setting where $\delta=1/L$,
it is shown that the upper and lower bounds almost coincide, yielding an optimal query complexity of 
about $\log(1/\epsilon)+2L$~\cite{tsitsiklis2018private}.  In contrast, under the Bayesian setting, it is shown that if we keep  both $\delta$ and  $L$ fixed, then the optimal query complexity scales as   $L\log(1/\epsilon)$ as $\epsilon$ tends to $0$ \cite{xu2018query}.

Unfortunately, the existing results are not tight, and moreover can perform very poorly within  an important parameter regime, where the adversary accuracy requirement $\delta$ can be small or even comparable to that of the learner. For example, under the Bayesian setting, when $\epsilon$ and $\delta$ go to 0 proportionally while $L$ stays fixed, the upper bound proved in~\cite{tsitsiklis2018private} is of order $L\log(1/\epsilon)$, while the lower bound given in~\cite{xu2018query} is of order $L$, leaving a large gap. This regime of parameters, however, is of both practical and theoretical interest. On the practical front, it is often the case that an adversary is interested in obtaining the learner's model for the same use as that of the learner. As a result, an estimator with a disproportionately large error margin might be of no practical use to the adversary. On the theoretical front, allowing both $\epsilon$ and $\delta$ to vary simultaneously significantly increases the difficulty of analysis. Indeed, as we explain in detail in Section \ref{sec:newIdeas}, there is strong evidence that the learner strategies and proof techniques developed in \cite{tsitsiklis2018private} and \cite{xu2018query} are not sufficiently refined to achieve the optimal query complexity bounds in this regime. By developing new insights and algorithmic ideas, we will establish a tight characterization of the optimal query complexity as a function of $\epsilon, \delta,$ and $L$, improving the existing upper \emph{and} lower bounds over the entire range of these parameters.

%\dy{Upon reading the paragraph above, I worry that the reader might get the impression that we only extend previous results to include another range of parameters. Perhaps it is worth it also emphasize here that our results {\it improve upon} both the upper and the lower bounds across the entire parameter regime.}

%\kx{Agreed. I've made some adjustments in red to this effect. Feel free to change as you see fit.}

%Such small $\delta$ case is important, as  in practice the learners can usually tolerate an inaccurate adversary, and would only need to  guarantee that no adversary can be $\delta$-accurate for some small $\delta$. This observation leads to the first open question:

%\emph{What is the optimal query complexity as a function of $\epsilon, \delta, L$ over the entire  range of allowable parameters? }

We further depart from the existing noiseless private learning model by incorporating noisy query responses in the Bayesian setting, a crucial feature for many practical applications including the pricing example and Federated Learning considered in Section~\ref{sec:motivation}.  In this case, we establish query complexity bounds that match up to a constant multiplicative factor that only depends on the noise level.  Note that even without privacy considerations, the 
binary search problem with noisy responses admits highly non-trivial dynamics, and has been studied extensively~\cite{feige1994computing, burnashev1974interval, waeber2013bisection}; see~\cite{pelc2002searching} for a more comprehensive survey. Sophisticated search algorithms, such as the probabilistic bisection algorithm~\cite{waeber2013bisection} and the Burnashev-Zigangirov algorithm~\cite{burnashev1974interval}, have been developed based on the idea of recursively refining a posterior belief distribution on $X^*$. Unfortunately, just like the bisection strategy, none of these algorithms provide privacy protections for the learner since their query trajectories eventually concentrate around the target $X^*$. Existing private learning strategies designed for the noiseless case also do not admit trivial extensions to the noisy setting. This is not a coincidence: existing strategies heavily exploit the learner's ability to determine, with absolute certainty, a certain small sub-interval that contains the target $X^*$ using only  a constant number of queries \cite{tsitsiklis2018private, xu2018query}, which is impossible when responses are noisy. A further challenge arises in the noisy response model where the analysis must carefully track the posterior distributions of $X^*$ for both the learner and the adversary simultaneously. On  the one hand, we want the posterior distribution of $X^*$ given the responses to concentrate fast, so that the learner can accurately estimate $X^*$ despite the noises.
On the other hand, we need to make sure the posterior distribution given the queries does not concentrate too rapidly, ensuring that the adversary cannot accurately learn.  This greatly complicates the design and the analysis of the optimal querying strategy, as a closed-form expression for the posterior distributions is out of reach in this case. 

%\emph{2.~How to optimally protect privacy while coping with noise in the responses?}

%It can be shown that under these search algorithms, the posterior distribution of $X^*$ given the responses concentrates geometrically fast, making their query complexity optimal up to multiplicative constants that only depend on $p$. 
%However, these algorithms are almost non-private, as the query converges to $X^*$. 

\paragraph{Relation to data-owner privacy models}  As mentioned in the Introduction, the Private Sequential Learning model differs significantly from the existing literature on private iterative learning,
such as private stochastic gradient descent~\cite{song2013stochastic,abadi2016deep,agarwal2018cpsgd}, private online learning~\cite{jain2012differentially}, and private Federated Learning~\cite{geyer2017differentially,mcmahan2017learning,nasr2019comprehensive,melis2019exploiting}. 
The focus therein is to protect \emph{data owners' privacy} by preventing 
the adversary inferring about a data owner from the outputs of learning algorithms, often using the notion of differential privacy~\cite{Dwork08}. 
In that setting, a common privacy-preserving mechanism is to inject calibrated  noise  
at each iteration of the learning algorithms. In contrast, our work aims to protect \emph{the learner's privacy} 
by preventing the adversary inferring the learned model from the learner's queries. 
As a result, our problem setup, privacy-preserving mechanisms, and main results are significantly different from those in this literature. Our focus on a decision maker's obfuscation task is related, at a high level, to recent studies of information-theoretically sound obfuscation in various sequential decision-making problems \cite{fanti2015spy, luo2016infection, tsitsiklis2018delay, erturk2019dynamically, tang2020privacy}. However, most of these models focus on protecting data and information already in the position of the decision-maker. As such, they do not address the unique privacy challenges arising in learning, where the learner has to protect a piece of information that they themselves are just in the process of discovering.

\subsection{Summary of main results}
We offer a preview of our main results in this sub-section. For noiseless responses, we prove almost matching upper and lower bounds on the optimal query complexity in both Bayesian and deterministic settings, including an extension to multi-dimensional private learning. 
For models with noisy responses, we focus on the Bayesian setting and prove upper and lower bounds that match up to constant multiplicative  factor that only depend on the noise level, which coincide with the state-of-the-art noise-dependent characterizations in non-private version of the noisy learning problem \cite{waeber2013bisection}. 

%\dy{perhaps we include a table that compares the rates of the upper/lower bounds with prior work.}

%\kx{sounds like a good idea.}

%We also analyze the more general and realistic model where the responses are corrupted by noise, under which we 
\subsubsection{Noiseless Responses}
For noiseless responses, we establish the following characterization of the optimal query complexity $N(\epsilon,\delta,L)$: 

\begin{itemize}
\item Bayesian setting:
$$
N(\epsilon,\delta,L) \approx  \log \frac{1}{L\delta} + L \log \frac{\delta}{\epsilon}.
$$
\item Deterministic setting:
$$
N(\epsilon,\delta,L) \approx 2L + \log \frac{\max\left\{ 2^{-L}, \delta \right\}}{\epsilon};
$$
\end{itemize}
Our upper and lower bounds are tight up to an additive factor of $4L$ and 8 for the Bayesian and deterministic settings, respectively. Notably, under the Bayesian setting, when $\delta$ is a constant multiple of $\epsilon$,
the optimal query complexity scales as $\log(1/\epsilon)$ plus a constant multiple of $L$,
implying that in this regime the correct price for privacy is only additive in $L$,
as opposed to multiplicative, as suggested by the upper bound $L\log(1/\epsilon)$ in~\cite{tsitsiklis2018private}. A comparison between our results and the existing bounds is given in Table~\ref{tab:compare}, with examples illustrated in Figures \ref{fig:old_new_bayes} and \ref{fig:old_new_det},  under the Bayesian and deterministic settings, respectively.

%To prove our upper bounds, we construct new querying strategies that ensure estimation accuracy for the learner while maintaining a desired level of privacy against the adversary. To show our lower bounds, we prove that no strategy that submits fewer queries can achieve the same level of accuracy and privacy. %See the supplementary material for details.

\begin{figure}
\centering
\includegraphics[scale=0.5]{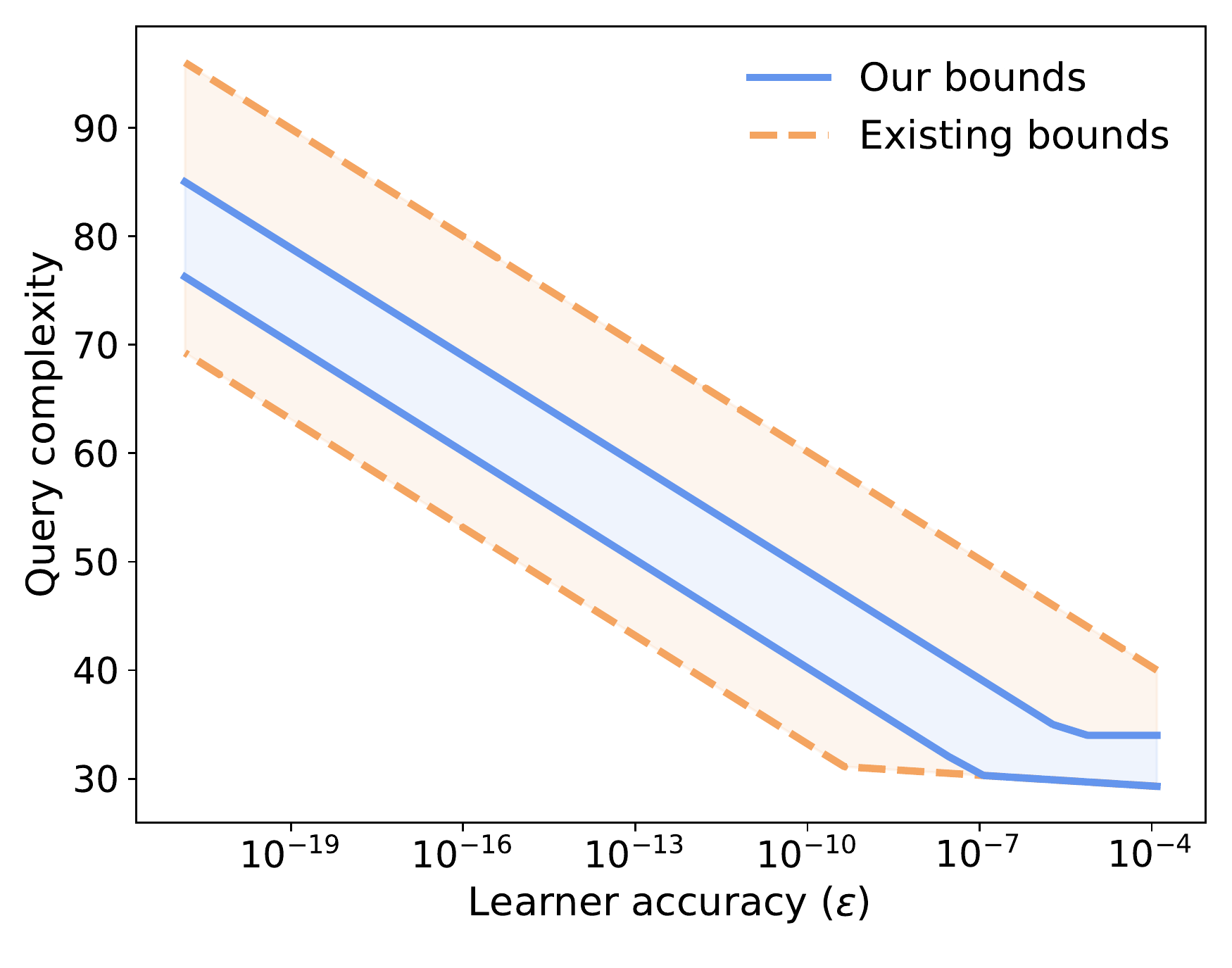}
\caption{Our results in Theorem \ref{thm:freq} (solid) versus the best known bounds \cite{tsitsiklis2018private} (dashed) for the noiseless deterministic setting, $L=15$ and $\delta= 4\epsilon^{0.9}$. }
\label{fig:old_new_det}
\end{figure}

We further extend these results to a model of high dimensional learning. Suppose the learner aims to estimate a target value $X^*\in [0,1]^d$ by submitting queries along different dimensions:  each query provides a binary response on the position of  $X_i^*$ relative to the query location, for some dimension $i$. We show that the optimal query complexity $N_d(\epsilon,\delta,L)$ in $d$ dimensions behaves as follows:
\begin{itemize}
\item Bayesian setting:
$$
N_d(\epsilon,\delta,L) \approx  d\left(\log \frac{1}{L^{1/d}\delta} + L^{1/d} \log \frac{\delta}{\epsilon} \right).
$$
\item Deterministic setting:
$$
N_d(\epsilon,\delta,L) \approx d\left(2L^{1/d} + \log \frac{\max\left\{ 2^{-L^{1/d}}, \delta \right\}}{\epsilon}\right). 
$$
\end{itemize}
We see that the optimal query complexity suffers from a multiplicative factor of $d$, analogous to the 
optimal query complexity $d\log(1/\epsilon)$ when there is no privacy constraint. However, the query
complexity per dimension depends on $L^{1/d}$, which decreases in $d$. 

\begin{table}[h]
\begin{tabular}{ |p{3.6cm}||p{5cm}|p{6.5cm}| }
 \hline
& Bayesian setting &Deterministic setting
\\
\hline
& &\\[-1em]
Existing upper bound & $L\lceil\log\frac{1}{L\epsilon}\rceil+L-1$   \cite{tsitsiklis2018private}  &  $\left\lceil\log\frac{1}{L\epsilon}\right\rceil+2L$  \cite{tsitsiklis2018private}  \\
& &\\[-1em]
 \hline
& &\\[-1em]
Existing lower bound &   $L\left(\log\frac{\delta}{\epsilon}-3\log\log\frac{\delta}{\epsilon}-1\right)$  \cite{xu2018query} &  $\max\left\{\left\lceil\log\frac{1}{\epsilon}\right\rceil, \left\lceil\log\frac{\delta}{\epsilon}\right\rceil+2L-4\right\}$  \cite{tsitsiklis2018private}\\
& &\\[-1em]
\hline
& &\\[-1em]
Our upper bound &$\left\lfloor \log \frac{1}{L\delta}\right\rfloor+ L\left(\left\lceil\log \frac{\delta}{\epsilon}\right\rceil+2\right)-1$ & $\max\left\{\left\lceil\log\frac{1}{\epsilon}\right\rceil+L, \left\lceil\log\frac{\delta}{\epsilon}\right\rceil+2L\right\}$\\
& &\\[-1em]
\hline
& &\\[-1em]
Our lower bound    &$\left\lfloor \log \frac{1}{L\delta}\right\rfloor+ L\left(\log\frac{\delta}{\epsilon} -2\right)-1$ & $\max\left\{\left\lceil\log\frac{1}{\epsilon}\right\rceil+L-8, \left\lceil\log\frac{\delta}{\epsilon}\right\rceil+2L-4\right\}$\\
 \hline
\end{tabular}
\caption{A comparison of our upper and lower bounds on the optimal sample complexity with existing bounds, under the Bayesian and deterministic settings with noiseless query responses.}
\label{tab:compare}
\end{table}

\subsubsection{Noisy Responses}
We further consider the noisy response model, where each response can be incorrect with probability $1-p$ for some noise level $p \in (0,1)$. In this case, we need to redefine learner's accuracy, since it is no longer possible for the learner to estimate $X^*$ well with probability 1. We consider two natural definitions of accuracy: say the learner is accurate (a) on average if $\mathbb{E}|\widehat{X}-X^*|\leq \epsilon/2$; and (b) with high probability if $\mathbb{P}\left\{|\widehat{X}-X^*|\leq \epsilon/2\right\}\geq 1-1/M$. Below is our main result for the noisy response model.

The optimal query complexities under accuracy definitions (a), (b) satisfy
\begin{equation}
\label{eq:noisy.na}
N_{\mathsf{avg}}(\epsilon, \delta, L) \asymp L\log \frac{\delta}{\epsilon} + \log \frac{1}{\epsilon};
\end{equation}
\begin{equation}
\label{eq:noisy.nb}
N_{\mathsf{whp}}(\epsilon, M, \delta,L) \asymp L\log\frac{M\delta}{\epsilon} + \log \frac{1}{\epsilon},
\end{equation}
%\nbr{JX. Let's discuss these expressions}
where $\asymp$ denotes bounds from above and below, up to multiplicative constants that only depend on $p$. 
By manipulating the constants,~\eqref{eq:noisy.na} can be rewritten as $N_\mathsf{avg}(\epsilon,\delta,L)\asymp \log(1/(L\delta))+L\log(\delta/\epsilon)$.
Comparing with the optimal query complexity under the noiseless Bayesian setting, we see that under definition (a), the noise level (parametrized by $p$) only affects the optimal query complexity
through the constant multiplier;
when the learner needs to be accurate with probability at least $1-1/M$, as per definition (b), there is an extra cost of $\Theta(L\log M)$ queries.
As in the noiseless setting, the multiplier $L$ only acts on the terms $\log(\delta/\epsilon)$ and $\log(M\delta/\epsilon)$ under
accuracy requirements (a) and (b) respectively,
indicating that the learner only starts paying the price for privacy after entering the $\delta$-accurate regime.

We comment that even for the vanilla noisy binary search problem with no privacy constraints, finding the exact $p$-dependent multiplicative constant
is an open problem. We also comment that as $p\rightarrow 1/2$, we obtain constants in the upper and lower bounds that go to infinity at the same rate of $(p-1/2)^{-2}$.

\section{New insights and algorithmic ideas}
\label{sec:newIdeas}

%\kx{I've reworked some of the paragraphs here. Please verify. }

In this section we highlight  the new algorithmic ideas and analysis techniques that enable us to obtain sharp bounds. A fundamental difficulty in the design of private learning strategies is that the learner's queries simultaneously serve two, sometimes competing, goals: (1) to gather information about the target $X^*$, and (2) to deceive the adversary as to the target's location. More specifically, our approach begins by recognizing that a single learner query can be used to accomplish one, or multiple, of the following tasks: 
\begin{enumerate}[(i)]
\item to obtain information in order to identify a small interval that contains $X^*$ (diameter of this interval depends on the privacy level and adversary accuracy);
\item to obtain information in order pin-point $X^*$ within the said small interval, down to an $\epsilon$-accuracy; 
\item to serve as a ``decoy'' to throw off the adversary. 
\end{enumerate}

%\kx{why}  Queries of type (ii) and (iii) rarely overlap; otherwise a query can possibly multitask. 

The key to our analysis is to put emphasis on understanding the interaction of the queries serving these different types of tasks. We design more efficient querying strategies where multiple tasks are accomplished simultaneously. We also provide sharp analysis on the maximum number of queries that can be used for more than one purpose. In contrast, it appears that prior work has not clearly identified or articulated these separate roles of learner queries. As a result, the existing learner strategies either leave certain types of tasks out of the analysis, or have each of the three types of tasks to be served by separate queries, leading to inefficiencies. 

Specifically, starting with the Bayesian setting, recall that privacy is breached if an adversary can estimate $X^*$ up to an additive error of $\delta/2$ with probability at least $1/L$. That makes $\delta L$ an appropriate choice for the diameter of the small interval. Queries leading to the identification of such an interval (type (i)) do not significantly compromise the learner's privacy, because the interval's size is too large for the adversary to extract useful information. Beyond this point, however, the learner must submit further queries to narrow the range of $X^*$ down to $\epsilon$, and these queries must be carefully obfuscated. In other words, effective queries at this point should serve to accomplish tasks (ii) and (iii) simultaneously.  As such, in the design of the optimal strategy, we will divide learning into two phases: 
\begin{enumerate}
\item a {\it pure-learning phase}, corresponding to task (i), where the sole focus of the learner is to identify a small interval containing the target $X^*$, and
\item a {\it private-refinement phase}, corresponding to tasks (ii) and (iii), where queries serve to simultaneously refine and obfuscate a fine-grained estimate of $X^*$. 
\end{enumerate}
 In comparison, the strategy proposed in~\cite{tsitsiklis2018private} is more wasteful as the learner would try to tackle all three tasks at the same time, despite there being no need to worry about privacy in the pure-learning phase.

For the lower bound, previous work~\cite{xu2018query} adopts a genie-aided reduction argument, where a $\delta$-length interval containing $X^*$ is assumed to be revealed to the learner from the get-go. Although this reduction simplifies the proof, it cannot lead to a tight lower bound as it ignores all type (i) queries. In order to capture these queries, our lower bound proof introduces an adversary who adopts a more intelligent ``truncated proportional-sampling" strategy. The key to this strategy is to disregard a certain number of queries which could negatively impact the adversary's estimator. It is worth pointing out that the number of queries the adversary should disregard is exactly the number of queries submitted in the {\it pure-learning phase} under the learner's optimal strategy, revealing an elegant duality between the learner and the adversary.

The story is slightly different under the deterministic setting. Unlike the Bayesian setting where the adversary only needs to perform well when averaging over a random $X^*$, here the adversary can no longer make guesses. Knowing that the adversary cannot guess, the learner only needs to worry about privacy breaches in a $\delta$-width interval containing $X^*$. Moreover, before reaching this interval, the learner can ensure privacy by injecting possible alternative locations of $X^*$ along each query sequence. That corresponds to reusing queries for tasks (i) and (iii). In particular, we will design a query strategy that mirrors the two-phase architecture described for the Bayesian setting, with a coarse learning phase (i) followed by that of a refinement (ii) . However, obfuscation efforts (iii) are  now implemented during the first phase. Our upper bound proof involves designing an efficient strategy that maximizes the number of reused queries. For the lower bound, previous work~\cite{tsitsiklis2018private} once again assumes that $X^*$ is in a $\delta$-length interval known to the learner, therefore ignoring all the type (i) queries. To obtain the sharp lower bound, our analysis dissects the query sequence, and separately investigates those queries that can be used both for protecting privacy and searching for $X^*$, and those queries that only fulfill one purpose.
%two purposes, and those that only help accomplish one. \nbr{JX. Here ``two purposes'' and ``accomplish one'' are a bit vague. Is it possible to be more specific about which purpose?}

These intuitions are used also for the extensions of our results to the multi-dimensional model, and the model with noisy responses. However when the responses are noisy, we encounter some additional challenges.

For proof of the upper bound in the noisy response model, as we previously mentioned, the main difficulty is caused by an intractable posterior distribution of $X^*$ given the queries. Even queries that are far from $X^*$ can change the shape of the posterior distribution and potentially leak the location of $X^*$ to the adversary. Therefore it is harder to show that a strategy is private. To overcome this difficulty, our analysis involves the design of a querying strategy that forces certain conditional independence structures between the query sequence and some local neighborhood of $X^*$. We then use the conditional independence to carefully control the privacy leakage across all phases of learning.

% a querying strategy design with minimal privacy leakage \nbr{JX. Here ``minimal privacy leakage'' is a bit vague. Can we be a bit more specific about how we design with minimal privacy leakage?}, and carefully control on the privacy leakage across all stages of learning.

For the lower bound proof, we need to establish a tight lower bound on the learner's probability of error. To that end, two sets of tools are deployed. Part of the proof utilizes information-theoretic arguments. The key step is to establish an upper bound on the {\it rate of information transfer}, which governs the speed at which the learner can gather information from the responses. For the second part of the proof, we reduce the learner's estimation problem to a family of binary hypothesis testing problems, and bound the testing errors from below using the Bhattacharyya coefficient~\cite{kailath1967divergence}.

\medskip
\paragraph{Organization} The remainder of this paper is organized as follows. Section~\ref{sec:problem} contains the problem formulation and definitions of accuracy and privacy, under both the noiseless and noisy response settings. In Section~\ref{sec:results} we state the main results on the optimal query complexities. In Section~\ref{sec:alg.noiseless} we give the construction of our querying strategies when the responses are noiseless. In Section~\ref{sec:lower.noiseless} we outline the lower bound proof strategies, again with noiseless responses. 
%In Section~\ref{sec:multi.dim} we extend our results under the stylized sequential problem to multi-dimensional private learning. 
In Section~\ref{sec:alg.noisy} we give the construction of our querying strategies and discuss the lower bound proof techniques under the noisy response setting. 
%\nbr{I am thinking whether we should also briefly discuss the lower bound proof strategies under the noisy setting and the key difference compared to the noiseless settings}
%\nb{A subsection is added to the end of Section~\ref{sec:lower.noiseless} that sketches the lower bound proof.} 
In Section~\ref{sec:conclusion} we conclude the paper and give a brief discussion on future work. The full proofs of all the results are contained in the appendix.

\section{Problem formulation} \label{sec:problem}
%and results}
Consider the problem of learning some unknown true value $X^*\in [0,1]$. Let $\widehat{X}$ be the learner's estimator of $X^*$ and $\widetilde{X}$ be the adversary's. 
The learner submits queries $q_1,q_2,...\in [0,1]$ sequentially. 
Each time a query $q_i$ is submitted, 
the learner receives a response $r_i$. When the responses are noiseless, $r_i=\mathds{1}\{X^*\geq q_i\}$. Under the noisy response setting, we assume that
\[
r_i\sim \text{Bernoulli}(p)\quad \text{ if } X^*\geq q_i, \text{ and }r_i\sim \text{Bernoulli}(1-p)\quad \text{ if } X^*< q_i
\]
for some $p\in (1/2,1)$. That is, each observed response can be erroneous with probability $1-p$. 

The learner's query $q_i$ can depend on all the past queries and responses, and is allowed to incorporate outside randomness. 
Since all random variables and all random vectors with finite alphabets can be simulated from a random variable uniformly distributed on $[0,1]$,
without loss of generality, 
let $Y\sim \text{Unif}[0,1]$ be the random seed that the learner may use to generate queries. 
Then $q_i$ can be written as $f_{i-1}(q_1,...,q_{i-1},r_1,..., r_{i-1},Y)$ for some function $f_{i-1}$. 
Note that the first query $q_1$ is submitted without any information and is only a function of $Y$. Thus we have $q_2=f_1(q_1,r_1, Y)=f_1(f_0(Y), r_1, Y) := \phi_1(r_1, Y)$.
%and $q_3=f_2(q_1,q_2,r_1,r_2,Y)=f_2(f_0(Y), \phi_1(r_1, Y), r_1, r_2, Y)=: \phi_2(r_1,r_2, Y)$. 
It is easy to see that all $q_i$ can be written iteratively as a function of only the past responses and $Y$, i.e., $q_i=\phi_{i-1}(r_1,...,r_{i-1}, Y)$.

Then a querying strategy $\phi$ is defined by an initial mapping $f_0: [0,1]\rightarrow [0,1]$ used to generate $q_1$ from $Y$, 
a sequence of mappings $(\phi_i)_i$ with $\phi_i: \{0,1\}^i\times [0,1]\rightarrow [0,1]$ used to generate the rest of the query sequence, 
and a final estimator $\widehat{X}$, which can depend on $Y$ and all the queries and responses. 
The adversary's estimator $\widetilde{X}$, on the contrary, is formed with only access to the queries and the querying strategy $\phi$ but not
the responses or the random seed $Y$.

The goal of the learner is to design a querying strategy to ensure that she can accurately estimate $X^*$, but the adversary cannot. Different ways to quantify the estimators' performance arise naturally when the responses are noisy versus noiseless. We discuss the two settings separately.

\subsection{Noiseless responses}

Following~\cite{tsitsiklis2018private},  
we consider both the Bayesian setting where $X^*\in [0,1]$ is uniformly distributed  on $[0,1]$ and the setting where $X^*$ is deterministic. 
The two settings call for different definitions for accuracy and privacy, which we shall discuss separately. 

\paragraph*{Bayesian setting}
We assume $X^*$ is uniformly distributed  on $[0,1]$, which is independent
from the random seed $Y$, as the learner does not know the true value $X^*$ a priori. 
We say a strategy $\phi$ is
\begin{itemize}
\item $\epsilon$-accurate for $\epsilon>0$, if 
$
\mathbb{P}\{ |\widehat{X}-X^* | \leq \epsilon/2 \}=1;
$
\item $(\delta,L)$-private for $\delta>0$ and an integer $L\geq 2$, if there is 
\emph{no} adversary $\widetilde{X}$ such that 
$$
\mathbb{P}\left\{ \left|\widetilde{X}-X^* \right| \leq \delta/2 \right \}> \frac{1}{L}.
$$
\end{itemize}

\paragraph*{Deterministic setting}
Suppose $X^*$ is a deterministic but arbitrary number on $[0,1]$. 
Then the only source of randomness 
in the querying strategy is from $Y$. We say a strategy $\phi$ is
\begin{itemize}
\item $\epsilon$-accurate for $\epsilon>0$, if 
$$
\mathbb{P}\left\{ \left|\widehat{X}-X^* \right|\leq \epsilon/2 \right\}=1, \quad \forall X^* \in [0,1];
$$
\item $(\delta,L)$-private for $\delta>0$ and an integer $L\geq 2$, if for each query sequence $\bar{q}$, the $\delta$-covering number
\footnote{The $\delta$-covering number of a set $A\subseteq \mathbb{R}$ is defined as the size of the smallest set $\mathcal{N}$, such that $\cup_{r\in \mathcal{N}} [r-\delta/2,r+\delta/2]\supseteq A$.}
of the information set $\mathcal{I}(\bar{q})$ is at least $L$. The information set is defined as the set of all true values that could lead to the query sequence $\bar{q}$ under strategy $\phi$ with non-negligible probability. Note that the query sequence $q$ is a random vector that depends on $X^*$ and $Y$, i.e., $q=q(X^*,Y)$. Formally we define
$$
\mathcal{I} \left( \bar{q} \right)=\left\{ X^*\in [0,1]: \mathbb{P} \left\{ q(X^*,Y)=\bar{q} \right\}>0 \right\}.
$$
\end{itemize}

Unlike the Bayesian setting, the definition for privacy no longer involves an adversary's estimator $\widetilde{X}$. However, the  definition does admit a probabilistic interpretation. One can show (see~\cite[Appendix~A]{tsitsiklis2018private} for a proof) that  this definition of privacy is equivalent to the following: 
for each query sequence $\bar{q}$, there is \emph{no} adversary estimator $\widetilde{X}$ such that for all $X^*\in \mathcal{I}(\bar{q})$, $\mathbb{P}\{ |\widetilde{X}-X^* |\leq \delta/2 \}> \frac{1}{L}$.

Compare this to the definition of $(\delta,L)$-privacy in the Bayesian setting,
the difference is that when $X^*$ is deterministic, the adversary can no longer average over some prior distribution of $X^*$; instead, 
for some query sequence she needs to learn well for all admissible values of $X^*$.
Knowing that the adversary is not allowed to make guesses,
the learner can ensure privacy more easily, by injecting admissible values of $X^*$ into each query sequence.
%In this sense the learning task is harder for the adversary, which means
%that it is easier for the learner to achieve privacy.
This distinction from the Bayesian setting is reflected by a smaller optimal query complexity, as our results show. 
%in Theorem~\ref{thm:freq}.

For both Bayesian and deterministic settings, we define 
the optimal query complexity as 
\[
 N\left(\epsilon,\delta,L\right)= \min\{  n: \exists\phi\text{ that is both }\epsilon\text{-accurate and }(\delta,L)\text{-private and submits at most }n\text{ queries}\}.
\]
Note that for a larger $\delta$ or a larger $L$, 
%it is easier for the adversary to satisfy $\mathbb{P}\{|\widetilde{X}-X^*|\leq \delta/2\}>1/L$, 
%meaning that 
the $(\delta,L)$-private constraint is a stronger requirement. 
Therefore $N(\epsilon,\delta,L)$ is monotone nondecreasing in $\delta$ and $L$.

Same as~\cite{tsitsiklis2018private}, we focus on the regime of parameters 
$$
2\epsilon\leq\delta \leq \frac{1}{L},
$$ 
which is natural and without loss of generality. To see this,  
on one end of the spectrum, if $\delta>1/L$, then the adversary can make an arbitrary guess to break the privacy constraint: simply choosing $\widetilde{X}=1/2$ yields $\mathbb{P}\{|\widetilde{X}-X^*|\leq \delta/2\}=\delta>1/L$. 
In this regime the $(\delta,L)$-privacy constraint is too strong to be satisfied by any querying strategy. 
On the other end of the spectrum, the regime $\delta\leq 2\epsilon$ is arguably not that interesting,
as it is unnatural to require an adversary, who only have access to queries but not responses, 
to estimate $X^*$ almost as accurately as the learner does.

\subsection{Noisy responses under the Bayesian setting}
In the noisy response setting, we only consider the Bayesian formulation where $X^*\sim \text{Unif}[0,1]$. Since the responses contain noise, no learner that submits a finite number of queries can estimate accurately with probability one. Hence the definition for the learner's accuracy needs to be modified. We consider the following two natural definitions.

\begin{enumerate}[(a)]
\item (accurate on average) We say a querying strategy is $\epsilon$-accurate for $\epsilon>0$ if $\mathbb{E}|\widehat{X}-X^*|\leq \epsilon/2$;
\item (accurate with high probability) We say a querying strategy is $(\epsilon,M)$-accurate for $\epsilon>0$ and $M\geq 2$ ($M$ is not necessarily an integer)
%\nbr{Do we need $M$ to be an integer?}
%\nb{It is not necessary for $M$ to be an integer.}
if $\mathbb{P}\{|\widehat{X}-X^*|> \epsilon/2\}\leq 1/M$.
\end{enumerate}

For any estimator $\widehat{X}$ taking values in $[0,1]$, $\mathbb{E}|\widehat{X}-X^*|\leq \mathbb{P}\{|\widehat{X}-X^*|> \epsilon/2\} + \epsilon/2$. On the other hand it is possible to have $\mathbb{E}|\widehat{X}-X^*|\leq \epsilon/2$ but $\mathbb{P}\{|\widehat{X}-X^*|>\epsilon/2\}>1/2$ however small $\epsilon$ is. Therefore accuracy with high probability is a more stringent constraint on the learner than accuracy on average when $M$ is large compared to $1/\epsilon$. Compared with the noiseless response setting, there is an additional parameter $M$ to be taken into consideration. When designing the querying strategy, the learner needs to control not only the size, but also the probability of error in the presence of noise.

The definition of privacy is the same as that in the noiseless case. A querying strategy is called $(\delta,L)$-private if no adversary's estimator $\widetilde{X}$ can achieve $\mathbb{P}\{|\widetilde{X}-X^*|\leq \delta/2\}>1/L$.

Define the optimal query complexity under accuracy definition (a) as
\[
N_\mathsf{avg}\left(\epsilon,\delta,L\right)=\min\left\{n: \exists\phi\text{ that is }\epsilon\text{-accurate, }(\delta,L)\text{-private and submits at most }n\text{ queries}\right\}.
\]
Similarly define
\[
N_\mathsf{whp}\left(\epsilon,M,\delta,L\right)=\min\left\{n: \exists\phi\text{ that is }(\epsilon,M)\text{-accurate, }(\delta,L)\text{-private and submits at most }n\text{ queries}\right\}.
\]

%then we can argue that the privacy constraint is too weak. Indeed, it would be too stringent to
%require an adversary to estimate $X^*$ almost as accurately as the learner does.
%then we can argue that the privacy constraint is so weak that it becomes trivial. 
%Indeed, suppose the learner adopts the aforementioned bisection querying strategy.
%which is the cheapest algorithm to achieve $\epsilon$-accuracy. 
%Then from the sequence of queries, the adversary 
%can follow the sequence of queries, 
%can deduce the binary responses except for the response to the last query. 
%However, since she is missing the last bit of information, 
%she can achieve $2\epsilon$-accuracy at best. Therefore, if $\delta\leq 2\epsilon$, 
%the learner can simply run the bisection method while satisfying the privacy constraint, achieving the optimal querying complexity $\log(1/\epsilon)$. 
%~\footnote{Here and subsequently $\log$ refers to logarithm with base 2.}

\section{Main results}\label{sec:results}
In this section we present the main results in this paper. When the responses are noiseless, we give almost matching upper and lower bounds on the optimal query complexity in both the Bayesian and deterministic settings. When the responses are noisy, under the Bayesian setting, we give upper and lower bounds that match up to multiplicative constants.

\subsection{Noiseless responses}\label{sec:results.noiseless}
When the responses are noiseless, we discuss the Bayesian and the deterministic settings separately.

\paragraph*{Bayesian setting} We first focus on the Bayesian setting. 
It was shown in~\cite[Proposition~B.2]{tsitsiklis2018private} that if $2\epsilon<\delta\leq 1/L$, then
\begin{equation}
\label{eq:old.bayes.upper}
N(\epsilon,\delta,L) \leq L\log (1/(L\epsilon))+L-1.
\end{equation}
Notice that this upper bound does not change with $\delta$. 
Since the private learning task becomes easier for smaller $\delta$, 
this upper bound is not tight.

Subsequently, it was shown in~\cite[Theorem~2.1]{xu2018query} 
that if $4\epsilon<\delta< 1/L$, then
\[
N(\epsilon,\delta,L)\geq L\log(1/\epsilon)-L\log(2/\delta)-3L\log\log(\delta/\epsilon).
\]
As pointed out by~\cite[Corollary~2.2]{xu2018query}, if $\delta,L$ stay as fixed constants 
while $\epsilon\rightarrow 0$, 
then the above upper and lower bounds imply that the optimal query complexity scales as $L\log (1/\epsilon)$.
However if $\delta\rightarrow 0$ as well and is comparable to $\epsilon$, then the upper and lower bounds above can be quite far apart.

The following is our first main result.  
We relax the $\delta>4\epsilon$ constraint in the lower bound to $\delta\geq 2\epsilon$ and obtain sharper upper and lower bounds that almost match
in the entire parameter regime.
\begin{theorem}[Bayesian setting]\label{thm:bayes}
If $2\epsilon \leq \delta\leq 1/L$, then
\[
\left\lfloor \log \frac{1}{L\delta}\right\rfloor+ L\left(\log\frac{\delta}{\epsilon} -2\right)-1
\leq N(\epsilon, \delta, L) \leq \left\lfloor \log \frac{1}{L\delta}\right\rfloor+ L\left(\left\lceil\log \frac{\delta}{\epsilon}\right\rceil+2\right)-1.
\]
\end{theorem}
The above result captures the impact of the privacy requirement up to an additive gap of $4L$ and has the following nice physical interpretations. 
The $\log \frac{1}{L\delta}$ factor is the number of queries needed for the learner to estimate $X^*$ within an interval of length $L\delta$,
before which the learner does not need to worry about privacy breach by the adversary. 
The $L\log \frac{\delta}{\epsilon}$ factor is due to the fact that the learner needs to submit at least $\log \frac{\delta}{\epsilon}$ queries 
within an $\delta$-length interval containing $X^*$ to fulfill the $\epsilon$-accuracy requirement and the extra multiplicative $L$ factor is the price to pay for
hiding this $\delta$-length interval from the adversary. 

%\nbr{Maybe we can briefly explain why we get $L\log (\delta/\epsilon)$ factor and $\log (1/(L\delta))$ factor}
%\nb{TODO}
If $\delta$ is a constant multiple of $\epsilon$, Theorem~\ref{thm:bayes} implies that the optimal query complexity scales as $\log(1/\epsilon)$ plus a constant multiple of $L$. 
Since the bisection method takes $\log(1/\epsilon)$ queries, 
in this regime the query complexity price to pay for privacy, in terms of $L$,  is additive in $L$, in contrast to 
multiplicative as suggested by~\eqref{eq:old.bayes.upper}.

\paragraph*{Deterministic setting}
Next we shift to the deterministic setting. It was shown in~\cite[Theorem~4.1]{tsitsiklis2018private} that
\begin{equation}
\label{eq:freq.lower}
\max\left\{\log\frac{1}{\epsilon},\log \frac{\delta}{\epsilon}+2L-4\right\} \leq N(\epsilon,\delta,L) \leq \log \frac{1}{L\epsilon}+2L.
\end{equation}
Again the upper bound cannot be tight because it does not vary with $\delta$. 
In the following theorem we sharpen both the upper and lower bounds, shrinking the gap between them to only 8 queries.

\begin{theorem}[Deterministic setting]\label{thm:freq}
If $2\epsilon\leq \delta\leq 1/L$, then
\[
\max\left\{\left\lceil\log\frac{1}{\epsilon}\right\rceil+L-8, \left\lceil \log\frac{\delta}{\epsilon}\right\rceil+2L-4\right\}
\leq N(\epsilon,\delta,L) 
\leq \max\left\{\left\lceil\log\frac{1}{\epsilon}\right\rceil+L, \left\lceil \log\frac{\delta}{\epsilon}\right\rceil+2L\right\}.
\]
\end{theorem}
The above result captures the impact of the privacy requirement up to an additive gap of 8 queries and can be understood intuitively as follows. Note that the by Theorem~\ref{thm:freq}, the optimal query complexity
is approximately $\max\{2L,L+\log(1/\delta)\}+\log(\delta/\epsilon)$.
Recall that to protect privacy in the deterministic setting, the $\delta$-covering 
number of the information set needs to be at least $L$. To this end,
the learner needs to ``plant'' at least $L$ pairs of $\epsilon$-separated queries 
such that each pair is at least $\delta$ away from the others. 
The $\max\{ 2L, L+\log (1/\delta) \} $ factor is the number of queries needed for the learner 
to plant these $L$ pairs of neighboring queries, while  searching for a $\delta$-length interval containing $X^*$. 
The $\log (\delta/\epsilon)$ factor is the number of extra queries that the learner needs to submit within the $\delta$-length interval 
to fulfill the $\epsilon$-accuracy requirement. Adding up these two factors yields the
optimal query complexity established in Theorem \ref{thm:freq}.
%without worrying about privacy breach. 

%When $\delta \le 2^{-L}$, i.e., $\log \frac{1}{\delta} \ge L$, these $L$ pairs of neighboring queries 
%can be ``planted'' while conducting the bisection search for a $\delta$-length containing $X^*$, in a total of $L+\log (1/\delta)$ queries. 
%Afterwards, the learner just needs to submit $\log \frac{\delta}{\epsilon}$ queries within the $\delta$-length interval to fulfill the $\epsilon$-accuracy requirement
%without worrying about privacy breach. 
%When $\delta > 2^{-L}$, i.e., $\log \frac{1}{\delta} < L$, the bisection search is too aggressive to ``plant'' all $L$ pairs of neighboring queries;
%instead the learner 

%\nbr{Similarly as above, maybe we can intuitively explain why we expect the different factors in the optimal query complexity.}
%\nb{TODO}
We can see from Theorem~\ref{thm:freq} that the lower bound in~\eqref{eq:freq.lower} is tight when $\delta\geq 2^{-L}$. However when $\delta< 2^{-L}$, the optimal query complexity is roughly $2L+\log (2^{-L}/\epsilon)=L+\log(1/\epsilon)$; and the query complexity price to pay for privacy, in terms of $L$, is an additive factor of $L$.

\subsection{Multidimensional private learning}\label{sec:multi.dim}
In this section we extend our results in Section~\ref{sec:results.noiseless} to $d$ dimensions for $d>1$. Suppose the true value $X^*$ is in $\mathbb{R}^d$. The closeness of estimators to $X^*$ is measured with respect to the $\|\cdot\|_\infty$ norm, and the accuracy and privacy levels of a querying strategy are defined accordingly. 
By using the $\|\cdot\|_\infty$ norm to measure the adversary's accuracy, we are allowing the adversary to accurately estimate one or some of the coordinates of $X^*$.
That is because in high dimensions, a single coordinate of the model parameter often does not provide meaningful predictive power.
As a result, we only declare privacy breach when the adversary gets ``close" to $X^*$ in $\mathbb{R}^d$. Here we use the $\|\cdot\|_\infty$ norm to measure closeness. But we comment that as a consequence of our result, if the Euclidean norm were used, the complexity would only differ by a multiplicative constant.

We assume that the learner is only allowed to ask questions of the type ``is $X_i^*\geq q$?" for some $i\in [d]$ and $q\in [0,1]$. 
Denote the optimal query complexity in $d$-dimensions as $N_d(\epsilon,\delta,L)$. Below we state our results for the Bayesian and deterministic settings. The proofs are contained in the appendix. 
\paragraph*{Bayesian setting}
% maybe no need to repeat the definitions of eps-accuracy and (delta,L)-privacy?
Suppose $X^*$ is uniformly distributed on $[0,1]^d$. We say a querying strategy $\phi$ is
\begin{itemize}
\item $\epsilon$-accurate for $\epsilon>0$, if
$
\mathbb{P} \{ \|\widehat{X}-X^* \|_\infty \leq \epsilon/2 \}=1;
%\mathbb{P}\left\{\left\|\widehat{X}-X^*\right\|_\infty\leq \epsilon/2\right\}=1;
$
\item $(\delta,L)$-private for $\delta>0$ and an integer $L\geq 2$, if there is no adversary $\widetilde{X}$ such that
$
\mathbb{P}\{\|\widetilde{X}-X^*\|_\infty\leq \delta/2\}>1/L.
%\mathbb{P}\left\{\left\|\widetilde{X}-X^*\right\|_\infty\leq \delta/2\right\}>1/L.
$
\end{itemize}
%\nbr{We may want to briefly explain why we choose $\|\|_\infty$ norm in view of ICML reviews}
%\nb{A discussion on the metric is added to the start of this subsection.}
We focus on the parameter regime $2\epsilon\leq \delta \leq 1/\lceil L^{1/d}\rceil$. We have argued in Section~\ref{sec:problem} why $2\epsilon\leq\delta$ is reasonable to assume. To justify the other end of the spectrum, note that if $\delta>1/L^{1/d}$, then the naive estimator $\widetilde{X}=1/2$ achieves $\mathbb{P}\{\|\widetilde{X}-X^*\|_\infty\leq \delta/2\}=\delta^d>1/L$, making it impossible to fulfill the privacy constraint.

Denote $\Ld=\Ld(L,d)=L^{1/d}$. Below is our main result on the multidimensional optimal query complexity in the Bayesian setting.

\begin{theorem}[Bayesian setting]\label{thm:bayes.ddim}
If $2\epsilon\leq \delta \leq 1/\lceil \Ld\rceil$, then
\[
N_d(\epsilon,\delta,L)\leq d\left(\left\lfloor \log\frac{1}{\lceil \Ld\rceil \delta}\right\rfloor + \left\lceil \Ld\right\rceil \left(\left\lceil \log\frac{\delta}{\epsilon}\right\rceil+2\right)-1\right).
\]
Furthermore, assuming that the queries on $X_i^*$ depend only on the responses to the previous queries on $X_i^*$ and some random seed $Y_i$, with $Y_1,...,Y_d$ mutually independent, then
\[
N_d(\epsilon,\delta,L)\geq d\left(\left\lfloor \log\frac{1}{\Ld \delta}\right\rfloor + \Ld\left(\log\frac{\delta}{\epsilon}-2\right)-1\right).
\]
\end{theorem}

\paragraph*{Deterministic setting}
Suppose $X^*\in [0,1]^d$ is deterministic, we say a querying strategy $\phi$ is
\begin{itemize}
\item $\epsilon$-accurate for $\epsilon>0$, if 
$
\mathbb{P}\{ \|\widehat{X}-X^* \|_\infty\leq \epsilon/2 \}=1$
for all $X^* \in [0,1]^d$;
%\[
%\mathbb{P}\left\{ \left\|\widehat{X}-X^* \right\|_\infty\leq \epsilon/2 \right\}=1, \quad \forall X^* \in [0,1]^d;
%\]
\item $(\delta,L)$-private for $\delta>0$ and an integer $L\geq 2$, if for each query sequence $\bar{q}$, the $\delta$-covering number
of the information set $\mathcal{I}(\bar{q})$ is at least $L$. \footnote{Here the $\delta$-covering number is defined in terms of the $\|\cdot\|_\infty$ norm in $\mathbb{R}^d$.
%, i.e., the size of the smallest set $\mathcal{N}\subset [0,1]^d$ such that for each $x\in \mathcal{I}(\bar{q})$, there exists some $r\in \mathcal{N}$ for which $|r_i-x_i|\leq \delta/2$ for all $i\in [d]$.
}
\end{itemize}
\begin{theorem}[Deterministic setting]\label{thm:freq.ddim}
If $2\epsilon\leq \delta \leq 1/\lceil \Ld\rceil$, then
\[
d\left(2\Ld + \log\frac{\max\{2^{-\lceil \Ld\rceil}, \delta\}}{\epsilon}-8\right)\leq N_d(\epsilon,\delta,L) \leq d\left(2\left\lceil \Ld\right\rceil + \left\lceil \log\frac{\max\{2^{-\lceil \Ld\rceil}, \delta\}}{\epsilon}\right\rceil+1\right).
\]
\end{theorem}

From the upper and lower bounds in the theorem statements, we see that in $d$-dimensions the optimal query complexity suffers from a multiplicative factor of $d$. This is consistent with the optimal query complexity $d\log (1/\epsilon)$ when there is no privacy constraint. The query complexity for each dimension depends on $\Ld=L^{1/d}.$ As $d$ grows, the price to pay for privacy per dimension decreases. In the extreme case where $d\rightarrow \infty$ with $L$ fixed, the optimal query complexity behaves like $d\log (1/\epsilon)$ in both the Bayesian and the deterministic settings, making the privacy constraint obsolete in high dimensions.

One interesting direction to strengthen Theorem~\ref{thm:bayes.ddim} and Theorem~\ref{thm:freq.ddim} is to allow the learner to query ``is $X^*$ in $H$?" where $H$ is an arbitrary half-space in $\mathbb{R}^d$. The upper bounds are still valid since every comparison query corresponds to a half-space. However for both the Bayesian and the deterministic setting, our current lower bound proof strategies do not accommodate this wider class of queries. This variant of the problem was studied by~\cite[Theorem~2]{xu2019query} under the Bayesian setting, where the author gives a lower bound of $c_1\delta^{d-1}L\log(\delta/\epsilon)-c_2 L$ for constants $c_1,c_2$ that depend on $d$. The lower bound is obtained via a hyperplane transversality argument. More specifically, divide $[0,1]^d$ into $\delta$-wide cubes and consider an adversary who samples from the cubes that intersect with the queried hyperplanes. The maximum number of cubes each hyperplane can intersect with grows like $\delta^{-(d-1)}$, resulting in the $\delta^{d-1}$ factor in the lower bound. However this dependence on $\delta$ and the dimension is far from desirable. We conjecture that allowing the learner to query arbitrary half-spaces does not help lower the query complexity, and that the querying strategies we construct for the proofs of Theorem~\ref{thm:bayes.ddim} and Theorem~\ref{thm:freq.ddim} remain optimal.

\subsection{Noisy responses}
The following theorem is our main result when the responses are noisy. Recall that each response is corrupted with probability $1-p$. In short, we are able to characterize the optimal query complexities up to constants that only depend on $p$.
\begin{theorem}\label{thm:noisy}
Assume that $4\epsilon\leq \delta\leq 1/L$. Then
\begin{equation}
\label{eq:res.mean}
\frac{1}{2c_2(p)}\max\left\{L\log\frac{\delta}{16\epsilon}, \; \log\frac{1}{8\epsilon}\right\}
\leq N_\mathsf{avg}\left(\epsilon,\delta,L\right)
\leq \left(\frac{14}{c_3(p)}+\frac{7}{c_4(p)}\right)\left(\log \frac{1}{\epsilon} + L\log\frac{64\delta}{\epsilon}\right),
\end{equation}
and
\begin{align}
\nonumber& \max\left\{\frac{L}{2c_2(p)}\log \frac{\delta}{8\epsilon}, \;
\frac{1}{2c_2(p)}\log \frac{1}{4\epsilon}, \;
\frac{L}{2c_1(p)}\log\frac{M}{8}\right\}\\
 \label{eq:res.prob} & \leq  N_\mathsf{whp}\left(\epsilon,M, \delta,L\right)
\leq \left(\frac{8}{c_3(p)} + \frac{7}{c_4(p)}\right)\left( \log\frac{1}{\epsilon} + L\log\frac{12 M\delta}{\epsilon}\right),
\end{align}
where 
\begin{align*}
c_1(p) = & D(\text{Bern}(1-p)||\text{Bern}(p))=(1-p)\log\frac{1-p}{p}+p\log\frac{p}{1-p},\\
c_2(p) = & h(1/2)-h(p),\;\;\;\text{with }h(p)=H(\text{Bern}(p)) = -p\log p - (1-p)\log (1-p),\\
c_3(p) = & (p-1/2)^2 \log e,\\
c_4(p) = & D(\text{Bern}(1/2)||\text{Bern}(p))=\tfrac{1}{2}\left(\log\frac{1}{2p}+\log\frac{1}{2(1-p)}\right)
\end{align*}
are constants that only depend on $p$.
\end{theorem}
%\nbr{Do we really need to keep $1-1/M$ in the theorem statement? If $M\ge 2$, then $1-1/M\ge 1/2$.}
%\nb{The theorem is now restated with $1-1/M$ replaced with $1/2$. The proofs are updated accordingly.}
%[Comment on the slightly different parameter regime? It seems necessary for proving $N_b\gtrsim L\log M$.]

It follows from~\eqref{eq:res.mean} and~\eqref{eq:res.prob} and the basic inequality $\max\{a,b\}\geq (a+b)/2$ that there exist constants $c_5,c_6,c_7,c_8$ that only depend on $p$, such that
\[
c_5\left(L\log\frac{\delta}{\epsilon} + \log\frac{1}{\epsilon}\right) \leq N_\mathsf{avg}(\epsilon,\delta,L) \leq c_6\left(L\log\frac{\delta}{\epsilon} + \log\frac{1}{\epsilon}\right),
\]
\[
c_7\left(L\log\frac{M\delta}{\epsilon} + \log\frac{1}{\epsilon}\right) \leq N_\mathsf{whp}(\epsilon,M,\delta,L)\leq c_8\left(L\log\frac{M\delta}{\epsilon} + \log\frac{1}{\epsilon}\right).
\]
As we mentioned in the Section~\ref{sec:problem}, definition (b) is a stronger condition on the learner's accuracy than (a) when $M$ is large. As a result an extra additive factor of order $L\log M$ shows up in $N_\mathsf{whp}(\epsilon,M,\delta,L)$. The intuition behind the $L\log M$ factor is that there must be at least $\Omega(\log M)$ queries near $X^*$ to achieve an error probability of $1/M$; these queries then need to be duplicated $L$ times to disguise the location of $X^*$ from the adversary.
%\nbr{Maybe we can briefly explain why we expect a $L\log M$ factor. Loosely speaking, I guess this is because we need
%$\log M$ queries to achieve an error probability of $1/M$ in view of Chernoff's bound.}
%\nb{The sentence above is added to provide some intuition on this factor.}

Compared with the noiseless response setting where we obtained tight control on the optimal query complexity, the results for the noisy response case are only up to multiplicative constants that depend on $p$. We remark that even for the vanilla noisy binary search problem without privacy consideration, the precise $p$-dependent constant remains an open problem.

When $p=1/2$ the problem becomes completely noisy and the responses provide no information. As $p\rightarrow 1/2$, the optimal query complexity should go to infinity, which is reflected by the fact that $c_1,c_2,c_3,c_4$ all converge to 0. In fact, Taylor expansions around $p=1/2$ reveal that $c_i(p) \asymp (p-1/2)^2$ for all $i\in [4]$. As a result, the constants $c_5,c_6,c_7,c_8$ go to infinity at the same rate as $p\rightarrow 1/2$.

\section{The querying strategies with noiseless responses}\label{sec:alg.noiseless}
In this section we describe our querying strategies when the responses do not contain noise, and argue heuristically why our constructions lead to the optimal query complexity. Section~\ref{sec:proof} contains the full proofs of our results, where these heuristic arguments are made precise. 
%As before we discuss the Bayesian and deterministic settings separately.

\subsection{The Bayesian setting}\label{sec:alg.bayes}
When constructing a querying strategy, we want it to possess the merits of accuracy, meaning that the learner can learn $X^*$ well with probability one;  privacy, meaning that the adversary cannot learn $X^*$ well with probability greater than $1/L$; and efficiency, meaning that the strategy submits as few queries as possible. As discussed in the Introduction, the bisection method is the most efficient, but not private. The grid search, on the contrary, is almost completely private, but very inefficient.

In~\cite{tsitsiklis2018private} a querying strategy named \emph{replicated bisection} is proposed, which can be viewed as a combination of the grid search and bisection. First the learner divides $[0,1]$ into $L$ subintervals $I_1,...,I_L$ of equal length, and queries all the endpoints to determine the subinterval $I_{i^*}$ that contains $X^*$. She then runs a bisection search on $I_{i^*}$, submitting replicated queries in all the other subintervals at the same time. The replicated bisection method is private since the adversary cannot discern which subinterval contains $X^*$ without observing the responses. However,  it can be vastly inefficient, as its query complexity increases multiplicatively in $L$ due to the replication.

Notice that the replicated bisection method is not adapted to $\delta$, even though the learner is expected to take fewer queries to achieve $(\delta,L)$-privacy when $\delta$ is small. In fact when $\delta$ is small, it is rather wasteful and unnecessary to repeat bisection searches on subintervals of length $1/L$. It was shown in~\cite{tsitsiklis2018private} that the replicated bisection strategy incurs a query complexity $L\log (1/(L\epsilon))+L-1$, which entails that the cost for privacy under replicated bisection is roughly a multiplicative factor of $L$ regardless of the range of $\delta$. Next we present our $\delta$-adaptive querying strategy, which is much more efficient for small values of $\delta$. In particular, we discover that when $\delta$ is proportional to $\epsilon$, the cost for privacy should be additive in $L$ instead of multiplicative.

By the definition of $(\delta,L)$-privacy, the learner only needs to safeguard $X^*$ against the adversary making a $\delta$-accurate guess. In this sense, we want the querying strategy to create $L$ subintervals of length $\delta$, which  are equally likely to contain $X^*$ from the adversary's perspective. To that end, we construct the following multistage querying strategy (precise description in Algorithm~\ref{alg:bayes} in the appendix):
\begin{enumerate}
\item Run bisection search on $[0,1]$ for $K_1$ steps to locate $X^*$ within an interval $I$ of length $2^{-K_1}\approx L\delta$;
\item Divide $I$ into $L$ subintervals $I_1,..., I_L$ of equal length (about $\delta$). Query the $L-1$ endpoints of all the subintervals (the two endpoints of $I$ were already queried in stage 1) to determine which subinterval contains $X^*$.
\item Say $I_{i^*}$ is the true subinterval which contains $X^*$. Run bisection search on the $I_{i^*}$ for $K_2$ steps until $\epsilon$-accuracy is achieved, while submitting cloned queries in the other $L-1$ subintervals in parallel. 
\end{enumerate}
See Fig.~\ref{fig:bayes.strategy} for a graphical illustration.
Since it takes about $\log(\delta/\epsilon)$ for the bisection on an interval of length $\delta$ to reach accuracy $\epsilon$, the total number of queries submitted is roughly $\log\frac{1}{L\delta} + (L-1) + L\log\frac{\delta}{\epsilon}$. See Section~\ref{sec:proof.bayes} for the full proof that the multistage querying strategy achieves the upper bound in Theorem~\ref{thm:bayes}.

%\begin{center}
\begin{figure}[h]
\includegraphics[width = \textwidth]{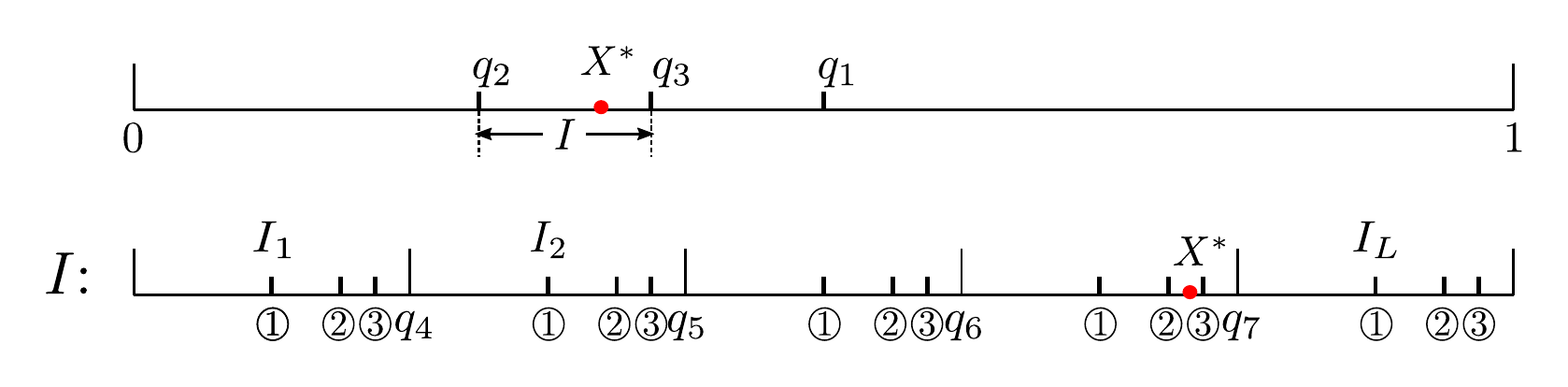}
\captionof{figure}
{An example of the querying strategy with $L=5$, $K_1=3$, $K_2=3$ under the Bayesian setting. The learner first runs $K_1$ steps of bisection to locate $X^*$ within $I$. Divide $I$ into $L$ equal length subintervals $I_1,...,I_L$. By querying the endpoints $q_4,...,q_7$ of the subintervals, the learner locates the subinterval that contains $X^*$, in this case $I_4$. She then proceeds to submit $K_2$ batches of queries. The first, second and third batches of queries submitted are labeled \textcircled{\raisebox{-0.9pt}{1}}, \textcircled{\raisebox{-0.9pt}{2}}, \textcircled{\raisebox{-0.9pt}{3}} respectively. On $I_4$, the queries are submitted via bisection while clones are submitted on the other subintervals in parallel.}
\label{fig:bayes.strategy}
\end{figure}
%\end{center}

The optimal querying strategy for when $X^*\in \mathbb{R}^d$ is based on the one-dimensional strategy. Since accuracy of the adversary is measured in terms of the $\|\cdot\|_\infty$ norm, the learner only needs to run replicated bisection on $L^{1/d}$ subintervals in each dimension to ensure that $\mathbb{P}\{\|\widetilde{X}-X^*\|_\infty\leq \delta\}\leq (L^{-1/d})^d=1/L$. See Section~\ref{sec:d.dim} for details.

\subsection{The deterministic setting}\label{sec:alg.freq}
Recall that under the deterministic setting, a querying strategy is called $(\delta, L)$-private if for each query sequence $\bar{q}$, the $\delta$-covering number of the information set $\mathcal{I}(\bar{q})$ is at least $L$. Compared with the Bayesian setting, the privacy requirement is weaker, as the learner only needs to create $L$ possible locations for $X^*$ that are at least $\delta$ apart. 

To achieve $(\delta,L)$-privacy, a querying strategy named \emph{opportunistic bisection} is proposed in~\cite[Theorem~4.1]{tsitsiklis2018private}. First the learner submits $L$ pairs of queries $((i-1)/L, (i-1)/L+\epsilon)_{i=1,...,L}$, known as $L$ {\it guesses}.
%. The authors of~\cite{tsitsiklis2018private} referred to this step as submitting $L$ {\it guesses}. 
For each guess submitted, the learner is effectively testing a hypothesis $X^*\in I_i=[(i-1)/L, (i-1)/L+\epsilon)$. If none of the guesses is correct, then the learner proceeds to run a bisection search on the $1/L$-length subinterval that contains $X^*$. If one of the guesses is correct, say $X^*\in I_{i^*}$, then the accuracy requirement is already achieved by taking $\widehat{X}$ to be the midpoint of $I_{i^*}$. However to disguise this finding from the adversary, the learner runs a ``fake" bisection search on a randomly selected $1/L$-length subinterval. By ``fake" bisection we mean a simulated bisection search where the binary responses are generated {\it i.i.d.} Bernoulli(1/2). The opportunistic bisection method is private since the adversary cannot rule out the possibility that $X^*\in I_i$ for some $i=1,...,L$. In other words, the information set $\mathcal{I}(\bar{q})$ always contains $\cup_{i\leq L}I_i$. Hence it has $\delta$-covering number no less than $L$. 

First notice that just like replicated bisection, the opportunistic bisection is not adapted to $\delta$. As a result it does not enjoy savings when $\delta$ is small. Furthermore, the guesses are submitted on a grid, which is highly inefficient. Ideally, we want the guesses to not only help conceal the location of $X^*$ from the adversary, but also help the learner locate $X^*$ at the same time. 

Naturally, the most efficient way to submit the guesses is via a bisection search. The problem with bisection is that when $\delta$ is large, some of the $L$ guesses may not be $\delta$ apart from each other. In a way, the bisection search is too aggressive when $\delta>2^{-L}$, which calls for a more sophisticated query sequence construction. Below is the construction of our $\delta$-adaptive querying strategy:

\begin{enumerate}
\item Submit $L$ guesses that are at least $\delta$ apart. This further breaks down into two cases, depending on the value of $\delta$:

If $\delta\leq 2^{-L}$, submit the guesses via a bisection search. That is, the first guess is at $1/2$, the second guess is at $1/4$ if $X^*<1/2$ and at $3/4$ otherwise, etc. However if at any point a guess turns out to be correct ($X^*\in [s,s+\epsilon)$ for a guess at $s$), then in order to hide this knowledge from the adversary, the learner keeps submitting guesses via a fake bisection search using random responses distributed {\it i.i.d.} Bernoulli(1/2).

If $\delta>2^{-L}$, submit the first guess at 0. The next $K$ guesses are submitted via a bisection search, locating $X^*$ in a interval $I$ of length $2^{-K}$. As in the previous case, transition into a fake bisection search whenever a guess is found to contain $X^*$. Submit the rest of the $(L-K-1)$ guesses through a grid search on $I$. Here $K$ is chosen to be an integer in $\{0,1,...,L-1\}$ for which $2^{-K}/(L-K)\in [\delta,2\delta]$\footnote{We prove in Section~\ref{sec:proof.freq} in the appendix that such $K$ always exists. In fact the initial guess at 0 is to ensure existence of an integer solution for $K$.}. In this way the closest pair of guesses are made as close as possible, while still being at least $\delta$ apart.

\item If none of the $L$ guesses made in stage 1 is correct, then through the $L$ guesses the learner should locate $X^*$ within an interval $J$ of length about $\max\{2^{-L}, \delta\}$. Run a bisection search on $J$ until $\epsilon$-accuracy is reached. If any of the guesses is correct, replace this step with a fake bisection search on a simulated interval $J$. When $\delta\leq 2^{-L}$, $J$ is obtained from the last step of the fake bisection search in stage 1;  when $\delta>2^{-L}$, $J$ is selected from the $L-K$ subintervals of $I$ uniformly at random.
% \nbr{Here it is unclear which previous step you're referring to}.
\end{enumerate}

Examples of the above querying strategy is illustrated in Fig.~\ref{fig:freq_small_delta} (when $\delta\leq 2^{-L}$) and Fig.~\ref{fig:freq_large_delta} (when $\delta>2^{-L}$). Since each guess contains 2 queries, the first stage involves $2L$ queries. The total number of queries submitted under our querying strategy is roughly $2L + \log (\max\{2^{-L},\delta\}/\epsilon)$. See Section~\ref{sec:proof.freq} in the appendix for the precise descriptions of our querying strategy and the proof of the upper bound in Theorem~\ref{thm:freq}.

%\begin{center}
\begin{figure}[h]
\includegraphics[width = \textwidth]{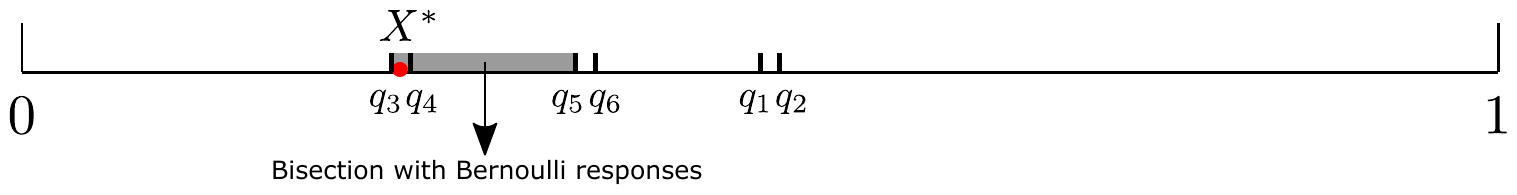}
\captionof{figure}{An example of the querying strategy under the deterministic setting when $\delta\leq 2^{-L}$, with $L=3$. From the response to the first four queries the learner deduces that $X^*$ is between $q_3$ and $q_4=q_3+\epsilon$. The learner proceeds to run a ``fake'' bisection in $[q_3,q_5)$ by generating Bernoulli responses to confuse the adversary. From the perspective of the adversary, $X^*$ could be in any of the three length-$\epsilon$ subintervals.}
\label{fig:freq_small_delta}
\end{figure}
%\end{center}

%\begin{center}
\begin{figure}[h]
\includegraphics[width = \textwidth]{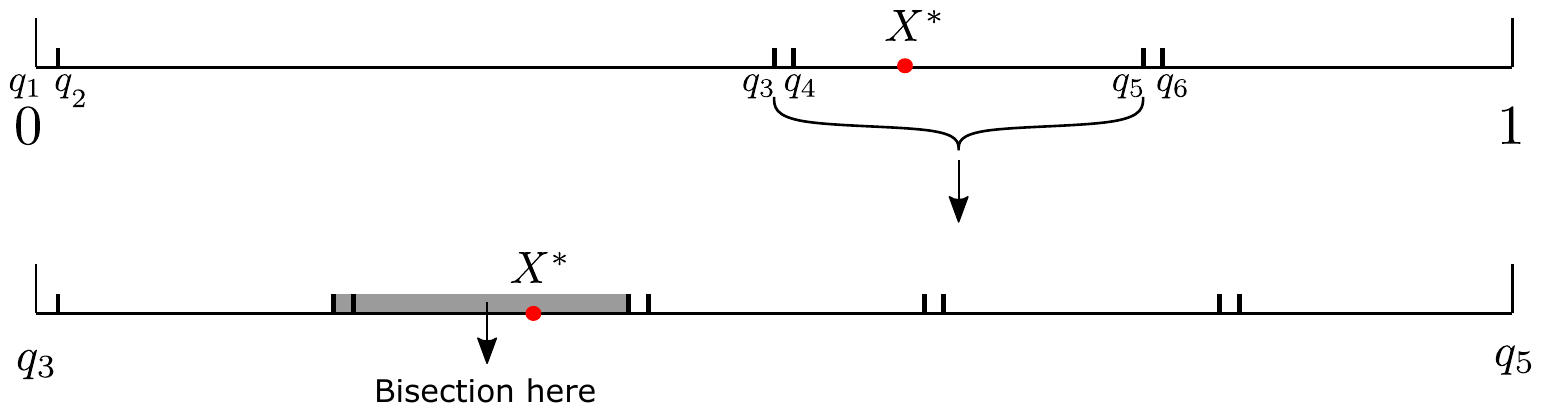}
\captionof{figure}{An example of the querying strategy under the deterministic setting when $\delta>2^{-L}$, with $L=7$ and $K=2$. The first guess is at 0. The learner submits the next $K$ guesses via bisection to locate $X^*$ in $[q_3,q_5)$. She then partitions $[q_3,q_5)$ into $L-K$ equal length subintervals. Eventually through bisection the learner is able to approximate $X^*$ up to accuracy $\epsilon$. But from the perspective of the adversary, $X^*$ could be in $[q_{2i-1},q_{2i})$ for any $i\leq L$.}
\label{fig:freq_large_delta}
\end{figure}
%\end{center}

When $X^*\in \mathbb{R}^d$, each {\it guess} corresponds to a cube in $\mathbb{R}^d$ of diameter $\delta$. By running the one-dimensional algorithm submitting $L^{1/d}$ guesses in each dimension, the learner forms $L$ guesses in $\mathbb{R}^d$. See Section~\ref{sec:d.dim} for the detailed description of the querying strategy in $d$ dimensions.

\section{Lower bound proof strategies with noiseless responses}\label{sec:lower.noiseless}
In this section, we sketch our lower bound proof strategies for Theorem~\ref{thm:bayes} and Theorem~\ref{thm:freq}. The rigorous proofs are deferred to Section \ref{sec:proof}. The multidimensional lower bound proofs are developed in Section~\ref{sec:d.dim} and build upon the arguments in this section.

\subsection{Bayesian setting}\label{sec:lower.bayes}
The lower bound is shown by constructing an intelligent adversary so that the learner cannot disguise the location of $X^*$ without a certain number of queries. The adversary strategy considered in~\cite{xu2018query} is called \emph{proportional-sampling}, which samples from all the queries proportionally. In particular, given an observed query sequence $q_1,...,q_n$, the proportional-sampling estimator is defined as $\widetilde{X}=q_J$, where $J\sim \text{Unif}\{1,...,n\}$.

We argue that the proportional sampling is not the optimal estimation strategy for the adversary. Since the first few queries are very unlikely to be close to $X^*$, the adversary suffers unnecessary loss whenever those queries are sampled. Instead we propose \emph{truncated proportional-sampling} scheme: disregard the first $K= \lfloor\log(1/(L\delta))\rfloor$ queries and proportionally sample from $q_{K+1}, ..., q_n$. Notice here a subtle but interesting duality between the learner and the adversary: the number of queries disregarded by the adversary $K$ is exactly the number of queries submitted in stage 1 of the learner's optimal querying strategy. However, in the proof of our lower bound, this argument with discarded queries is effective against any learner strategy, regardless of what the learner tries to achieve with these early queries. 

%Even though the adversary seems to throw away useful information, since  the first $K$ queries are very unlikely to be close to $X^*$, 
%it follows that $n>K$ and the adversary in fact achieves more accurate estimation of $X^*$. 

For any querying strategy that is $(\delta,L)$-private, it must satisfy $\mathbb{P}\{|\widetilde{X}-X^*|\leq \delta/2\}\leq 1/L$ for all adversary strategies. Suppose the adversary's estimator $\widetilde{X}$ is obtained through the truncated proportional-sampling, then we have

\begin{equation}\label{eq:prop.sampling}
\mathbb{P}\left\{|\widetilde{X}-X^*|\leq \delta/2\right\}=\frac{\sum_{i=K+1}^n\mathbb{P}\{|q_i-X^*|\leq \delta/2\}}{n-K}.
\end{equation}
The numerator can  be interpreted as the expected number of queries among $q_{K+1}, ..., q_n$ that are in $I=[X^*-\delta/2, X^*+\delta/2]$. Loosely speaking, since $I$ is a length-$\delta$ interval, the learner needs to submit at least $\log(\delta/\epsilon)$ queries in $I$ to estimate $X^*$ within $\epsilon$-accuracy. Out of these queries, we show that they are almost all taken from $q_{K+1}, ..., q_n$, because with high probability the learner simply wouldn't have enough information during the first $K$ queries to locate a small enough neighborhood of $X^*$. As a result, $\log(\delta/\epsilon)$ roughly serves as a lower bound for the numerator in~\eqref{eq:prop.sampling}. Deduce from $\mathbb{P}\{|\widetilde{X}-X^*|\leq \delta/2\}\leq 1/L$ that $n\geq K+L\log(\delta/\epsilon)$, which only differs from the precise lower bound in Theorem~\ref{thm:bayes} by an additive factor of $2L+1$. 
See Section~\ref{sec:proof.bayes} in the appendix for a rigorous lower bound proof.

%\nbr{Somewhere here we may want to briefly emphasize
%the duality between the learner and the adversary:
%the number of queries to disregard is exactly the number of queries submitted in stage 1 of the learner’s optimal strategy. 
%}
%\nb{A sentence is added after the description of truncated proportional-sampling that points out this duality.}

\subsection{Deterministic setting}\label{sec:lower.freq}

As in the upper bound proof, we discuss the $\delta>2^{-L}$ case and $\delta\leq 2^{-L}$ separately.

{\bf Case 1: $\delta>2^{-L}$}. We only need a lower bound of $\lceil\log(\delta/\epsilon)\rceil+2L-4$ since it is always above $\lceil \log(1/\epsilon)\rceil+L-8$ in this regime. We adopt the lower bound shown in~\cite[Theorem~4.1]{tsitsiklis2018private}
$
N(\epsilon,\delta,L)\geq 2L+\log\frac{\delta}{\epsilon}-4,
$
which when $\delta>2^{-L}$, is almost tight in view of the upper bound Theorem~\ref{thm:freq}.
For completeness we sketch their proof here. Fix an $\epsilon$-accurate and $(\delta,L)$-private querying strategy $\phi$ and let $I=[0,\delta]$. 
On the one hand, for some $X^*\in I$, there are at least $\log(\delta/\epsilon)$ queries in $I$ by the optimality of the bisection search method. 
On the other hand, note that for a point $x$ to belong to the information set $\mathcal{I}(\bar{q})$, there must be two queries that are at most $\epsilon$ apart on opposite sides of $x$; otherwise, the learner cannot be $\epsilon$-accurate. 
Since the $\delta$-covering number of the information set is at least $L$, there are at least $L$ pairs of queries that are at most $\epsilon$ apart. The interval $I$ is of length only $\delta$, so almost all these $L$ pairs of queries are outside of $I$, yielding a total of roughly $2L+\log (\delta/\epsilon)$ queries.

This proof strategy however, cannot yield a tight lower bound when $\delta$ is small. By fixing $I=[0,\delta]$, it disregards the cost of finding the interval $I$ of length $\delta$ that contains $X^*$. Since the learner does not know $X^*\in I$ a priori, she would need to submit more than $2L$ queries outside of $I$. 
When $\delta\leq 2^{-L}$, proving a tight lower bound for the number of queries outside of $I$ turns out to involve much more sophisticated analysis.

{\bf Case 2: $\delta\leq 2^{-L}$}. The lower bound of $\lceil\log(\delta/\epsilon)\rceil+2L-4$ continues to hold when $\delta$ is small. Thus it suffices to show there is at least one true value $X^*\in[0,1]$ for which the learner submits at least
$L+\lceil\log(1/\epsilon)\rceil-8$ queries. 
On a high level, we prove this lower bound by finding an interval $I$ roughly of length $\delta$ and $X^*\in I$ such that when $X^*$ is the true value, there are at least $\log(\delta/\epsilon)$ queries in $I$ and $L+\log(1/\delta)-8$ queries outside of $I$. 
%Within $I$, it takes at least $\log (\delta/\epsilon)$ queries to approximate $X^*$ up to $\epsilon$ accuracy. 
To prove the lower bound for the number of queries outside of $I$, we show that there are roughly $\log(1/\delta)$ queries that are at least $\delta$ away from each other. Moreover, there are at least roughly $L$ pairs of queries that are at most $\epsilon$ apart to ensure that the $\delta$-covering number of the information set is at least $L$, which contribute around $L$ extra queries outside of $I$. Hence there are at least around $\log(\delta/\epsilon)+\log(1/\delta)+L=\log(1/\epsilon)+L$ queries needed by $\phi$. The key challenge lies in showing existence of such an interval $I$, as taking $I=[0,\delta]$ no longer works.

\section{The analysis with noisy responses}\label{sec:alg.noisy}
In this section we give the construction of the querying strategies that achieve the upper bounds in Theorem~\ref{thm:noisy} when the responses are noisy, and discuss the lower bound proof techniques. Our querying strategies relies heavily on an existing search algorithm known as the Burnashev-Zigangirov(BZ) algorithm~\cite{burnashev1974interval}. For completeness, we give in Section~\ref{sec:BZ} a brief description of the BZ algorithm and its statistical properties.

\subsection{Background: the Burnashev-Zigangirov algorithm}\label{sec:BZ}
Suppose $[0,1]$ is divided into $1/\Delta$ (assumed to be an integer) equal length subintervals, labeled $I_1,...,I_{1/\Delta}$ from left to right. Let $J$ denote the subinterval that contains the true value $X^*$. The BZ algorithm is a selection procedure that returns $\widehat{J}$, an estimator of $J$. 

Since $X^*$ is distributed uniformly on $[0,1]$, the algorithm starts from a uniform distribution $\mu_1$ on $[0,1]$, which can be viewed as a priori belief distribution on the location of $X^*$. Each time the learner observes a response $R_j$ to a query $X_j$, where $R_j\sim \text{Bernoulli}(p)$ if $X^*\geq X_j$ and 
$R_j \sim \text{Bernoulli}(1-p)$ if $X^*<X_j$ for some $p\in (1/2,1)$. Then the belief distribution is updated as follows:
\[
\frac{d\mu_{j+1}}{d\mu_j}(x) = 
\begin{cases}
\frac{2(1-\alpha)\mathds{1}\{x\in [0,X_j)\} + 2\alpha\mathds{1}\{x\in [X_j,1]\}}{\int \left(2(1-\alpha)\mathds{1}\{y\in [0,X_j)\} + 2\alpha\mathds{1}\{y\in [X_j,1]\}\right) \mu_{j}(dy)} & \text{if }R_j=1;\\
\frac{2\alpha\mathds{1}\{x\in [0,X_j)\} + 2(1-\alpha)\mathds{1}\{x\in [X_j,1]\}}{\int \left(2(1-\alpha)\mathds{1}\{y\in [0,X_j)\} + 2\alpha\mathds{1}\{y\in [X_j,1]\}\right) \mu_{j}(dy)} & \text{if }R_j=0,
\end{cases}
\]
%\nbr{JX. Is there any reason to put factor of $2$ in the above expression?}
where $\alpha\in (1/2,p)$ is a parameter whose value will be later specified. Note that if $\alpha=p$, the display above is exactly the posterior update rule for the distribution of $X^*$ given the responses. The intuition behind choosing $\alpha < p$ is to tilt the update rule in the more conservative direction, so that the effect of ``incorrect" responses can be mitigated.

The query $X_j$ is selected to be close to the median of $\mu_j$. Specifically, if $I_j = [s,t)$ is the subinterval that contains the median of $\mu_j$, then $X_j$ is chosen to be the left endpoint $s$ of $I_j$ with probability $\pi_1 = (\mu_j[0, t) - \mu_j[t, 1])/(2\mu_j(I_j))$, and $X_j=t$ with probability $\pi_2=1-\pi_1 = (\mu_j[s,1] - \mu_j[0,s))/(2\mu_j(I_j))$. Here $\pi_1$ and $\pi_2$ are chosen so that the conditional mean of $X_j$ is exactly the median of $\mu_j$. Since the learner only queries the endpoints of the subintervals, the density of $\mu_j$ is a piecewise-constant function whose change points can only occur at the endpoints of the subintervals. Suppose $n$ queries are submitted, the estimator $\widehat{J}$ is taken to be the subinterval with the highest $\mu_{n+1}$ density, ties broken arbitrarily.

For simplicity write $\overline{p}=1-p$, $\overline{\alpha} = 1-\alpha$. It has been shown that the error probability of $\widehat{J}$ decreases exponentially in the number of queries~\cite[Eq~(3.24)]{burnashev1974interval}:
\begin{equation}
\label{eq:BZ}
\mathbb{P}\{X^*\notin \widehat{J}\}\leq \frac{1-\Delta}{\Delta}\left[\frac{\overline{p}}{2\overline{\alpha}}+\frac{p}{2\alpha}\right]^n \leq \frac{1}{\Delta}\left[\frac{\overline{p}}{2\overline{\alpha}}+\frac{p}{2\alpha}\right]^n.
\end{equation}
The factor $\overline{p}/(2\overline{\alpha})+p/(2\alpha)$ is minimized at 
\[
\alpha = \frac{\sqrt{p}}{\sqrt{p}+\sqrt{\overline{p}}},\;\;\;\text{with }\frac{\overline{p}}{2\overline{\alpha}}+\frac{p}{2\alpha}=\frac{1}{2} + \sqrt{p\overline{p}}.
\]
%\nbr{JX. The above expressions are a bit overly complicated. Basically, $\alpha=\frac{\sqrt{p}}{\sqrt{p}+\sqrt{\bar{p}}}$ and 
%$\frac{\overline{p}}{2\overline{\alpha}}+\frac{p}{2\alpha}=\frac{1}{2} + \sqrt{p\bar{p}} \le 1 - (p-1/2)^2$,
%where the last inequality holds because $\sqrt{x(1-x)} \le 1/2-(x-1/2)^2$ for $x\in[0,1]$.
% }
It follows from~\eqref{eq:BZ} that
\begin{equation}
\label{eq:BZ.clean}
\mathbb{P}\{X^*\notin\widehat{J}\} \leq \frac{1}{\Delta}\left(\frac{1}{2}+\sqrt{p\overline{p}}\right)^n\leq \frac{1}{\Delta}\left(1-(p-1/2)^2\right)^n\leq \frac{1}{\Delta}\exp\left(-(p-1/2)^2 n\right).
\end{equation}
The last two inequalities are due to the basic inequalities $\sqrt{x(1-x)}\leq 1/2-(x-1/2)^2$ for $x\in [0,1]$ and $1+x\leq e^x$ for all $x\in\mathbb{R}$.

The lemma below follows from~\eqref{eq:BZ.clean} via a simple scaling argument.

\begin{lemma}
\label{lmm:BZ}
Suppose the BZ algorithm is run on an interval $I$ divided into $\Delta$-length subintervals. The output estimator $\widehat{J}$ satisfies
\[
\mathbb{P}\{X^*\notin \widehat{J}\}\leq \frac{|I|}{\Delta} 2^{-c_3(p) n},
\]
where $c_3(p) = (p-1/2)^2\log e$.
\end{lemma}

\subsection{Construction of the querying strategies}
The idea behind the construction of the querying strategies inherits from the construction under the noiseless response setting. Recall that under the querying strategy described in Section~\ref{sec:alg.bayes}, the learner first runs bisection search to locate $X^*$ within a length $L\delta$ interval. She then runs replicated bisection on the $L$ length $\delta$ subintervals, submitting queries via the bisection search in the true subinterval containing $X^*$ and cloning those queries in the other $L-1$ subintervals. When the responses are noisy, firstly we replace the bisection searches with the BZ algorithm. Moreover, the learner can no longer discern the true interval by querying the endpoints of the subinterval only once. Instead we need to query each endpoint enough times, so that via a maximum-likelihood type procedure, the learner can estimate the true subinterval with high enough certainty.

Under the requirement that the learner is accurate on average, as per definition (a), we construct the following multi-stage querying strategy that achieves the upper bound in~\eqref{eq:res.mean}.
\begin{enumerate}
\item Let $L'=7L$. Divide $[0,1]$ into $(L'\delta)$-length subintervals\footnote{For simplicity, we assume $(L'\delta)^{-1}$ is an integer. If not, the analysis can be repeated by dividing $[0,1]$ into subintervals of length $(\lfloor(L'\delta)^{-1}\rfloor)^{-1}$.}
and run the BZ algorithm to estimate the subinterval that contains $X^*$. The BZ algorithm is run for $K_1= \frac{1}{c_3(p)}\log\frac{8}{7\epsilon L\delta}$ iterations. Write $I$ for the subinterval returned by the BZ algorithm.

\item Divide $I$ into $L'$ $\delta$-length subintervals, labeled $J_1,...,J_{L'}$ from left to right. Label the endpoints as $x_0,..., x_{L'}$ so that $J_k = [x_{k-1}, x_k)$. Submit $m=\frac{1}{c_4(p)}\log\frac{64\delta}{\epsilon}$ queries at each of the $L'-1$ endpoints $x_1, ..., x_{L'-1}$, where $c_4(p) = D(\text{Bern}(1/2)||\text{Bern}(p))=(\log\frac{1}{2p}+\log\frac{1}{2(1-p)})/2$. 

Write $m_k$ for the sum of the $m$ responses to the query at $x_k$. In other words, $m_k$ denotes the number of times the learner receives the response to ``$X^*\geq x_k$" being $1$. Let
\[
\widehat{k}=\arg\max_{1\leq k\leq L'}\sum_{i=1}^{k-1}m_i+\sum_{i=k}^{L'-1}(m-m_i),
\]
and take $J_{\widehat{k}}$ as the estimator for the subinterval that contains $X^*$.

\item Divide $J_{\widehat{k}}$ into length $(\epsilon/4)$ subintervals and run the BZ algorithm, while submitting queries in parallel in the other $L'-1$ subintervals $\{J_k\}_{k\neq \widehat{k}}$, as one would do in the replicated bisection. Run the BZ algorithm for $K_2 = \frac{2}{c_3(p)}\log \frac{4\sqrt{2}\delta}{\epsilon}$ iterations and obtain the output $J\subseteq J_{\widehat{k}}$.

\item Define the estimator $\widehat{X}$ as the midpoint of $J$.
\end{enumerate}

When the learner needs to be accurate with high probability, as per definition (b), we adopt the same multi-stage querying strategy above, with a slightly modified set of parameters. Let
\begin{equation}
\label{eq:param.prob}
K_1 = \frac{1}{c_3(p)}\log\frac{3M}{\delta},\quad m = \frac{1}{c_4(p)}\log(12 M),\quad K_2 = \frac{1}{c_3(p)}\log\frac{12 M\delta}{\epsilon}. 
\end{equation}

We claim that these querying strategies satisfy the desired levels of accuracy and privacy, and achieve the query complexities stated in the upper bound in~\eqref{eq:res.mean} and~\eqref{eq:res.prob} under definitions (a) and (b) respectively. See Section~\ref{sec:proof.upper.noisy} for a full proof.

\subsection{Lower bound proof techniques with noisy responses }
%The lower bound proof under definition a) of $\epsilon$-accuracy reuses the techniques used in the lower bound proof under the noiseless response setting.
Our lower bound proof reuses the techniques under the noiseless response setting, but substantial care is needed to deal with the noise. 
By upper bounding the mutual information between the responses and $X^*$, and applying the continuous version of Fano's inequality~\cite{duchi2013distance}, we can show that there are at least $\Omega(\log (|J|/\epsilon)$ queries on average in an interval $J$ that contains $X^*$. By letting $J$ be the length-$\delta/2$ interval containing $X^*$ and considering a proportionally-sampling adversary, we obtain a lower bound $\Omega(L\log(\delta/\epsilon))$ on the optimal query complexity. By taking $J=[0,1]$, we obtain the term $\Omega(\log(1/\epsilon))$ in the lower bound.

For the more stringent notion of $(\epsilon,M)$-accuracy, we need to establish an additional $\Omega(L\log M)$ term in the lower bound.
Again by considering a proportionally-sampling adversary, it suffices to show that there exist at least $\Omega(\log M)$ queries on average in a length-$\delta/2$ interval $J$ containing $X^*$. To that end, we first reduce the estimations problem of $X^*$ to a family of binary hypothesis testing problems between pairs of hypothesis on $J$, and then bound the average testing error from below using the Bhattacharyya coefficient~\cite{kailath1967divergence}.

See Section~\ref{sec:proof.lower.noisy} for the complete proofs for the lower bounds in Theorem~\ref{thm:noisy}.

%\nbr{Maybe we should briefly discuss the lower bound proof strategies for completeness. As far as I understand, the main idea is to assume
%that adversary uses proportionally sampling and to show there are at least $\Omega(\log \frac{\delta}{\epsilon})$ 
%queries that are up to $\delta$ distance away from $X^*$. 
%For the latter task, we use the continuous version of Fano's inequality and lower bound the mutual information between
%responses and $X*$. For the more stringent $(\epsilon,M)$-accurate requirement, to establish the additional $\Omega(L\log M)$ lower bound
%to the query complexity, we further show there are at least $\Omega(\log M)$ 
%queries that are up to $\delta$ distance away from $X^*$ by reducing the estimation problem to a
%family of binary hypothesis testing problems and applying the Battacharyya coefficient to lower bound
%the average testing error. }
%\nb{This subsection now briefly describes the strategies of the lower bound proof.}

\section{Conclusion and future work}\label{sec:conclusion}
Motivated by privacy concerns in applications such as Federated Learning and online price learning, we study a sequential learning problem that focuses on protecting the learner's privacy. A learner aims to estimate a value by sequentially submitting queries and receiving binary responses, while ensuring an adversary who observes queries but not responses cannot estimate well. We design new querying strategies and prove 
upper bounds on the optimal query complexity. We also derive almost-matching lower bounds,
showing that our querying strategies are nearly optimal. In the deterministic setting, our upper and lower bounds have a gap of only $8$ queries. In the Bayesian setting, there is a gap of $4L$ between the upper and lower bounds. %In this paper, we obtain the optimal query complexity in the stylized private sequential learning model.
The results are further extended to when the unknown value is in high dimensions, and  when the binary responses are noisy.
Under the noisy response model,
we obtain upper and lower bounds on the optimal query complexity that match up to constant factors.
An important future direction is to investigate how to protect the learner's privacy 
in more general online convex optimization problems, such as stochastic gradient descent algorithms, where the adversary observes the query $x_t$ but not the stochastic 
gradient $g(x_t)$ at each iteration $t$.

\section*{Acknowledgment}
J.~Xu is supported by the NSF Grants IIS-1838124, CCF-1850743, and CCF-1856424. 
D.~Yang is supported by the NSF Grant CCF-1850743.

\begin{appendices}

\section{Proofs of Theorem~\ref{thm:bayes} and~\ref{thm:freq}}\label{sec:proof}
In this section we prove our results on the optimal query complexity for the sequential learning model, under the Bayesian setting (Theorems~\ref{thm:bayes}) and the deterministic setting (Theorem~\ref{thm:freq}).

\subsection{Analysis under the Bayesian setting}\label{sec:proof.bayes}
\begin{proof}[Proof of Theorem \ref{thm:bayes}]

{\bf Upper bound:} 
First we give in algorithm~\ref{alg:bayes} a precise description of the multistage querying strategy introduced in Section~\ref{sec:alg.bayes}. We claim that this querying strategy achieves the upper bound in Theorem~\ref{thm:bayes}.
\begin{algorithm}
$K_1:=\lfloor \log (1/(L\delta))\rfloor$; $I:= [0,1]$\;
\For(\tcp*[f]{bisection to an interval $I$ of length $2^{-K_1}$}){i = 1 to $K_1$} {
	$q_i:=$ the midpoint of $I=[a,b]$\;
	\leIf {$r_i=1$} {$I:=[q_i,b]$} {$I:=[a,q_i]$} 
}
\For(\tcp*[f]{divide $I$ into $L$ equal-length subintervals $I_1,...,I_L$}){i in 1 to L} {
	$I_i:= [a+(i-1)(b-a)/L, a+i(b-a)/L]$, where $[a,b]=I$\;
	$J_i:=I_i$\;
}
\For(\tcp*[f]{query the endpoints of $I_1,...,I_L$}){i in 1 to L-1} {
	$q_{K_1+i}:=$ the right endpoint of $I_i$\;
}
Inspect the responses to find $i^*\in \{1,...,L\}$ such that $X^*\in I_{i^*}$\;
$K_2:=\lceil\log (\delta/\epsilon)\rceil +1$\;
\For(\tcp*[f]{replicated bisection on $I_1,...,I_L$}){i in 1 to $K_2$} {
	\For{j in 1 to L} {
		$q_{K_1+L-1+(i-1)L+j}:=$ the midpoint of $J_j$\;
	}
	\eIf{$r_{K_1+L-1+(i-1)L+i^*}=1$}{
		\lFor{j in 1 to L} {the left endpoint of $J_j:=q_{K_1+L-1+(i-1)L+j}$}
	}{
		\lFor{j in 1 to L} {the right endpoint of $J_j:=q_{K_1+L-1+(i-1)L+j}$}
	}
}
$\widehat{X}:=$ the midpoint of $J_{i^*}$\;
\caption{Our querying strategy under the Bayesian setting}
\label{alg:bayes}
\end{algorithm}

Under algorithm~\ref{alg:bayes}, the total number of queries submitted is
\[
K_1+L-1+LK_2 = \left\lfloor \log\frac{1}{L\delta}\right\rfloor + L\left(\lceil \log\frac{\delta}{\epsilon}\rceil+2\right)-1,
\]
matching the desired upper bound. It remains to show that algorithm~\ref{alg:bayes} is both $\epsilon$-accurate and $(\delta,L)$-private. First we establish accuracy. From the responses to all the queries, the learner can narrow down the possible values of $X^*$ to an interval $I^{(\text{final})}$ of length
\begin{equation}
\label{eq:eps.accuracy}
\left|I^{(\text{final})}\right|= \frac{1}{L}2^{-K_1}2^{-K_2}
 = \frac{1}{L}2^{-(\lfloor \log (1/L\delta)\rfloor+\lceil\log (\delta/\epsilon)\rceil +1)}\leq \frac{1}{L}2^{-\log (1/L\epsilon)}=\epsilon.
\end{equation}
The learner can then take $\widehat{X}$ to be the midpoint of this interval so that $|\widehat{X}-X^*|\leq \epsilon/2$.

Next we show privacy. Recall that the learner performs parallel bisections on the $L$ intervals $I_1,...,I_L$. Since the adversary only observes the queries and the querying strategy $\phi$, she learns that $X^*$ is contained in one of $L$ intervals $J_1,..., J_L$ where $J_j=[a_j,b_j]\subseteq I_j$. But she cannot tell which one of them $X^*$ is in. Therefore she cannot guess the location of $X^*$ with probability higher than $1/L$. More precisely, the posterior distribution of $X^*$ given all the query sequence is uniform over the union of $J_1,..., J_L$. Use $|\cdot|$ to denote the Lebesgue measure of subsets of $[0,1]$. We have
\begin{equation}\label{eq:post}
\mathbb{P}\{|\widetilde{X}-X^*|\leq \delta/2\mid \text{the query sequence}\}= \frac{\left|\left(\cup_{i\leq L}J_i\right)\cap \left[\widetilde{X}-\delta/2, \widetilde{X}+\delta/2\right]\right|}{\left|\cup_{i\leq L}J_i\right|}.
\end{equation}
Since the queries on $I_1,..., I_L$ are exact copies of each other, $J_1,..., J_L$ are also equidistant translations on the real line. The left endpoints $a_1,..., a_L$ of $J_1,..., J_L$ satisfy $a_{i+1}=a_i+|I_1|$ for all $i$ where $|I_1|=2^{-K_1}/L\geq \delta$. Moreover, note that the lengths of all $J_i$ are equal, and because the adversary does not observe the response to the last batch of queries, $|J_i|=2|I^{(\text{final})}|$. From~\eqref{eq:eps.accuracy} we have $|J_i|\leq 2\epsilon$. Therefore under the assumption that $\delta\geq 2\epsilon$, any interval of length $\delta$ can only intersect with $\cup_i J_i$ on a set of Lebesgue measure at most $|J_1|$. Deduce that the right hand side of~\eqref{eq:post} is upper bounded by $|J_1|/|\cup_i J_i|=1/L$. Therefore
\[
\mathbb{P}\{|\widetilde{X}-X^*|\leq \delta/2\} = \mathbb{E}\left(\mathbb{P}\{|\widetilde{X}-X^*|\leq \delta/2\mid \text{the queries}\}\right)\leq 1/L.
\]

{\bf Lower bound:} 
Suppose $\phi$ is an $\epsilon$-accurate and $(\delta,L)$-private strategy that submits at most $n$ queries. Denote $\mathbf{n}(X^*,Y)$ as the number of queries submitted when $X^*$ is the truth and the random seed is $Y$, so $n=\sup_{X^*,Y}\mathbf{n}(X^*, Y)$.
The goal is to bound $n$ from below. Consider the querying strategy $\widetilde{\phi}$ that concatenates trivial queries at $0$ to the query sequence so that the length of query sequence is always $n$, $i.e., $ $\widetilde{q}_i=q_i$ for $i\leq \mathbf{n}(X^*,Y)$ and $\widetilde{q}_i=0$ for $\mathbf{n}(X^*,Y)<i\leq n$. Clearly $\widetilde{\phi}$ is also $\epsilon$-accurate and $(\delta,L)$-private, because the trivial queries at $0$ do not provide the adversary with any extra information. Moreover the maximum number of queries submitted by $\widetilde{\phi}$ equals that submitted by $\phi$. Hence for the rest of this proof, without loss of generality, we can assume that the learner always submits exactly $n$ queries under $\phi$.

Since $\phi$ is $(\delta,L)$-private, we have $\mathbb{P}\{|\widetilde{X}-X^*|\leq \delta/2\}\leq 1/L$ for each adversary $\widetilde{X}$. Consider the adversary that adopts the \emph{truncated proportional-sampling} strategy described in Section~\ref{sec:lower.bayes}: let $\widetilde{X}=q_J$ where $J\sim \text{Unif}\{K+1,...,n\}$. Choose $K=\lfloor \log (1/(L\delta))\rfloor$. 
Let us point out that $n$ must be larger than $K$ so truncated proportional-sampling can be run. We will show later in the proof that $n>K$ always holds for any strategy $\phi$ that is $\epsilon$-accurate.
%This adversary is smart to neglect the first $K$ queries since they are very unlikely to be close to $X^*$, as our calculations will show. 
By construction,
\[
\mathbb{P}\left\{\left|\widetilde{X}-X^*\right|\leq \delta/2\right\}=\mathbb{E}\frac{\sum_{i=K+1}^n\mathds{1}\{|q_i-X^*|\leq \delta/2\}}{n-K}\leq \frac{1}{L}.
\]
Deduce that
\[
n\geq K+L\left(\sum_{i=1}^n \mathbb{P}\left\{|q_i-X^*|\leq \frac{\delta}{2}\right\}-\sum_{i=1}^K \mathbb{P}\left\{|q_i-X^*|\leq \frac{\delta}{2}\right\}\right).
\]
We claim that
\begin{enumerate}[(i)]
\item $\sum_{i\leq n}\mathbb{P}\{|q_i-X^*|\leq \delta/2\}\geq \log (\delta/4\epsilon)$.
\item $\sum_{i\leq K}\mathbb{P}\{|q_i-X^*|\leq \delta/2\}\leq 1/L$.
\item $n>K$, so that the truncated proportional-sampling strategy is valid.
\end{enumerate}
The desired lower bound immediately follows.

{\it Proof of (i) and (iii):} The statement (i) claims that on average, there are at least $\log(\delta/4\epsilon)$ queries in the interval $[X^*-\delta/2,X^*+\delta/2]$. One would expect this to be true because $[X^*-\delta/2, X^*+\delta/2]$ is an interval of length $\delta$. In order for the learner to achieve $\epsilon$-accuracy, it needs to submit at least $\log (\delta/\epsilon)$ queries by optimality of the bisection method. Next we make this argument rigorous. The randomness of the interval $[X^*-\delta/2, X^*+\delta/2]$ complicates the proof. We will instead show something stronger than (i). We claim that for each fixed interval $I\subseteq [0,1]$, we have 
\begin{equation}
\label{eq:queries.in.I}
\sum_{i\leq n}\mathbb{P}\{q_i\in I\mid X^*\in I\}\geq \log (|I|/2\epsilon).
\end{equation}
To see why (i) follows from~\eqref{eq:queries.in.I}, note that for each interval $I$ of length $\delta/2$,
\[
\sum_{i\leq n}\mathbb{P}\left\{|q_i-X^*|\leq \delta/2 \mid X^*\in I\right\}
\geq \sum_{i\leq n}\mathbb{P}\{q_i\in I\mid X^*\in I\}\geq \log (|I|/2\epsilon)=\log (\delta/4\epsilon).
\]
Moreover, claim (iii) also follows from~\eqref{eq:queries.in.I} by taking $I=[0,1]$:
\[
n=\sum_{i\leq n}\mathbb{P}\{q_i\in [0,1]\} \geq \log (1/2\epsilon)>\lfloor \log (1/(L\delta))\rfloor =K,
\]
where the strict inequality holds because by assumption $2\epsilon \le \delta$ and $L\ge2.$

It remains to show~\eqref{eq:queries.in.I}. Since $\phi$ is $\epsilon$-accurate, we have
\[
\mathbb{P}\left\{\left|\widehat{X}-X^*\right|>\epsilon/2\mid X^*\in I, Y=y\right\}=0
\]
for all but a negligible (zero-measure) set of the random seed $Y$, 
denoted as $\mathcal{N}^y$. For $y\notin \mathcal{N}^y$, conditioning on $Y=y$, the estimator $\widehat{X}$ is only a function of the responses $r_1,...,r_n$. Further conditioning on $X^*\in I$, 
since $X^*$ is independent from the random seed $Y$, 
 $X^*$ is distributed uniform in $I$. By the continuous version of Fano inequality~\cite[Proposition~2]{duchi2013distance}, %\hl{(check if their proof goes through if the logs are with base 2, i.e., if the information theoretic quantities are defined in terms of bits, is their log2 actually 1 bit?) Answer: yes it can be done by simply dividing both the numerator and denominator by log2}
\[
\mathbb{P}\left\{\left|\widehat{X}-X^*\right|>\epsilon/2\mid X^*\in I,Y=y\right\}
\geq 1-\frac{I(X^*;r_1,..., r_n\mid X^*\in I, Y=y)+1}{\log (|I|/\epsilon)}.
\]
Hence
\begin{equation}
\label{eq:fano}
H(r_1,..., r_n\mid X^*\in I,Y=y)\geq I(X^*; r_1,..., r_n\mid X^*\in I,Y=y)
\geq \log (|I|/\epsilon)-1 = \log(|I|/2\epsilon).
\end{equation}
Using the entropy chain rule, the left hand side can also be written as
\begin{align}
\nonumber & H(r_1,..., r_n\mid X^*\in I,Y=y)\\
\label{eq:chain.rule} = & H(r_1|X^*\in I,Y=y) + \sum_{i=1}^{n-1} H(r_{i+1}\mid X^*\in I, Y=y, r_1,..., r_i).
\end{align}
%\nbr{The following argument does not consider $H(r_1|X^*\in I)$. 
%But I think the following argument still holds if we set $i=0$ and view $r_1,\ldots, r_0$ as void.}
Expand each summand:
%\hl{Jiaming, could you please check if the derivations below are correct? Can they be simplified?}
\begin{align}
\nonumber & H\left(r_{i+1}\mid X^*\in I, Y=y, r_1,..., r_i\right)\\
\label{eq:summand} = & \sum_{\rho_1,.., \rho_i}\mathbb{P}\left\{r_1=\rho_1,...,r_i=\rho_i\mid X^*\in I, Y=y\right\}
H\left(r_{i+1} \mid X^*\in I, Y=y, r_1=\rho_1,..., r_i=\rho_i\right).
\end{align}
Write $I=[a,b]$. On the event $X^*\in I$, if $q_{i+1}=\phi_i(\rho_1,..., \rho_i, y)$ is smaller than $a$, then $r_{i+1}=1$. Similarly if $q_{i+1}>b$, then $r_{i+1}=0$. In other words, the value of $r_{i+1}$ is completely determined by $\rho_1,...,\rho_i$ if $q_{i+1}\notin I$ and $X^*\in I$. Hence the summation~\eqref{eq:summand} equals
\begin{align*}
&\sum_{\rho_1,.., \rho_i: \phi(\rho_1,...,\rho_i, y)\in I} \mathbb{P}\left\{r_1=\rho_1,...,r_i=\rho_i\mid X^*\in I, Y=y\right\}H\left(r_{i+1} \mid X^*\in I, Y=y, r_1=\rho_1,..., r_i=\rho_i\right)\\
\leq & \mathbb{P}\left\{q_{i+1}\in I \mid X^*\in I, Y=y\right\}.
\end{align*}
With $Y=y$ fixed, we have $q_1=f_0(y)$ is deterministic. Similarly argue that $H(r_1|X^*,Y=y)\leq \mathds{1}\{q_1 \in I\}$.
Combine with~\eqref{eq:fano} and~\eqref{eq:chain.rule} to deduce that
\[
\sum_{i\leq n}\mathbb{P}\{q_i \in I\mid X^*\in I, Y=y\}
\geq H(r_1,..., r_n\mid X^*\in I, Y=y)\geq \log(|I|/2\epsilon).
\]
The above holds for all $y\notin \mathcal{N}^y$. Since $\mathcal{N}^y$ is a negligible set, we have
\[
\sum_{i\leq n}\mathbb{P}\{q_i\in I\mid X^*\in I\}
 = \int_{[0,1]\backslash \mathcal{N}^y}\sum_{i\leq n}\mathbb{P}\{q_i \in I\mid X^*\in I, Y=y\} dy \geq \log (|I|/2\epsilon).
\]
The proof of~\eqref{eq:queries.in.I} is complete.

{\it Proof of (ii):} To show (ii) we introduce the notion of {\it learner intervals}, which stands for the sequence of intervals that the learner knows $X^*$ is in, as the learner submits queries sequentially. Start from $I_0=[0,1]$. If $r_1=1$, then the learner learns that $X^*\in [q_1,1]$ and $I_1$ is defined as $[q_1,1]$. Otherwise $I_1=[0,q_1]$. For all $i$,
\begin{align*}
\mathbb{P}\left\{|q_i-X^*|\leq \frac{\delta}{2}\right\} = & \mathbb{E}\left[\mathbb{P}\left\{|q_i-X^*|\leq \frac{\delta}{2}\mid r_1,..., r_{i-1}\right\}\right]\\
 = &\mathbb{E} \left[\frac{\left|I_{i-1}\cap [q_i-\delta/2,q_i+\delta/2]\right|}{|I_{i-1}|}\right]\\
 \leq & \delta \mathbb{E}(1/|I_{i-1}|).
\end{align*}
Next we show that $ \mathbb{E}(1/|I_{i}|) \leq 2^i$ for all $i$ by induction. Suppose it is true for $i=0,...,k$. For $i=k+1$,
\[
\mathbb{E}(1/|I_{k+1}|) = \mathbb{E}\left[\mathbb{E}\left(1/|I_{k+1}| \;\Big\rvert \; r_1,..., r_k\right)\right].
\]
Conditioning on $r_1,..., r_k$, the learner interval $I_k$ is deterministic and so is $q_{k+1}$. Let $I_k=[a_k,b_k]$. There are three possibilities for $I_{k+1}$:
\begin{enumerate}
\item $q_{k+1}\notin I_k$. In this case the $r_{k+1}$ provides no additional information on the location of $X^*$. Therefore $I_{k+1}=I_k$.
\item $q_{k+1}\in I_k$ and $r_{k+1}=1$. The learner learns that $X^*\geq q_{k+1}$ and $I_{k+1}=[q_{k+1},b_k]$.
\item $q_{k+1}\in I_k$ and $r_{k+1}=0$. In this case $I_{k+1}=[a_k,q_{k+1}]$.
\end{enumerate}
Therefore
\begin{align*}
\mathbb{E}\left(1/|I_{k+1}| \mid r_1,..., r_k\right)
= & \mathds{1}\left\{q_{k+1}\notin I_k\right\}\frac{1}{|I_k|} + \mathds{1}\left\{q_{k+1}\in I_k\right\}\mathbb{P}\left\{X^*\geq q_{k+1}\mid r_1,...,r_k\right\}\frac{1}{b_k-q_{k+1}} \\
& + \mathds{1}\left\{q_{k+1}\in I_k\right\}\mathbb{P}\left\{X^*< q_{k+1}\mid r_1,...,r_k\right\}\frac{1}{q_{k+1}-a_k}\\
= & \mathds{1}\left\{q_{k+1}\notin I_k\right\}\frac{1}{|I_k|} + \mathds{1}\left\{q_{k+1}\in I_k\right\}\frac{1}{|I_k|} + \mathds{1}\left\{q_{k+1}\in I_k\right\}\frac{1}{|I_k|}\leq \frac{2}{|I_k|}.
\end{align*}
Hence $\mathbb{E}(1/|I_{k+1}|)\leq \mathbb{E}(2/|I_k|)\leq 2\cdot 2^k = 2^{k+1}$. Deduce that $\mathbb{P}\{|q_i-X^*|\leq \delta/2\}\leq \delta \mathbb{E}(1/|I_{i-1}|)\leq \delta 2^{i-1}$ for all $i$. Therefore
\[
\sum_{i=1}^K\mathbb{P}\left\{\left|q_i-X^*\right|\leq \frac{\delta}{2}\right\}\leq \delta \sum_{i=1}^K 2^{i-1}\leq \delta 2^K\leq 1/L,
\]
where the last inequality is from $K=\lfloor \log (1/(L\delta))\rfloor \leq \log (1/(L\delta))$.
\end{proof}

\subsection{Analysis under the deterministic setting}\label{sec:proof.freq}
%$2L+\log(\max\{2^{-L},\delta\}/\epsilon)-8$ queries.

\begin{proof}[Proof of Theorem \ref{thm:freq}]

{\bf Upper bound:} 
As mentioned in Section~\ref{sec:alg.freq}, our querying strategy requires different construction for the $\delta\leq 2^{-L}$ and $\delta> 2^{-L}$ cases. The precise description of the querying strategy is given in algorithm~\ref{alg:small.delta} (for $\delta\leq 2^{-L}$) and algorithm~\ref{alg:large.delta} (for $\delta>2^{-L}$). 
%The reader might notice that algorithm~\ref{alg:large.delta} starts with an initial guess $X^*\in [0,\epsilon)$.  This creates an $\epsilon$-length interval $[0,\epsilon)$ belonging to the information set, which is needed to ensure that the $\delta$-covering number of the information set is at least $L$ (See the derivations around \eqref{eq:find.K} for details). 

\begin{algorithm}
$GuessedIt := False$\;
$I:=[0,1]$\;
\For(\tcp*[f]{submit $L$ guesses via (possibly partially fake) bisection}){i = 1 to L} {
	$q_{2i-1}:=$ the midpoint of $I=[a,b]$\;
	$q_{2i}:=q_{2i-1}+\epsilon$\;
	\eIf{not GuessedIt}{
		Inspect the responses $r_{2i-1}$ and $r_{2i}$\;
		\leIf {$r_{2i-1}=1$} {$I:=[q_{2i-1},b]$} {$I:=[a,q_{2i-1}]$} 
		\lIf {$r_{2i-1}=1$ and $r_{2i}=0$} {$GuessedIt:=True$}
	}(\tcp*[f]{once guessed correctly, proceed with a fake bisection}){
		Sample $R\sim \text{Bernoulli}(1/2)$\;
		\leIf {$R=1$} {$I:=[q_{2i-1},b]$} {$I:=[a,q_{2i-1}]$}
	}	
}
$i:= 2L+1$, $J:=I$\;
\While(\tcp*[f]{run (possibly fake) bisection on $J$}){$|J|>\epsilon$} {
	$q_i:=$ the midpoint of $J=[a,b]$\;
	\eIf{not GuessedIt} {
		\leIf {$r_i=1$} {$J:=[q_i,b]$} {$J:=[a,q_i]$}
	}{
		Sample $R\sim \text{Bernoulli}(1/2)$\;
		\leIf {$R=1$} {$J:=[q_i,b]$} {$J:=[a,q_i]$}
	}
	$i:=i+1$\;
}
$\widehat{X}:=$ the midpoint of $J$\;
\caption{Our querying strategy under the deterministic setting when $\delta\leq 2^{-L}$}
\label{alg:small.delta}
\end{algorithm}

First consider the case $\delta\leq 2^{-L}$. Under algorithm~\ref{alg:small.delta}, the learner first submits $L$ guesses ($2L$ queries). She then conducts a bisection search within an interval of length $2^{-L}$, taking $\lceil 2^{-L}/\epsilon\rceil$ queries to achieve $\epsilon$-accuracy. The total number of queries submitted is $L+\lceil\log (1/\epsilon)\rceil$.

Because the guesses $I_1,...,I_L$ are all of length $\epsilon$, algorithm~\ref{alg:small.delta} is $\epsilon$-accurate. It suffices to show it is also $(\delta,L)$-private. As argued in Section~\ref{sec:alg.freq}, the adversary cannot rule out the possibility that $X^*\in I_i$ for some $i=1,..,L$, so the information set contains the union of $I_1,...,I_L$. That is, for each query sequence $q$,
\[
\mathcal{I}(q)\supseteq \cup_{i\leq L}[q_{2i-1}, q_{2i}).
\]
When $\delta\leq 2^{-L}$, these $L$ intervals do not overlap. Since their left endpoints are submitted via a bisection search, legitimate or fake, they are at least $2^{-L}\geq \delta$ apart from each other. Therefore the $\delta$-covering number for $\mathcal{I}(q)$ is at least $L$. We have shown that algorithm~\ref{alg:small.delta} is $(\delta,L)$-private.

\begin{algorithm}
$q_1:=0; q_2:=\epsilon$;\tcp*[f]{submit initial guess at 0}

$K:=$ an integer solution in $\{0,1,..., L-1\}$ to $\ell_K=2^{-K}/(L-K) \in [\delta, 2\delta]$\;
\leIf {$r_1=1$ and $r_2=0$} {$GuessedIt:=True$} {$GuessedIt:=False$}
$I:=[0,1]$\;
\For(\tcp*[f]{submit the next $K$ guesses via bisection}){i=2 to K+1} {
	$q_{2i-1}:=$ the midpoint of $I=[a,b]$\;
	$q_{2i}:=q_{2i-1}+\epsilon$\;
	\eIf{not GuessedIt}{
		\leIf {$r_{2i-1}=1$} {$I:=[q_{2i-1},b]$} {$I:=[a,q_{2i-1}]$} 
		\lIf {$r_{2i-1}=1$ and $r_{2i}=0$} {$GuessedIt:=True$}
	}{
		Sample $R\sim \text{Bernoulli}(1/2)$\;
		\leIf {$R=1$} {$I:=[q_{2i-1},b]$} {$I:=[a,q_{2i-1}]$}
	}	
}
Divide $I$ into $L-K$ equal length subintervals $I_1,..., I_{L-K}$\;
\For(\tcp*[f]{submit the next $L-K-1$ guesses via grid search}){i = (K+2) to L} {
	$q_{2i-1}:=$ the right endpoint of $I_{i-K-1}$\;
	$q_{2i}:= q_{2i-1}+\epsilon$\;
	\lIf {$r_{2i-1}=1$ and $r_{2i}=0$} {$GuessedIt := True$}
}
\eIf{not GuessedIt} {
	$J:=$ the subinterval $I_{i^*}$ that contains $X^*$\;
}{
	$J:= I_{i^*}$ where $i^*$ is sampled uniformly from $\{1,..., L-K\}$\;
}
$i:= 2L+1$\;
\While(\tcp*[f]{run (possibly fake) bisection on $J$}){$|J|>\epsilon$} {
	$q_i:=$ the midpoint of $I=[a,b]$\;
	\eIf{not GuessedIt} {
		\leIf {$r_i=1$} {$J:=[q_i,b]$} {$J:=[a,q_i]$}
	}{
		Sample $R\sim \text{Bernoulli}(1/2)$\;
		\leIf {$R=1$} {$J:=[q_i,b]$} {$J:=[a,q_i]$}
	}
	$i:=i+1$\;
}
$\widehat{X}:=$ the midpoint of $J$\;
\caption{Our querying strategy under the deterministic setting when $\delta> 2^{-L}$}
\label{alg:large.delta}
\end{algorithm}

Next consider the case $\delta>2^{-L}$. Again algorithm~\ref{alg:large.delta} is clearly $\epsilon$-accurate. To show that it is also $(\delta,L)$-private, note that algorithm~\ref{alg:large.delta} is designed so that the closest pair of guesses are of distance $[\delta,2\delta]$ apart. Hence
\begin{enumerate}[(i)]
\item The intervals $I_i=[q_{2i-1},q_{2i})$, $i=1,...,L$ do not overlap, and their left endpoints are at least $\delta$ from each other;
\item After the $L$ guesses are submitted, the learner can always narrow down the possibilities for $X^*$ to an interval of length at most $2\delta$.
\end{enumerate}
% is it worth repeating this, or do we refer the reader back to the previous section?

We claim that (i) ensures $(\delta,L)$-privacy. As in the $\delta\leq 2^{-L}$ case, we have for each $q$, $\mathcal{I}(q)\supseteq \cup_{i\leq L}[q_{2i-1}, q_{2i})$. Assuming (i), the $\delta$-covering number of $\cup _{i\leq L}[q_{2i-1},q_{2i})$ is at least $L$. 

Given (ii), the learner only needs to submit at most $\lceil\log (2\delta/\epsilon)\rceil$ queries to achieve $\epsilon$-accuracy in stage 2. The total number of queries submitted under algorithm~\ref{alg:large.delta} is at most $2L+\lceil \log (\delta/\epsilon)\rceil+1$. Moreover, as we can see from algorithm~\ref{alg:large.delta} the learner always submits $q_1=0$. Omit this first trivial query to obtain the desired query complexity upper bound $2L+\lceil\log (\delta/\epsilon)\rceil$. 

We still need to show that (i) and (ii) are satisfied by algorithm~\ref{alg:large.delta}. The first $K$ guesses locate $X^*$ within an interval $I$ of length $2^{-K}$. The remaining $L-K-1$ odd queries then divide $I$ into $L-K$ subintervals of equal length. Therefore the closest pair of odd queries among $q_1,q_3,...,q_{2L-1}$ are at distance $2^{-K}/(L-K)$. In stage 2, the learner conducts a bisection search in one of the $L-K$ subintervals, which is also of length $2^{-K}/(L-K)$. Therefore (i) and (ii) translate to $\delta\leq 2^{-K}/(L-K)\leq 2\delta$. It remains to show that we can find at least one $K\in \{0,1,...,L-1\}$ for which 
\begin{equation}
\label{eq:find.K}
\ell_K:= \frac{2^{-K}}{L-K}\in [\delta, 2\delta].
\end{equation}
Observe that
\begin{enumerate}
\item $\ell_0=1/L\geq \delta$;
\item $\ell_{L-1}=2^{-(L-1)}\leq 2\delta$;
\item for all $K<L-1$,
\[
\frac{\ell_K}{\ell_{K+1}}= \frac{2^{-K}}{2^{-(K+1)}}\frac{L-K-1}{L-K}\leq 2.
\]
\end{enumerate}
These facts above ensure that there is at least one solution to~\eqref{eq:find.K} in $\{0,1,..., L-1\}$.

{\bf Lower bound:} 
The lower bound $\lceil\log(\delta/\epsilon)\rceil+2L-4$ has already been proven in~\cite[Theorem~4.1]{tsitsiklis2018private}. As in the upper bound proof we separately consider the cases $\delta> 2^{-L}$ and $\delta\leq 2^{-L}$. When $\delta>2^{-L}$, the term $\lceil\log(\delta/\epsilon)\rceil+2L-4$ is always larger than $\lceil\log(1/\epsilon)\rceil+L-8$. Thus we only need to show that
when $\delta> 2^{-L}$, optimal query complexity is lower bounded by $\lceil\log(1/\epsilon)\rceil+L-8$.

It suffices to show the lower bound holds for all realizations of the random seed $Y$ so the dependences on $Y$ are suppressed for the rest of the proof. Fix any querying strategy $\phi$ that is both $\epsilon$-accurate and $(\delta,L)$-private.
Let $\mathcal{Q}(X^*)$ denote the set of queries when the true value is $X^*$. 
We want to show there is at least one $X^*$ for which $|\mathcal{Q}(X^*)|\geq L+\log(1/\epsilon)-8$. To this end, 
we will prove the following claims:
\begin{enumerate}[(i)]
\item There exists an interval $I$ of length $2\delta$ and $\widetilde{\mathcal{Q}}=\{\widetilde{q}_1,...,\widetilde{q}_K\}$ where $K\geq \log (1/\delta)-3$ and $|\widetilde{q}_i-\widetilde{q}_j|>\delta$ for all $i\neq j$, such that for each $X^* \in I$, $\mathcal{Q}(X^*)\backslash I\supseteq \widetilde{\mathcal{Q}}$.
\item For each interval $I$ of length $2\delta$ and each $X^*\in I$, there exist at least $L-5$ distinct pairs of queries $\{s_1,t_1\}, ... \{s_{L-5}, t_{L-5}\}\subseteq \mathcal{Q}(X^*)\backslash I$ for which $|s_i-t_i|\leq \epsilon$ for all $i$.
\item For each interval $I$ of length $2\delta$, there exists $X^*\in I$ such that $\mathcal{Q}(X^*)$ contains at least $\log (\delta/\epsilon)$ queries in $I$.
\end{enumerate}
%\nbr{When I read the above claim first time, I was struggled to understand the claim and 
%why we need it.  Below are some possible suggestions to improve the readability. 

%First, I think we can move the definition of $\mathcal{Q}(X^*)$ outside of the claim.

%Second, it is unclear whether the interval $I$ in (ii) is an arbitrary interval of length $2\delta$,
%or it must be the interval in (i). From the proof of (ii), it seems that $I$ in (ii) can be arbitrary. 

%Finally, I feel here we may want to cast some high-level proof idea before introducing the claim.
%I extracted some words from our exchanged emails:

%``Basically, consider an interval $I$ of length $\delta$ around the true value $X^*$. We know that 
%(1) Within the interval $I$, there are at least $\log (\delta/\epsilon)$ queries;
%(2) Outside the interval $I$, there are at least $\log (1/\delta)$ queries that are at least delta away from each other. 
%The key is to argue that there are at least $\log (1/\delta) +L-1$ queries outside the interval $I$. 
%To this end, we first show that outside the interval $I$, 
%there exist (i) $\log (1/\delta)$ queries that are more than $\delta$ apart
%and (ii) $L$ pairs of queries that are at most $\epsilon$ apart. Note that 
%(i) holds for the bisection search strategy, but we need to argue that it holds for any $\epsilon$-accurate
%and $(\delta,L)$-private learner strategy; (ii) holds because the learner needs to achieve $\epsilon$-accuracy
%while ensuring the $\delta$-covering number is at least $L$.''

%}
Claims (i) and (ii) together imply that there exists an interval $I$ of length $2\delta$  such that for all $X^*\in I$,
\[
\mathcal{Q}(X^*)\backslash I \supseteq \widetilde{\mathcal{Q}} \cup \left(\cup_{i\leq L-5}\{s_i,t_i\}\right).
\]
Since all members of $\widetilde{\mathcal{Q}}$ are at least $\delta$-apart and $|s_i-t_i|\leq \epsilon$, at least one of $s_i$ and $t_i$ is outside of $\widetilde{\mathcal{Q}}$. To show that on top of $\widetilde{\mathcal{Q}}$, each pair $\{s_i,t_i\}$ contributes at least one extra member to $\mathcal{Q}(X^*)\backslash I$, we only need to rule out the case where two pairs $\{s_i,t_i\}$ and $\{s_j,t_j\}$ are such that $s_i,s_j\in \widetilde{\mathcal{Q}}$ and $t_i=t_j$. This cannot happen because otherwise,
\[
\delta < |s_i-s_j|\leq |s_i-t_i| + |s_j-t_i| = |s_i-t_i| + |s_j-t_j|\leq \epsilon+\epsilon,
\]
contradicting the assumption $\delta\geq 2\epsilon$. Thus $\mathcal{Q}(X^*)\backslash I$ contains at least $K+L-5$ distinct members. From claim (iii) there exists $X^*\in I$ for which $|\mathcal{Q}(X^*)\cap I|\geq \log (\delta/\epsilon)$. We have
\[
\left|\mathcal{Q}(X^*)\right| = \left|\mathcal{Q}(X^*)\cap I\right| + \left|\mathcal{Q}(X^*)\backslash I\right| \geq \log(\delta/\epsilon)+ K+L-5\geq L+\log\frac{1}{\epsilon}-8,
\]
which equals $2L+\log (\max\{2^{-L},\delta\}/\epsilon)-8$ when $\delta\leq 2^{-L}$. It remains to prove the three claims.

{\it Proof of (i):} To prove this claim, we first construct for each $X^*\in [0,1]$ a subsequence $\widetilde{q}$ of $q$ where all the queries in $\widetilde{q}$ are at least $\delta$ apart from each other. %The main idea behind the construction is to introduce an {\it oracle querying strategy} $\widetilde{\phi}$, which has access to more information than the ordinary querying strategy $\phi$. 

Let $\widetilde{q}_1=q_1$. If $X^*\in [\widetilde{q}_1-\delta,\widetilde{q}+\delta]$, then declare the construction finished, {\it i.e.} the subsequence $\widetilde{q}= (\widetilde{q}_1)$ is of length one. Otherwise look at $q_2=\phi_1(r_1)$. If $q_2\in [\widetilde{q}_1-\delta, \widetilde{q}_1+\delta]$, then $\widetilde{q}_1$ and $q_2$ must be on the same side of $X^*$ and $r_2=\mathds{1}\{X^*\geq q_2\}$ must be equal to $r_1$. Proceed to look at $q_3=\phi_2(r_1,r_2)=\phi_2(r_1,r_1)$, $q_4= \phi_3(r_1,r_1,r_1)$ and so on, until $q_i\notin [\widetilde{q}_1-\delta,\widetilde{q}_1+\delta]$. Let $\widetilde{q}_2=q_i$. Similarly define the rest of $\widetilde{q}$ as follows: For $k\geq 2$ if $\widetilde{q}_k$ is chosen to be $q_{i_k}$, then let $\widetilde{q}_{k+1}=q_{i_{k+1}}$, where
\[
i_{k+1} = \min_j\left\{j>i_k: q_j\notin \cup_{k'\leq k}[\widetilde{q}_{k'}-\delta,\widetilde{q}_{k'}+\delta]\right\}.
\]
Repeat this process until $[\widetilde{q}_k-\delta,\widetilde{q}_k+\delta]$ contains $X^*$. 
Note that such a $k$ always exists, as $\phi$ is $\epsilon$-accurate and hence there exists at least one query that is within $\epsilon$ distance to $X^*$.

Let $\widetilde{r}_i=\mathds{1}\{X^*\geq \widetilde{q}_i\}$.
%Note that $\widetilde{\phi}$ is not a legitimate querying strategy, since it requires access to additional information. 
Next we argue that $\widetilde{q}$ is completely determined by $\widetilde{r}$. Indeed, given $\widetilde{r}= (\widetilde{r}_1,..., \widetilde{r}_k)$, we have $\widetilde{q}_j=q_{i_j}$ for all $j\leq k$, where $i_1=1$ and
\[
i_2 = \min_j\{j>i_1: \phi_{j-1}(\widetilde{r}_1,...,\widetilde{r}_1)\notin [\widetilde{q}_1-\delta,\widetilde{q}_1+\delta]\}.
\]
Thus $\tilde{q}_2=q_{i_2}=\phi_{i_2-1}(\widetilde{r}_1,...,\widetilde{r}_1)$.
To determine $i_3$, inspect $q_{i_2+1}=\phi_{i_2}(r_1,...,r_{i_2})=\phi_{i_2}(\widetilde{r}_1,...,\widetilde{r}_1,\widetilde{r}_2)$. If $q_{i_2+1}\notin \cup_{j=1,2}[\widetilde{q}_j-\delta,\widetilde{q}_j+\delta]$, the we have $i_3= i_2+1$. Otherwise if $q_{i_2+1}\in[\widetilde{q}_1-\delta,\widetilde{q}_1+\delta]$, then we have $r_{i_2+1}=\widetilde{r}_1$ and $q_{i_2+2}=\phi_{i_2+1}(\widetilde{r}_1,...,\widetilde{r}_1,\widetilde{r}_2,\widetilde{r}_1)$; similarly if $q_{i_2+1}\in[\widetilde{q}_2-\delta,\widetilde{q}_2+\delta]$, then $q_{i_2+2}=\phi_{i_2+1}(\widetilde{r}_1,...,\widetilde{r}_1,\widetilde{r}_2,\widetilde{r}_2)$. As such we can reconstruct the queries $q_{i_2+3},q_{i_2+4}$ and so on until we find $j>i_2$ where $\phi_j\notin\cup_{j=1,2}[\widetilde{q}_j-\delta,\widetilde{q}_j+\delta]$. Then we have determined $i_3=j$ and $\widetilde{q}_3=q_j$, which is 
completely determined by $(\widetilde{r}_1, \widetilde{r}_2)$. 
Following the same argument, the entire $\widetilde{q}$ sequence can be reconstructed from $\widetilde{r}$. Consequently,
\[
\left|\left\{\widetilde{q}: X^*\in [0,1]\right\}\right|\leq \left|\left\{\widetilde{r}: X^*\in [0,1]\right\}\right|.
\]
Suppose $K+1$ is the maximum length of $\widetilde{q}\equiv \widetilde{q}(X^*)$ among all $X^* \in [0,1]$. Then the total number of distinct binary $\widetilde{r}$ sequences is at most $\sum_{k\leq K+1}2^k< 2^{K+2}$. In addition if $\widetilde{q}$ is of length $k$, then $X^*\in [\widetilde{q}_k-\delta,\widetilde{q}_k+\delta]$ by construction. Hence
\[
1=\left|[0,1] \right| \leq \left|\cup_{X^*\in [0,1]}\left[\widetilde{q}_k-\delta,\widetilde{q}_k+\delta\right]\right|
\leq 2\delta\left|\left\{\widetilde{q}: X^*\in [0,1]\right\}\right|
\leq 2\delta\cdot 2^{K+2}.
\]
Deduce that $K\geq \log(1/\delta)-3$. In other words, there exists $X^* \in [0,1]$ for which $\widetilde{q}$ is of length $k$ where $k\geq K+1\geq \log(1/\delta)-2$. We choose $I=[\widetilde{q}_k-\delta,\widetilde{q}_k+\delta]$ for such $\widetilde{q}$ and show that it satisfies the statement in (i). By construction all the queries in $\widetilde{q}$ are more than $\delta$ apart; therefore,
all the queries in $\widetilde{q}$ except $\widetilde{q}_k$ are all outside of $I$. As a result for all $X\in I$ and $i\leq k-1$, $\mathds{1}\{X\geq \widetilde{q}_i\}$ yields the same response as $\mathds{1}\{X^*\geq \widetilde{q}_i\}$. Deduce that $\widetilde{q}(X)=\widetilde{q}(X^*)$ for all $X\in I$.
%\nbr{I do not understand why this is needed for the
%proof of claim (i). In fact, this shows a slightly stronger claim: There exists an interval $I$ of length $2\delta$
%and a sequence of queries $\widetilde{\mathcal{Q}}=\{\widetilde{q}_1,..., \widetilde{q}_K\}$
%with $K\geq \log (1/\delta)-3$ and $|\widetilde{q}_i-\widetilde{q}_j|>\delta$ for all $i\neq j$
%such that $\widetilde{\mathcal{Q}} \subseteq \mathcal{Q}(X^*)\backslash I$ for all $X^* \in I.$}
To complete the proof of (i), take $\widetilde{\mathcal{Q}}=\{\widetilde{q}_1,...,\widetilde{q}_{k-1}\}$ to obtain a subset of $\mathcal{Q}(X^*)\backslash I$ of size at least $K\geq \log(1/\delta)-3$.

{\it Proof of (ii):} 
%This part of the proof borrows the idea from the lower bound proof in~\cite[Theorem~4.1]{tsitsiklis2018private}. 
For $q=q(X^*)=(q_1,...,q_n)$, let $\overline{\mathcal{Q}}(X^*)=\{q_1,...,q_n,0,1\}$. The key observation is that for each $x$ in the information set $\mathcal{I}(q)$, there must be two queries $s,t\in \overline{\mathcal{Q}}(X^*)$ with $s\leq x$, $t>x$ and $t-s\leq \epsilon$. Otherwise when $x$ is the truth, the learner could not have achieved $\epsilon$-accuracy through the query sequence $q$. The inclusion of $0,1$ in $\overline{\mathcal{Q}}(X^*)$ is because even if they are never queried, they could still serve in these $(s,t)$ pairs.

Let 
\[
\mathcal{P}=\{(s,t): s,t\in \overline{\mathcal{Q}}(X^*), 0<t-s\leq \epsilon\}
\]
denote the set of all pairs of queries that are no more than $\epsilon$-apart. We have
\[
\mathcal{I}(q)\subseteq \cup_{(s,t)\in \mathcal{P}}[s,t].
\]
From the definition of $(\delta,L)$-privacy, the $\delta$-covering number of $\cup_{(s,t)\in \mathcal{P}}[s,t]$ is at least $L$, which immediately implies $|\mathcal{P}|\geq L$. However since we want to lower bound the number of pairs $(s,t)$ where both $s$ and $t$ are outside of $I$, the proof is slightly more complicated. Write $I=[a,b]$. If one of $s,t$ is in $I$, then $[s,t]\subseteq [a-\epsilon,b+\epsilon]\subseteq [a-\delta/2,b+\delta/2]$. This is an interval of length $3\delta$. We also need to discount the pairs that use 0 or 1 as one of the endpoints. Let
\begin{align*}
\widetilde{\mathcal{P}}= & \{(s,t): s,t\in \overline{\mathcal{Q}}(X^*)\backslash (I\cup \{0,1\}), 0<t-s \leq \epsilon\}\\
\supseteq & \{(s,t)\in \mathcal{P}: [s,t]\subseteq [0,1]\backslash ([a-\delta/2,b+\delta/2]\cup [0,\delta]\cup [1-\delta,1])\}.
\end{align*}
The $\delta$-covering number for $[a-\delta/2,b+\delta/2]\cup [0,\delta]\cup [1-\delta,1])$ is at most 5. Deduce that the $\delta$-covering number for $\cup _{(s,t)\in\widetilde{\mathcal{P}}}[s,t]$ is at least $L-5$. Thus $|\widetilde{\mathcal{P}}|\geq L-5$.

{\it Proof of (iii):} The part of the proof is similar to the proof of (i). We take $\widetilde{q}(X^*)$ to be the subsequence of $q(X^*)$ that contains all the queries in $\mathcal{Q}(X^*)$ that are in $I$. Let $J(X^*)$ be the interval formed by the two queries in $q(X^*)$ to the left and right of $X^*$ that are the closest to $X^*$. For all $X^*\in I$, $X^*\in J(X^*)$ and thus $I\subseteq \cup_{X^*\in I}J(X^*)$.
%First define an oracle querying strategy $\widetilde{\phi}$ where the learner knows that $X^*\in I$ so the learner skips all the queries outside of $I$. Since the querying strategy $\phi$ is $\epsilon$-accurate, for each $X^*\in I$, the oracle learner can locate an interval $J(X^*)$ of length at most $\epsilon$ that contains $X^*$. Deduce that $I\subseteq \cup_{X^*\in I}J(X^*)$. 
Since $|I|=2\delta$ and the querying strategy $\phi$ is $\epsilon$-accurate so that  
$|J(X^*)| \le \epsilon$, we have that $\{J(X^*):X^*\in I\}$ contains at least $2\delta/\epsilon$ distinct members. 

Let $\widetilde{r}_i(X^*)=\mathds{1}\{X^*\geq \widetilde{q}_i(X^*)\}$.
Next we show that for each $X^*\in I$,  $J(X^*)$ is completely determined by $\widetilde{r}(X^*)$. 
Indeed given any $X^*\in I$, the responses to the queries outside of $I$ can be deduced from their  position relative to $I$. Therefore from only $\widetilde{r}(X^*)$, which only contains responses to the queries in $I$, one can reconstruct the entire query sequence $q(X^*)$, from which one can infer $J(X^*)$. Thus
\[
\left|\left\{\widetilde{r}(X^*):X^*\in I\right\}\right|
\geq \left|\left\{J(X^*):X^*\in I\right\}\right|
\geq 2\delta/\epsilon.
\]
Suppose $T$ is the maximal length of $\widetilde{q}(X^*)$ among all $X^*\in I$. 
Then $\widetilde{r}(X^*)$ can take no more than $\sum_{t\leq T}2^t<2^{T+1}$ distinct values. Deduce that $T\geq \log(\delta/\epsilon)$. Recall that $\widetilde{q}(X^*)$ is a subsequence of $q(X^*)$ and only contains queries in $I$. Conclude that there exists $X^*\in I$ for which the  querying strategy $\phi$ submits at least $\log (\delta/\epsilon)$ queries that are in $I$.
\end{proof}

\section{Proofs of the multidimensional results}\label{sec:d.dim}
In this section we prove our multidimensional results Theorem~\ref{thm:bayes.ddim} and Theorem~\ref{thm:freq.ddim}. To avoid repetition we only outline the proofs and highlight the parts that differ from the one-dimensional case.
\begin{proof}[Proof of Theorem~\ref{thm:bayes.ddim}]
{\bf Upper bound:} Consider the following multistage querying strategy:
\begin{enumerate}
\item For each $i=1,...,d$, submit $K_1=\lfloor \log (1/ (\lceil \Ld\rceil \delta))\rfloor$ queries on $X_i^*$ via bisection. This stage locates $X^*$ in a cube $J=[a_1,b_1]\times...\times [a_d,b_d]$ of diameter $2^{-K_1}\geq \lceil \Ld\rceil \delta$. This stage involves $dK_1$ queries.
\item For each $i=1,...,d$, run replicated bisection on $[a_i,b_i]$: First evenly split $[a_i,b_i]$ into $\lceil \Ld\rceil$ subintervals. For all $q$ that are endpoints of the subintervals, query the events $\{X_i^*\geq q\}$. From the responses determine the true subinterval $X_i^*$ is in. Run bisection on this subinterval to find $X_i^*$ up to $\epsilon$-accuracy and submits cloned queries in all the other subintervals. This stage involves $d(\lceil \Ld\rceil-1+ \lceil \Ld\rceil \lceil \log (2\delta/\epsilon)\rceil)$ queries.
\end{enumerate}
To show this strategy is $(\delta,L)$-private, notice that from the adversary's perspective, for each $i$ there are $\lceil \Ld\rceil$ subintervals that contain $X_i^*$ with equal probability. That creates $L'=\lceil \Ld\rceil^d$ cubes $J_1,..., J_{L'}$ that are at least $\delta$ apart in $\|\cdot\|_\infty$ distance. Since $L'\geq L$, the adversary cannot achieve $\|\widetilde{X}-X^*\|_\infty\leq \delta$ with probability greater than $1/L$.

{\bf Lower bound:} Suppose $\phi$ is an $\epsilon$-accurate and $(\delta,L)$-private strategy that submits at most $n$ queries. By assumption the query sequence on $X_i^*$ depends only on $X_i^*$ and $Y_i$. Thus we can write $\mathbf{n}_i(X_i^*,Y_i)$ for the number of queries submitted on the $i$'th coordinate. Let $n_i=\sup_{X_i^*,Y_i} \mathbf{n}_i(X_i^*,Y_i)$, so that $n \ge \sum_{i\leq d}n_i$. 
 As in the one-dimensional proof we can assume that the learner always submits exactly $n_i$ queries on $X_i^*$ by filling up the end of the query sequence with trivial queries on $\{X_i^*\geq 0\}$.

Consider an adversary that adopts the truncated proportional-sampling scheme on each coordinate. Let $q_i=(q_{i,1},...,q_{i, n_i})$ denote the sequence of queries submitted on $X_i^*$. The adversary obtains $\widetilde{X}_i$ by sampling from the empirical distribution of $q_{i,K_2+1},..., q_{i,n_i}$ with $K_2=\lfloor \log (1/ ( \Ld\delta))\rfloor$. Since $q_i$ only depends on $X_i^*$ and $Y_i$ with $\{(X_i^*,Y_i)\}_{i\leq d}$ mutually independent, we have that the sequences $q_1,...,q_d$ are also mutually independent. Thus
\begin{align*}
\mathbb{P}\left\{\left\|\widetilde{X}-X^*\right\|_\infty\leq \delta/2\right\}
= & \mathbb{E}\prod_{i\leq d}\left(\frac{\sum_{j=K_2+1}^{N_i}\mathds{1}\{|X_i^*-q_{i,j}|\leq \delta/2\}}{n_i-K_2}\right)\\
= & \prod_{i\leq d}\left(\frac{\sum_{j=K_2+1}^{N_i}\mathbb{P}\{|X_i^*-q_{i,j}|\leq \delta/2\}}{n_i-K_2}\right)\leq \frac{1}{L}.
\end{align*}
Via the same analysis as in the one-dimensional proof,
\[
\sum_{j=K_2+1}^{n_i}\mathbb{P}\{|X_i^*-q_{i,j}|\leq \delta/2\} \geq \log (\delta/4\epsilon)- \delta 2^{K_2}.
\]
Deduce that
\[
\prod_{i\leq d}(n_i-K_2)\geq L\left(\log (\delta/4\epsilon)- \delta 2^{K_2}\right)^d.
\]
Given the lower bound on $\prod_{i\leq d}(n_i-K_2)$, the minimal value for $\sum_{i\leq d}(n_i-K_2)$ is achieved when all the summands are equal. Hence the total number of queries is at least
\[
\sum_{i\leq d}n_i \geq dK_2 + d\Ld\left(\log (\delta/4\epsilon)- \delta 2^{K_2}\right)
\geq d\left(K_2 + \Ld\left(\log\frac{\delta}{\epsilon}-2\right)-1\right),
\]
where the second inequality is from the choice of $K_2$.
\end{proof}

\begin{proof}[Proof of Theorem~\ref{thm:freq.ddim}]
{\bf Upper bound:} As in the one-dimensional proof, the upper bound is proved by constructing a querying strategy that first submits $L$ guesses (intervals of length $\epsilon$) that are at least $\delta$ apart. In $[0,1]^d$, we submit $\lceil \Ld\rceil$ guesses on the location of $X_i^*$ for each coordinate $i\leq d$. These guesses across the $d$ coordinates form $L'=\lceil \Ld\rceil^d $ cubes of diameter $\epsilon$, all of which are contained in the information set $\mathcal{I}(\bar{q})$. Moreover the centers of these $L'$ cubes are at least $\delta$ away from each other in $\|\cdot\|_\infty$ norm. Since $L'\geq L$, the $\delta$-covering number of $\mathcal{I}(\bar{q})$ is at least $L$.

The way the guesses are submitted following algorithm~\ref{alg:small.delta} when $\delta\leq 2^{-\lceil \Ld\rceil}$ and algorithm~\ref{alg:large.delta} when $\delta>2^{-\lceil \Ld\rceil}$, except that $L$ is replaced with $\lceil \Ld\rceil$ in the algorithms.
Each guess consists of two queries $\epsilon$ away from each other. In total it takes $2d\lceil \Ld\rceil$ queries to submit all the guesses. If none of the guesses is correct, the guesses help the learner narrow down the range of $X_i^*$ to an interval $J_i$. Via similar analysis as in the one-dimensional proof, we have $|J_i|=2^{-\lceil\Ld\rceil}$ when $\delta\leq 2^{-\lceil\Ld\rceil}$ and $|J_i|\in [\delta,2\delta]$ otherwise. The next stage of the strategy simply runs bisection in $J_i$ to achieve $\epsilon$-accuracy on $X_i^*$, which requires at most $\log (|J_i|/\epsilon)\leq \log (\max\{2^{-\lceil \Ld\rceil}, \delta\}/\epsilon)+1$ queries.

{\bf Lower bound:} First consider the case $\delta\leq 2^{-\lceil \Ld\rceil}$. 
%\nbr{Shoudn't be $\delta \leq 2^{-\lceil \Ld\rceil}$}
From the lower bound proof of Theorem~\ref{thm:freq}, we have for each $i\leq d$,
\begin{enumerate}[(i)]
\item There exists an interval $J_i$ of length $2\delta$ such that if $X_i^*$ is in $J_i$, then there are at least $\log (1/\delta)-3$ queries on $X_i^*$ that are outside of $J_i$ and are separated from each other by at least $\delta$;
\item For each interval $J$ of length $2\delta$, there exists $x_i\in J$ such that if $x_i$ is the true value for $X_i^*$, then there are at least $\log (\delta/\epsilon)$ queries on $X_i^*$ that are in $J$.
\end{enumerate}
The above guarantee that there are at least $d(\log (1/\epsilon)-3)$ queries in total. The extra queries arise from the privacy requirement. For a point $x$ to enter the information set $\mathcal{I}(\bar{q})$, there must be at least $2$ queries on $X_i^*$ surrounding $x_i$, that are $\epsilon$-close to each other. Hence $\mathcal{I}(\bar{q})$ is contained in the union of $d$-dimensional hyperrectangles $\prod_{i\leq d}[s_i,t_i]$, where $s_i,t_i$ are pairs of queries on $X_i^*$ with $0<t_i-s_i\leq \epsilon$. Suppose aside from the queries identified by (i), there are $m_i$ extra queries on $X_i^*$ outside of $J_i$. Approximately speaking, the queries outside of $\prod_{i\leq d}J_i$ form at most $\prod_{i\leq d}m_i$ hyperrectangles that are contained in $\mathcal{I}(\bar{q})$. Therefore the $\delta$-covering number of $\mathcal{I}(\bar{q})$ is at most $\prod_{i\leq d}m_i$. Deduce that $\prod_{i\leq d}m_i \geq L$ and thus $\sum_{i\leq d}m_i\geq d\Ld$. The queries identified in (i),(ii) plus $\sum_{i\leq d}m_i$ is approximately the lower bound stated in Theorem~\ref{thm:freq.ddim}. The extra constant $-5$ is to account for the hyperrectangles that are formed with queries that are either near the boundary of $[0,1]^d$ or near $\prod_{i\leq d}J_i$.

The proof for the $\delta> 2^{-\lceil \Ld\rceil}$ 
% \nbr{See above.}
case is much simpler. Fix any cube $J=\prod_{i\leq d} J_i$ of diameter $\delta$. For each $i\in [d]$, there are at least $\log (\delta/\epsilon)$ queries about $X_i$ inside of $J_i$. Outside of $J$ there are at least $2d\Ld$ queries to ensure that the information set contains at least $L$ hyperrectangles. The proof is similar to the previous case and is therefore omitted.
\end{proof}

\section{Proof of Theorem~\ref{thm:noisy} (noisy responses)}
\subsection{Proof of the upper bounds}\label{sec:proof.upper.noisy}
\subsubsection{Proof of the upper bound in~\eqref{eq:res.mean}}
It suffices to show that under definition (a) of $\epsilon$-accuracy, the querying strategy described in Section~\ref{sec:alg.noisy} is $\epsilon$-accurate, $(\delta,L)$-private, and achieves the upper bound in~\eqref{eq:res.mean}.

{\bf Accuracy:}
Discuss the following events:
\begin{enumerate}
\item $\mathcal{E}_1$: the BZ algorithm returns the wrong subinterval in stage 1. In other words, $X^*\notin I$.
\item $\mathcal{E}_{2,j}$ for $j=1,...,L'-1$: stage 1 does not incur an error, but stage 2 returns a subinterval out of $J_1,..., J_{L'}$ that does not contain $X^*$, with $\widehat{k}$ at distance $j$ away from the correct index. In other words, $\mathcal{E}_{2,j} = \{X^*\in J_{k^*} \text{ for some } k^*\in [L'], \text{ and } |\widehat{k}-k^*|=j\}$.
\item $\mathcal{E}_3$: the BZ algorithm makes an error in stage 3: $X^*\in J_{\widehat{k}}$ but $X^*\notin J$.
\item $\mathcal{E}_4$: $X^*\in J$.
\end{enumerate}
The events above are disjoint, and their union forms the entire probability space. It is also easy to see that $|\widehat{X}-X^*|$ is upper bounded by $1, (j+1)\delta, \delta, \epsilon/8$ on $\mathcal{E}_1, \mathcal{E}_{2,j}, \mathcal{E}_3, \mathcal{E}_4$ respectively. Hence
\begin{equation}
\label{eq:mean.split}
\mathbb{E}\left|\widehat{X}-X^*\right|
\leq \mathbb{P}\left\{\mathcal{E}_1\right\} + \sum_{j=1}^{L'-1} (j+1)\delta\mathbb{P}\left\{\mathcal{E}_{2,j}\right\} + \delta\mathbb{P}\left\{\mathcal{E}_3\right\} + \tfrac{\epsilon}{8}\mathbb{P}\left\{\mathcal{E}_4\right\}.
\end{equation}
We claim that all events but $\mathcal{E}_4$ occur with low probability. Firstly, it follows from Lemma~\ref{lmm:BZ} that
\begin{align}
\label{eq:event.1}\mathbb{P}\left\{\mathcal{E}_1\right\}\leq & \frac{1}{L'\delta}2^{-c_3(p)K_1}\leq \frac{\epsilon}{8},\\
\label{eq:event.3}\mathbb{P}\left\{\mathcal{E}_3\right\}\leq & \frac{\delta}{\epsilon/4}2^{-c_3(p)K_2} \leq \frac{\epsilon}{8\delta}
\end{align}
from the choice of $K_1,K_2$. 

To handle $\mathcal{E}_{2,j}$, note that conditional on $X^*\in J_{k^*}$,
\[
m_i\stackrel{indep}{\sim}
\begin{cases}
\text{Binomial}(m,p) & \text{for}\quad 1\leq i \leq k^*-1; \\
\text{Binomial}(m,1-p) & \text{for}\quad k^*\leq i\leq L'.
\end{cases}
\]
Hence $\widehat{k}=\arg\max_{1\leq k\leq L'}\sum_{i=1}^{k-1}m_i+\sum_{i=k}^{L'-1}(m-m_i)$ is the maximum likelihood for $k^*$, and for all $k\neq k^*$,
\begin{align*}
\mathbb{P}\left\{\widehat{k}=k \rvert X^*\in J_{k^*}\right\}\geq &  \mathbb{P}\left\{\sum_{i=1}^{k-1}m_i+\sum_{i=k}^{L'-1}(m-m_i) \leq \sum_{i=1}^{k^*-1}m_i+\sum_{i=k^*}^{L'-1}(m-m_i)\right\}\\
= & \mathbb{P}\left\{B\leq \frac{|k-k^*|m}{2}\right\},\hfill\text{for some }B\sim \text{Binomial}(|k-k^*|m, p),
\end{align*}
which is further bounded by $2^{-c_4(p) |k-k^*|m}$ from the binomial tail bound~\cite[Theorem~2.1]{mulzer2018five}. Deduce that
\begin{equation}
\label{eq:event.2.j}
\mathbb{P}\left\{\mathcal{E}_{2,j}\right\}\leq \mathbb{P}\{X^*\in J_{k^*}, \widehat{k}=k^*-j\} + \mathbb{P}\{X^*\in J_{k^*}, \widehat{k}=k^*+j\}\leq 2\cdot 2^{-c_4(p) jm}.
\end{equation}
Thus
\begin{align}
\nonumber\sum_{j=1}^{L'-1} (j+1)\delta\mathbb{P}\left\{\mathcal{E}_{2,j}\right\} \leq & 2\sum_{j=1}^{L'-1} (j+1)\delta 2^{-c_4(p) jm}\\
\nonumber \leq & 2\delta\left(\sum_{j=1}^\infty 2^{-c_4(p)jm} + \sum_{i=1}^\infty \sum_{j=i}^\infty 2^{-c_4(p)jm}\right)\\
= & \frac{2\delta 2^{-c_4(p) m}}{1-2^{-c_4(p)m}}\left(1+\frac{1}{1-2^{-c_4(p) m}}\right)\\
\label{eq:event.2} \leq & 8\delta 2^{-c_4(p) m}\leq \epsilon/8,
\end{align}
where the last two inequalities are due to the choice $m=\frac{1}{c_4(p)}\log \frac{64\delta}{\epsilon}$ and $\delta\geq 2\epsilon$.

Combining~\eqref{eq:mean.split}-\eqref{eq:event.2} yields that
\[
\mathbb{E}\left|\widehat{X}-X^*\right| \leq \epsilon/8 + \epsilon/8 + \epsilon/8 +\epsilon/8 = \epsilon/2.
\]

{\bf Privacy:}
The goal is to show that for all adversary's estimators $\widetilde{X}$ that could depend on $q$, 
\[
\mathbb{P}\left\{|\widetilde{X}-X^*|\leq \delta/2\right\}\leq \frac{1}{L}.
\]
In the multi-stage algorithm the learner first runs the BZ algorithm on a $L'\delta$-fine grid to obtain an interval estimator $I$, then runs replicated BZ on the $L'$ subintervals $J_1,..., J_{L'}$ of $I$. Write $I^*$ for the true subinterval on the $L'\delta$-fine grid that contains $X^*$. When $X^*\in I$, {\it i.e.} $I^*=I$, we used $k^*$ to index the true subinterval out of $J_1,...,J_{L'}$ that contains $X^*$. For the proof of $(\delta,L)$-privacy, we need to expand the definition of $k^*$ to incorporate the case $X^*\notin I$ as well. Label the $L'$ length-$\delta$ subintervals of $I^*$ as $J^*_1,..., J^*_{L'}$ and define $k^*$ so that $X^*\in J^*_{k^*}$. Recall that $\widehat{k}$ is the learner's estimator of $k^*$. We have
\begin{align}
\nonumber &\mathbb{P}\left\{|\widetilde{X}-X^*|\leq \delta/2\right\}\\
\label{eq:term.1}\leq & \mathbb{P}\{X^*\notin I\}\\
\label{eq:term.2} + & \mathbb{P}\left\{|\widetilde{X}-X^*|\leq \delta/2, \widehat{k}<k^*, X^*\in I\right\}\\
\label{eq:term.3} + & \mathbb{P}\left\{|\widetilde{X}-X^*|\leq \delta/2, \widehat{k}>k^*, X^*\in I\right\}\\
\label{eq:term.4} + & \mathbb{P}\left\{|\widetilde{X}-X^*|\leq \delta/2, \widehat{k}=k^*, X^*\in I\right\}.
\end{align}
Of the four terms above, the first term equals the probability that the algorithm makes a mistake in the BZ algorithm in the first stage:
\[
\eqref{eq:term.1}= \mathbb{P}\mathcal{E}_1\leq \frac{1}{L'\delta}2^{-c_3(p)K_1}\leq \frac{1}{L'}
\]
where the last inequality holds due to $K_1 = \frac{1}{c_3(p)}\log \frac{8}{L'\epsilon \delta}  \geq \frac{1}{c_3(p)}\log(1/\delta)$ in view of $2\epsilon \le 1/L$.

For the second term, use the Bayes rule to write
\begin{equation}\label{eq:term.2.bayes}
\eqref{eq:term.2} = \mathbb{P}\left\{\widehat{k}<k^*, X^*\in I\right\}\mathbb{P}\left\{|\widetilde{X}-X^*|\leq \frac{\delta}{2}\mid \widehat{k}<k^*, X^*\in I\right\}.
\end{equation}
Since $\widetilde{X}$ is a function of $q$,
\begin{align}
\nonumber& \mathbb{P}\left\{|\widetilde{X}-X^*|\leq \delta/2\mid \widehat{k}<k^*, X^*\in I\right\}\\
\nonumber\leq & \mathbb{E}_q\left(\sup_{t\in [0,1]}\mathbb{P}\left\{X^*\in [t-\delta/2, t+\delta/2]\cap I\mid \widehat{k}<k^*, X^*\in I, q\right\}\right)\\
\label{eq:term.2.max}\leq & 2\mathbb{E}_q\left(\max_{k\leq L'}\mathbb{P}\left\{k^*=k\mid \widehat{k}<k^*, X^*\in I, q\right\}\right).
\end{align}
The last inequality is because all intervals of the form $[t-\delta/2, t+\delta/2]\cap I$ must be covered by the union of two consecutive subintervals $J_k\cap J_{k+1}$ for some $k$. 

Next we show that for all $q$ and $k$,
\[
\mathbb{P}\left\{k^*=k\mid \widehat{k}<k^*, X^*\in I, q\right\} = \mathbb{P}\left\{k^*=k\mid \widehat{k}<k^*, X^*\in I\right\}, i.e.,
\]
\[
\mathcal{L}(q\mid k^*=k,\widehat{k}<k^*, X^*\in I) = \mathcal{L}(q\mid \widehat{k}<k^*, X^*\in I).
\]
In other words, $k^*$ is independent of $q$ conditional on $\widehat{k}<k^*$ and $X^*\in I$. Denote the queries submitted in the three stages as $q^{(1)}$, $q^{(2)}$ and $q^{(2)}$. We will establish conditional independence in two steps:
\begin{enumerate}
\item Show that $(q^{(1)}, q^{(2)})$ is independent of $k^*$ conditional on $\widehat{k}<k^*$ and $X^*\in I$:

note that given $I^*$, the conditional distribution of $X^*$ is uniform on $I^*$. Therefore $k^*$ is distributed uniformly on $[L']$ and is independent of $I^*$. Since the BZ algorithm only queries the endpoints of the subintervals, the distribution of the responses in the first stage $r^{(1)}$ only depends on $X^*$ through $I^*$. Hence $k^*$ is independent of the tuple $(I^*, r^{(1)})$. Moreover, $r^{(1)}$ completely determines $I$, so that $k^*$ is independent of $(I^*, I, r^{(1)})$. On the other hand, when $X^*\in I$, $\widehat{k}$ is can be written as $f(k^*,\textbf{noise}^{(2)})$, a function of only $k^*$ and the binary noise variables in the second stage. We have
\begin{align*}
&\mathcal{L}\left(r^{(1)}\mid k^*=k, \widehat{k}<k^*, X^*\in I\right)\\
= & \mathcal{L}\left(r^{(1)}\mid k^*=k, f(k^*,\textbf{noise}^{(2)})<k^*, I^*=I\right)\\
= & \mathcal{L}\left(r^{(1)}\mid \widehat{k}<k^*, I^*=I\right).
\end{align*}
The second equality is because by the independence of $(k^*, \textbf{noise}^{(2)})$ and $(I^*,I,r^{(1)})$, $(k^*, \textbf{noise}^{(2)})$ and $r^{(1)}$ are conditionally independent given $I^*=I$.
%\nbr{JX. It is still unclear to me how the above follows from the independence.} 
%Note that the conditional distribution of $r^{(1)}$ given $I^*=I$ does not depend on any function of $(k^*, \textbf{noise}^{(2)})$. Here we only drop the event $k^*=k$ because we only need to establish conditional independence between $k^*$ and $r^{(1)}$.

We have shown that $r^{(1)}$ and $k^*$ are independent conditional on $\widehat{k}<k^*$ and $X^*\in I$. Notice that $(q^{(1)},q^{(2)})$ is a deterministic function of $r^{(1)}$. Thus $(q^{(1)}, q^{(2)})$ and $k^*$ are also conditionally independent.
\item Show that $q^{(3)}$ is independent of $k^*$ conditional on $\widehat{k}<k^*$, $X^*\in I$, $q^{(1)}$ and $q^{(2)}$:

conditional on $\widehat{k}<k^*$ and $X^*\in I$, all the queries submitted in $\widehat{J}=J_{\widehat{k}}$ are smaller than $X^*$. Therefore the joint distribution of the queries in the third stage that fall in $\widehat{J}$ does not depend on $k^*$. Since the queries in the other subintervals are only copies of those in $\widehat{J}$, we have that $q^{(3)}$ is independent of $k^*$, conditional on $\widehat{k}<k^*$, $X^*\in I$, $q^{(1)}$ and $q^{(2)}$. 
\end{enumerate}

We have shown that the entire query sequence $q$ is independent of $k^*$ conditional on $\widehat{k}<k^*$ and $X^*\in I$. Thus
\[
\mathbb{P}\left\{k^*=k\mid \widehat{k}<k^*, X^*\in I, q\right\} = \mathbb{P}\left\{k^*=k\mid \widehat{k}<k^*, X^*\in I\right\}.
\]

Combine with~\eqref{eq:term.2.bayes} and~\eqref{eq:term.2.max} to obtain
\begin{align*}
\eqref{eq:term.2}\leq & \mathbb{P}\left\{\widehat{k}<k^*, X^*\in I\right\}\cdot 2\max_{k\leq L'}\mathbb{P}\left\{k^*= k\mid \widehat{k}<k^*, X^*\in I\right\}\\
= & 2\max_{k\leq L'}\mathbb{P}\left\{k^*=k,\widehat{k}<k^*, X^*\in I\right\}\\
\leq & 2\max_{k\leq L'}\mathbb{P}\left\{k^*=k\right\}=2/L'.
\end{align*}
where the last inequality is because $k^*$ is distributed uniformly on $[L']$. 

Following the same arguments, $\eqref{eq:term.3}\leq 2/L'$. Next we handle term~\eqref{eq:term.4}. 

Once again because $\widetilde{X}$ is a function of $q$, we have
\[
\eqref{eq:term.4} \leq  2\mathbb{E}\left(\max_{k\leq L'}\mathbb{P}\left\{k^*= k\mid \widehat{k}=k^*, X^*\in I, q\right\}\right).
\]
We claim that $k^*$ and $q$ are independent conditional on $\widehat{k}=k^*$ and $X^*\in I$. By the same arguments as in the analysis of~\eqref{eq:term.2} we can show that $k^*$ is independent of $(q^{1},q^{(2)})$ conditional on $\widehat{k}=k^*$ and $X^*\in I$. It remains to show that $k^*$ is independent of $q^{(3)}$ conditional on $\widehat{k}=k^*$, $X^*\in I$ and $(q^{(1)},q^{(2)})$. In other words,
\[
\mathcal{L}\left(q^{(3)}\mid k^*=k, \widehat{k}=k^*, X^*\in I, q^{(1)}, q^{(2)}\right) = \mathcal{L}\left(q^{(3)}\mid  \widehat{k}=k^*, X^*\in I, q^{(1)}, q^{(2)}\right)
\]
%Use the Bayes rule to write
%\[
%\mathbb{P}\left\{k^*= k\mid \widehat{k}=k^*, X^*\in I, q\right\} = \mathbb{P}\left\{k^*=k\mid \widehat{k}=k^*, X^*\in I\right\}\frac{\mathbb{P}\left\{q\mid k^* = k, \widehat{k}=k, X^*\in I\right\}}{\mathbb{P}\left\{q\mid \widehat{k}=k, X^*\in I\right\}}.
%\]
To show the above, first note that conditional on $k^*=k, \widehat{k}=k^*$, $X^*\in I$ and $(q^{(1)},q^{(2)})$, $X^*$ is distributed uniformly on $J_k$. The queries sequence $q^{(3)}$ are generated from running the BZ algorithm on $J_k$, and replicating in the other subintervals. The conditional distribution of the $q^{(3)}$ is therefore independent of the value of $k$. Thus
\begin{align*}
\eqref{eq:term.4}= & \mathbb{P}\left\{|\widetilde{X}-X^*|\leq \delta/2\mid \widehat{k}=k^*, X^*\in I\right\}\mathbb{P}\left\{\widehat{k}=k^*,X^*\in I\right\}\\
\leq & 2\max_{k\leq L'}\mathbb{P}\left\{k^*=k\mid \widehat{k}=k^*, X^*\in I\right\}\mathbb{P}\left\{\widehat{k}=k^*,X^*\in I\right\}\\
= & 2\max_{k\leq L'}\mathbb{P}\left\{k^*=k, \widehat{k}=k^*, X^*\in I\right\}\\
\leq & 2\max_{k\leq L'}\mathbb{P}\left\{k^*=k\right\} = 2/L'.
\end{align*}

Collect all the terms to deduce that
\[
\mathbb{P}\left\{|\widetilde{X}-X^*|\leq \delta/2\right\} \leq \frac{1}{L'} + \frac{2}{L'} + \frac{2}{L'} + \frac{2}{L'}= \frac{1}{L}
\]
by picking $L'=7L$.

{\bf Query complexity:} the total number of queries submitted by the querying strategy is
\begin{align*}
K_1 + (L'-1) m + L' K_2 = & \frac{1}{c_3(p)}\log\frac{8}{7\epsilon L\delta} + (7L-1)\frac{1}{c_4(p)}\log\frac{64\delta}{\epsilon} + \frac{14L}{c_3(p)}\log\frac{4\sqrt{2}\delta}{\epsilon}\\
\leq & \frac{1}{c_3(p)}\log\frac{8}{7\epsilon L\delta} + \left(\frac{14}{c_3(p)}+\frac{7}{c_4(p)}\right)L\log\frac{64\delta}{\epsilon}\\
\leq & \left(\frac{14}{c_3(p)}+\frac{7}{c_4(p)}\right)\left(\log \frac{1}{\epsilon} + L\log\frac{64\delta}{\epsilon}\right).
\end{align*}
where the last inequality is because $L\delta\geq 2\delta\geq 2\epsilon$.

\subsubsection{Proof of the upper bound in~\eqref{eq:res.prob}}
The proof is similar to that of the upper bound in~\eqref{eq:res.mean}. Adopt the multi-stage querying strategy described in Section~\ref{sec:alg.noisy}, with the parametrization in~\eqref{eq:param.prob}. We claim that this querying strategy is $(\epsilon,M)$-accurate and $(\delta, L)$-private with the desired query complexity as stated in the upper bound of~\eqref{eq:res.prob}.

{\bf Accuracy:} 
Recall the events $\mathcal{E}_1$, $\mathcal{E}_{2,j}, \mathcal{E}_3, \mathcal{E}_4$ defined in the proof of~\eqref{eq:res.mean}. We have
\begin{align*}
\mathbb{P}\left\{|\widehat{X}-X^*|>\epsilon/2\right\} \leq & \mathbb{P}\mathcal{E}_4^c 
= \mathbb{P}\mathcal{E}_1 + \sum_{j=1}^{L'-1}\mathbb{P}\mathcal{E}_{2,j} + \mathbb{P}\mathcal{E}_3\\
\leq & \frac{1}{L'\delta}2^{-c_3(p) K_1} + 2\sum_{j=1}^{L'-1}2^{-c_4(p) jm} + \frac{\delta}{\epsilon/4}2^{-c_3(p) K_2}\\
\leq & \frac{1}{L'\delta}2^{-c_3(p) K_1} + 4\cdot 2^{-c_4(p) m} + \frac{\delta}{\epsilon/4}2^{-c_3(p) K_2}
\end{align*}
where the second inequality follows from~\eqref{eq:event.1},~\eqref{eq:event.3} and~\eqref{eq:event.2.j}. Plug in the the values of $K_1,K_2$ and $m$ to conclude that $\mathbb{P}\{|\widehat{X}-X^*|>\epsilon/2\}\leq 1/(3M)+1/(3M)+1/(3M) = 1/M$.

{\bf Privacy:} the proof for $(\delta,L)$-privacy is almost identical to the proof of~\eqref{eq:res.mean}. The only part that differs is in the treatment of the term~\eqref{eq:term.1} due to a different choice of $K_1$. We have
\[
\eqref{eq:term.1}= \mathbb{P}\{X^*\notin I\}= \mathbb{P}\mathcal{E}_1\leq \frac{1}{L'\delta}2^{-c_3(p)K_1}\leq \frac{1}{L'}
\]
for $K_1=\frac{1}{c_3(p)}\log(3M/\delta)$.

{\bf Query complexity:} the total number of queries submitted is
\begin{align*}
K_1 + (L'-1) m + L' K_2 = & \frac{1}{c_3(p)}\log\frac{3M}{\delta} + (7L-1) \frac{1}{c_4(p)}\log(12 M) + \frac{7L}{c_3(p)}\log\frac{12 M\delta}{\epsilon}\\
\leq & \left(\frac{8}{c_3(p)} + \frac{7}{c_4(p)}\right)\left(\log\frac{1}{\epsilon} + L\log\frac{12 M\delta}{\epsilon}\right).
\end{align*}

\subsection{Proof of the lower bounds}\label{sec:proof.lower.noisy}
\subsubsection{An auxiliary lemma}
The lower bound proofs rely heavily on the following auxiliary lemma, which connects the expected number of queries near $X^*$ with the learner's error probability.

\begin{lemma}\label{lmm:learner.error}
For each deterministic interval $J$ and $\eta>0$,
\[
\mathbb{E}\left(\text{the number of queries in }J\mid X^*\in J\right) \geq \frac{1}{c_2(p)} \left(\left(1-\mathbb{P}\{|\widehat{X}-X^*|>\eta\mid X^*\in J\}\right)\log\frac{|J|}{2\eta}-1\right).
\]
\end{lemma}
%\nbr{JX. Why we cannot absorb the average over $j$ into the lemma and restate the conclusion as 
%\begin{align*}
%\mathbb{E}\left(\text{the number of queries in }[X^*-\delta/2,X^*+\delta/2]\right)
%\ge \frac{1}{c_2(p)} \left(\left(1-\mathbb{P}\{|\widehat{X}-X^*|>\eta\}\right)\log\frac{\delta}{4\eta}-1\right)
%\end{align*}
%for $0<\delta \le 2$? In this way, we do not need to repeat the step of averaging $j$. 
%For Proof of~(\ref{eq:na.lower.term2}), we can simply set $\delta=2$ and $\eta=\epsilon$.
%Also, maybe we can move the lemma and the proof before Section C.2.1, as this lemma is only for the learner's estimation and has nothing to do with privacy, 
%so it might be of independent interest.}

\begin{proof}
{\bf Step 1: discretize w.r.t. $Y$.}

By the independence between $X^*$ and the random seed $Y$, we can write
\begin{align}
\nonumber & \mathbb{E}\left(\text{the number of queries in } J \mid X^*\in J\right)\\
\nonumber= &  \int_0^1 \mathbb{E}\left(\text{the number of queries in } J \mid X^*\in J, Y=y\right)dy\\
=&\int_0^1 \sum_{i\leq n}\mathbb{P}\left\{q_i\in J\rvert X^*\in J, Y=y\right\} dy.
\label{eq:tonelli}
\end{align}

{\bf Step 2: establish a rate of information transfer.}

In this step we show that for all $y\in [0,1]$,
\begin{equation}
\label{eq:info.rate}
I(X^*; r_1,..., r_n\rvert X^*\in J, Y=y)\leq c_2(p)\sum_{i\leq n}\mathbb{P}\left\{q_i\in J\mid X^*\in J, Y=y\right\}.
\end{equation}
Recall that $c_2(p)=h(1/2)-h(p)$ where $h(t)= H(\text{Bern}(t))=-t\log t- (1-t)\log (1-t)$. The intuition behind~\eqref{eq:info.rate} is that since the observed responses are passed through a binary symmetric channel that flips the noiseless responses with probability $1-p$, each query reveals at most $h(1/2)-h(p)$ bits of information about $X^*$. Next we prove~\eqref{eq:info.rate}. We abbreviate $\{X^*\in J, Y=y\}$ as $\mathcal{E}_{J,y}$. Start from the left-hand side:
\begin{align}
\nonumber& I(X^*;r_1,..., r_n\rvert \mathcal{E}_{J,y})\\
\nonumber = & I(X^*;r_1 \rvert \mathcal{E}_{J,y}) + \sum_{i=1}^{n-1}I(X^*; r_{i+1}\rvert \mathcal{E}_{J,y}, r_1,...,r_i)\\
\nonumber = & H(r_1\rvert \mathcal{E}_{J,y}) - H(r_1\rvert X^*, \mathcal{E}_{J,y}) +\\
\label{eq:chain.rule}&  \sum_{i=1}^{n-1} \left(H(r_{i+1}\rvert \mathcal{E}_{J,y}, r_1,..., r_i) - H(r_{i+1}\rvert X^*, \mathcal{E}_{J,y}, r_1,..., r_i)\right).
\end{align}
To analyze the $i$'th summand in~\eqref{eq:chain.rule}, write
\[
H(r_{i+1}\rvert \mathcal{E}_{J,y}, r_1,..., r_i) = \sum_{\rho_1,..., \rho_i}\mathbb{P}\{r_1=\rho_1,..., r_i=\rho_i\rvert \mathcal{E}_{J,y}\}H(r_{i+1}\rvert \mathcal{E}_{J,y}, r_1=\rho_1,..., r_i=\rho_i).
\]
Recall that $q_{i+1}=\phi_i (r_1,..., r_i, Y)$. Therefore with the values of $r_1,..., r_i, Y$ fixed, the $i+1$'th query is deterministic. If it lands to the left of $J$, then conditional on $X^*\in J$, we have $X^*\geq q_{i+1}$ with (conditional) probability 1, implying that $r_{i+1} \sim \text{Bern}(p)$. We have
\[
H(r_{i+1}\rvert \mathcal{E}_{J,y}, r_1=\rho_1,..., r_i=\rho_i)  = h(p), \quad \text{if } \phi_i(\rho_1,..., \rho_i, y)\text{ is to the left of }J.
\]
By the same logic, the equality also holds if $\phi_i(\rho_1,..., \rho_i, y)$ is to the right of $J$. If $\phi_i(\rho_1,..., \rho_i, y)$ lands inside of $J$, we can use $h(1/2)$ to bound the conditional entropy of $r_{i+1}$ since it is a binary random variable. Deduce that
\begin{align*}
H(r_{i+1}\rvert \mathcal{E}_{j,y}, r_1,..., r_i) \leq 
 & h(1/2) \sum_{\rho_1,..., \rho_i: \phi_i(\rho_1,..., \rho_i, y)\in J}\mathbb{P}\{r_1=\rho_1,..., r_i=\rho_i\rvert \mathcal{E}_{J,y}\} + \\
& h(p) \sum_{\rho_1,..., \rho_i: \phi_i(\rho_1,..., \rho_i, y)\notin J}\mathbb{P}\{r_1=\rho_1,..., r_i=\rho_i\rvert \mathcal{E}_{J,y}\}\\
= & h(1/2) \mathbb{P}\{q_{i+1}\in J \rvert \mathcal{E}_{J,y}\} + h(p) \mathbb{P}\{q_{i+1}\notin J \rvert \mathcal{E}_{J,y}\}.
\end{align*}
On the other hand, if we condition on $\mathcal{E}_{J,y}, r_1,..., r_i$ and the value of $X^*$, then not only is $q_{i+1}$ deterministic, so is $\mathds{1}\{X^*\geq q_{i+1}\}$. As a result, $r_{i+1}$ is distributed Bernoulli with success probability either $p$ or $1-p$ depending on the relative position of $q_{i+1}$ to $X^*$. Hence
\[
H(r_{i+1}\rvert X^*, \mathcal{E}_{J,y}, r_1,..., r_i) = h(p).
\]
We have shown that
\begin{align}
\nonumber& H(r_{i+1}\rvert \mathcal{E}_{J,y}, r_1,..., r_i) - H(r_{i+1}\rvert X^*, \mathcal{E}_{J,y}, r_1,..., r_i)\\
\nonumber\leq & h(1/2) \mathbb{P}\{q_{i+1}\in J \rvert \mathcal{E}_{J,y}\} + h(p) \mathbb{P}\{q_{i+1}\notin J \rvert \mathcal{E}_{J,y}\} - h(p)\\
\label{eq:h.i+1}= & (h(1/2)-h(p)) \mathbb{P}\{q_{i+1}\in J \rvert \mathcal{E}_{J,y}\}.
\end{align}
Similarly we obtain the following bound:
\begin{equation}
\label{eq:h.1}
H(r_1\rvert \mathcal{E}_{J,y}) - H(r_1\rvert X^*, \mathcal{E}_{J,y}) \leq (h(1/2)-h(p)) \mathbb{P}\{q_1\in J \rvert \mathcal{E}_{J,y}\}.
\end{equation}
Combine~\eqref{eq:chain.rule},~\eqref{eq:h.i+1} and~\eqref{eq:h.1} to finish the proof of~\eqref{eq:info.rate}.

{\bf Step 3: apply Fano's inequality to connect the learner's probability of error with the expected number of queries in $J$.}

From the continuous Fano's inequality~\cite[Proposition~2]{duchi2013distance},
\[
\mathbb{P}\left\{|\widehat{X}-X^*|>\eta\rvert \mathcal{E}_{J,y}\right\}
\geq 1-\frac{I(X^*; r_1,..., r_n\rvert \mathcal{E}_{J,y})+1}{\log \frac{|J|}{2\eta}}.
\]
Combining the above with~\eqref{eq:info.rate} yields
\begin{equation}
\label{eq:fano.step}
\sum_{i\leq n}\mathbb{P}\left\{q_i\in J\rvert \mathcal{E}_{J,y}\right\} \geq \frac{1}{c_2(p)}\left(\left(1-\mathbb{P}\left\{|\widehat{X}-X^*|>\eta\rvert \mathcal{E}_{J,y}\right\}\right)\log \frac{|J|}{2\eta} - 1\right).
\end{equation}

{\bf Step 4: integrate w.r.t. $Y$.}

Combine~\eqref{eq:disc} with~\eqref{eq:fano.step}, and integrate w.r.t. $Y$ to obtain
\begin{align*}
& \mathbb{E}\left(\text{the number of queries in } J \mid X^*\in J\right) \\
\geq & \int_0^1 \frac{1}{c_2(p)}\left(\left(1-\mathbb{P}\left\{|\widehat{X}-X^*|>\eta\rvert \mathcal{E}_{J,y}\right\}\right)\log \frac{|J|}{2\eta} - 1\right)dy\\
= & \frac{1}{c_2(p)}\left(\left(1-\mathbb{P}\{|\widehat{X}-X^*|>\eta\mid X^*\in J\}\right)\log\frac{|J|}{2\eta}-1\right).
\end{align*}
\end{proof}

\subsubsection{Proof of the lower bound in~\eqref{eq:res.mean}}\label{sec:proof.na.lower}
Suppose $\phi$ is an $\epsilon$-accurate and $(\delta,L)$-private querying strategy that submits at most $n$ queries. As in the noiseless case, we can assume WLOG that the learner always submits exactly $n$ queries by concatenating trivial queries at 0 to the end of the query sequence. 

%We also argue that it suffices to prove the lower bound for the $p=1$(noiseless response) case. To see that, suppose the strategy $\phi$ consists of functionals $f_0$(initializer) and $\phi_1,..., \phi_{n}$. If the noiseless responses $r_1,..., r_n$ are observed, then the learner can always obtain a noisy version of the responses $\widetilde{r}_1, ..., \widetilde{r}_n$ by flipping the value of each response with probability $1-p$. Define the querying strategy $\widetilde{\phi} = (\widetilde{f}_0, \widetilde{\phi}_1,..., \widetilde{\phi}_{n})$ as follows: let $\widetilde{f}_0(Y) = f_0(Y)$, and $\widetilde{\phi}_i(r_1,..., r_i, Y) = \phi_(\widetilde{r}_1,..., \widetilde{r}_i, Y)$. Then $\widetilde{\phi}$ also fulfills the requirements of $\epsilon$-accuracy and $(\delta, L)$-privacy and submits $n$ queries. Thus the optimal query complexity with noiseless responses cannot be higher than $n$.
We will split the lower bound in~\eqref{eq:res.mean} into the following two inequalities, which we will prove in order.
\begin{enumerate}[(i)]
\item 
\label{eq:na.lower.term1}
$n\geq \frac{1}{2c_2(p)}L\log\frac{\delta}{16\epsilon}$.

\item
\label{eq:na.lower.term2}
$n\geq \frac{1}{2c_2(p)}\log\frac{1}{8\epsilon}$.
\end{enumerate}

\begin{proof}[Proof of~(\ref{eq:na.lower.term1})]
Consider the adversary who adopts the proportional-sampling strategy~\cite{xu2018query}, {\it i.e.}, $\widetilde{X}$ is the sampled from the empirical distribution of all the queries. We have
\[
\mathbb{P}\{|\widetilde{X}-X^*|\leq \delta/2\}\geq \frac{1}{n} \mathbb{E}\left(\text{the number of queries in the interval }[X^*-\delta/2,X^*+\delta/2]\right).
\]
For a querying strategy to be $(\delta,L)$-private, we must have $\mathbb{P}\{|\widetilde{X}-X^*|\leq \delta/2\}\leq 1/L$. Hence
\begin{equation}
\label{eq:prop.sample}
n\geq L\mathbb{E}\left(\text{the number of queries in the interval }[X^*-\delta/2,X^*+\delta/2]\right).
\end{equation}

Divide $[0,1]$ into length $\delta/2$ subintervals labeled $J_1,..., J_{2/\delta}$ (again ignoring non-divisibility issues). Suppose $J^*$ is the subinterval that contains $X^*$, then $J^*$ is distributed uniformly on $\{J_1,..., J_{2/\delta}\}$, and it must be a subset of $[X^*-\delta/2,X^*+\delta/2]$. Therefore
\begin{align}
\nonumber& \mathbb{E}\left(\text{the number of queries in }[X^*-\delta/2,X^*+\delta/2]\right)\\
\nonumber \geq & \mathbb{E}\left(\text{the number of queries in }J^*\right)\\
\label{eq:disc}\geq & \frac{\delta}{2}\sum_{j\leq 2/\delta} \mathbb{E}\left(\text{the number of queries in } J_j \rvert X^*\in J_j\right).
\end{align}
%\nbr{I do not see immediately how you derive the first inequality. I guess to be precise, you may use the following inequality
%$\indc{q_i \in X^*-\delta/2,X^*+\delta/2] } \ge \sum_{j\leq 2/\delta} \indc{q_i \in J_j, X^* \in J_j } $ and $\mathbb{P}\left\{ X^* \in J_j \right\} = \delta/2$}
%\nb{Some intermediate steps are added to explain this inequality.}

By applying Lemma~\ref{lmm:learner.error} with $J=J_j$ and $\eta=\epsilon$, we have for all $j\in [2/\delta]$,
\[
\mathbb{E}\left(\text{the number of queries in }J_j\mid X^*\in J_j\right) \geq \frac{1}{c_2(p)} \left(\left(1-\mathbb{P}\{|\widehat{X}-X^*|>\epsilon\mid X^*\in J_j\}\right)\log\frac{\delta}{4\epsilon}-1\right).
\]
Plug into~\eqref{eq:disc} to obtain
\begin{align*}
& \mathbb{E}\left(\text{the number of queries in }[X^*-\delta/2,X^*+\delta/2]\right)\\
\geq &\frac{\delta}{2}\sum_{j\leq 2/\delta}\frac{1}{c_2(p)} \left(\left(1-\mathbb{P}\{|\widehat{X}-X^*|>\epsilon\mid X^*\in J_j\}\right)\log\frac{\delta}{4\epsilon}-1\right)\\
= & \frac{1}{c_2(p)} \left(\left(1-\mathbb{P}\{|\widehat{X}-X^*|>\epsilon\}\right)\log\frac{\delta}{4\epsilon}-1\right)
\geq \frac{1}{2c_2(p)}\log\frac{\delta}{16\epsilon},
\end{align*}
where the last inequality is from $\mathbb{P}\{|\widehat{X}-X^*|>\epsilon\}\leq \mathbb{E}|\widehat{X}-X^*|/\epsilon\leq 1/2$. We have arrived at the desired lower bound
\[
n\geq L \mathbb{E}\left(\text{the number of queries in }[X^*-\delta/2,X^*+\delta/2]\right)
\geq \frac{1}{2c_2(p)}L\log\frac{\delta}{16\epsilon}.
\]
\end{proof}

\begin{proof}[Proof of~(\ref{eq:na.lower.term2})]
Apply lemma~\ref{lmm:learner.error} with $J=[0,1]$ and $\eta=\epsilon$, we have
\[
n\geq \frac{1}{c_2(p)}\left(\left(1-\mathbb{P}\{|\widehat{X}-X^*|>\epsilon\}\right)\log\frac{1}{2\epsilon}-1\right)\geq \frac{1}{c_2(p)}\log\frac{1}{8\epsilon},
\]
where the second inequality is from $\mathbb{P}\{|\widehat{X}-X^*|>\epsilon\}\leq 1/2$.
\end{proof}

\subsubsection{Proof of the lower bound in~\eqref{eq:res.prob}}
% the Llog(\delta/\epsilon) part
Suppose $\phi$ is an $(\epsilon, M)$-accurate and $(\delta,L)$-private querying strategy that submits at most $n$ queries. Argue as in the proof of~\eqref{eq:res.mean} that we can assume WLOG that the learner always submits exactly $n$ queries. We will split the lower bound in~\eqref{eq:res.prob} into the following three inequalities and prove them in order.
\begin{enumerate}[(i)]
\item 
\label{eq:nb.lower.term1}
$n\geq \frac{L}{c_2(p)}\log \frac{\delta}{8\epsilon}.$

\item
\label{eq:nb.lower.term2}
$n\geq \frac{1}{c_2(p)}\log \frac{1}{4\epsilon}.$

\item
\label{eq:nb.lower.term3}
$n\geq \frac{1}{c_1(p)}L\log M.$
\end{enumerate}

\begin{proof}[Proof of~(\ref{eq:nb.lower.term1})]
The proof is essentially identical to the proof of~\eqref{eq:na.lower.term1} in Section~\ref{sec:proof.na.lower}. Since the query complexities $N_\mathsf{avg}()$ and $N_\mathsf{whp}()$ are defined with the same notion of $(\delta,L)$-privacy, the proportional sampling argument used in the proof of~\eqref{eq:res.mean} remains valid here. As before, by considering a proportionally-sampling adversary and partitioning $[0,1]$ into length $\delta/2$ subintervals $J_1,...,J_{2/\delta}$, we have
\begin{align*}
n\geq &L\mathbb{E}\left(\text{the number of queries in }[X^*-\delta/2,X^*+\delta/2]\right)\\
\geq & L\cdot\frac{\delta}{2}\sum_{j\leq 2/\delta}\mathbb{E}\left(\text{the number of queries in }J_j\mid X^*\in J_j\right).
\end{align*}
Applying Lemma~\ref{lmm:learner.error} with $J=J_j$ and $\eta=\epsilon/2$ yields
\[
\mathbb{E}\left(\text{the number of queries in }J_j\mid X^*\in J_j\right) \geq 
\frac{\left(1-\mathbb{P}\left\{|\widehat{X}-X^*|>\epsilon/2 \mid X^*\in J_j\right\}\right)\log\left(\frac{\delta}{2\epsilon}\right)-1}{c_2(p)}.
\]
It follows that
\begin{align*}
n\geq & L\cdot\frac{\delta}{2}\sum_{j\leq 2/\delta}\frac{\left(1-\mathbb{P}\left\{|\widehat{X}-X^*|>\epsilon/2 \mid X^*\in J_j\right\}\right)\log\left(\frac{\delta}{2\epsilon}\right)-1}{c_2(p)}\\
= & \frac{L}{c_2(p)}\left(\left(1-\mathbb{P}\left\{|\widehat{X}-X^*|>\epsilon/2 \mid X^*\in J_j\right\}\right)\log\left(\frac{\delta}{2\epsilon}\right)-1\right).
\end{align*}
From the definition of $(\epsilon,M)$-accuracy, $\mathbb{P}\{|\widehat{X}-X^*|>\epsilon/2 \}\leq 1/M\leq 1/2$. Plug in to yield~(\ref{eq:nb.lower.term1}).
\end{proof}

\begin{proof}[Proof of~(\ref{eq:nb.lower.term2})]
Apply Lemma~\ref{lmm:learner.error} with $J=[0,1]$ and $\eta = \epsilon/2$. We have
\[
n\geq \frac{1}{c_2(p)}\left(\left(1-\mathbb{P}\left\{|\widehat{X}-X^*|>\epsilon/2 \right\}\right)\log(1/\epsilon)-1\right).
\]
Combine with $\mathbb{P}\{|\widehat{X}-X^*|>\epsilon/2 \}\leq 1/2$ to yield~(\ref{eq:nb.lower.term2}).
\end{proof}
%\nbr{Why in the above, we have $ \log\left( \frac{\delta}{2\epsilon}\right)$ instead of 
%$\log\left(\frac{\delta}{4\epsilon } \right) $ and $\log(1/\epsilon)$ instead of $\log(1/2\epsilon)$? I think we sacrifice a factor of $2$ when using
%continuous Fano's inequality. Also, if we does not want to keep the $1-1/M$ factor and assume $M\ge 2$, then the proofs
%of $(i)$ and $(ii)$ are exactly the same as before, right?}
%\nb{The lower bound proofs of the $\Omega(\log(\delta/\epsilon))$ and $\Omega(\log(1/\epsilon))$ terms are rewritten. The proofs of all four inequalities now call a unified lemma (Lemma~\ref{lmm:learner.error}) that connects the number of queries in a fixed interval with the learner's probability of error, conditional on $X^*$ being inside of this interval.}

\begin{proof}[Proof of~(\ref{eq:nb.lower.term3})]
As we have argued in the first two steps of the lower bound proof of~(\ref{eq:na.lower.term1}) in Section~\ref{sec:proof.na.lower}, 
\begin{equation}
\label{eq:prop.sample.3}
n\geq L\cdot \frac{\delta}{2}\sum_{j\leq 2/\delta} \int_0^1 \mathbb{E}(\text{the number of queries in }J_j\rvert \mathcal{E}_{j,y})dy.
\end{equation}
Fano's inequality is no longer sufficient to yield a $n=\Omega(\log M)$ lower bound on the expected number of queries in $J_j$. Instead our proof strategy is to reduce the estimation problem to a binary hypothesis test. 

Denote $J_j=[a_j,b_j]$ with midpoint $m_j=(a_j+b_j)/2$. Let $I_0=[a_j,m_j-\epsilon/2)$ and $I_1=[m_j+\epsilon/2, b_j]$ be two subintervals of $J_j$ that are $\epsilon$ apart. Then any learner that achieves $|\widehat{X}-X^*|\leq \epsilon/2$ must also be able to test between the two hypotheses $X^*\in I_0$ and $X^*\in I_1$. Indeed,
\begin{align}
\nonumber& \mathbb{P}\left\{|\widehat{X}-X^*|>\epsilon/2 \rvert \mathcal{E}_{j,y}\right\}\\
\nonumber\geq & \frac{|I_0|}{|J_j|}\mathbb{P}\left\{\widehat{X}\geq m_j\rvert X^*\in I_0, Y=y\right\} + \frac{|I_1|}{|J_j|}\mathbb{P}\left\{\widehat{X}<m_j\rvert X^*\in I_1, Y=y\right\}\\
\label{eq:multiple.test}= & \frac{2|I_0|}{|J_j|}\left(\tfrac{1}{2}\mathbb{P}\left\{\widehat{X}\geq m_j\rvert X^*\in I_0, Y=y\right\} + \tfrac{1}{2}\mathbb{P}\left\{\widehat{X}< m_j \rvert X^*\in I_1, Y=y\right\}\right),
\end{align}
where the last equality comes from $|I_0|=|I_1|$. The term in the parentheses is the average error probability of the test $\widehat{T}=\mathds{1}\{\widehat{X}\geq m_j\}$ under the uniform prior on the hypotheses. Furthermore, it can be viewed as an average error probability of a family of simple tests by symmetry of $I_0$ and $I_1$:
\begin{align}
\nonumber& \tfrac{1}{2}\mathbb{P}\left\{\widehat{X}\geq m_j\rvert X^*\in I_0, Y=y\right\} + \tfrac{1}{2}\mathbb{P}\left\{\widehat{X}< m_j \rvert X^*\in I_1, Y=y\right\}\\
\nonumber= & \tfrac{1}{2}\int \frac{\mathds{1}\{x\in I_0\}}{|I_0|} \mathbb{P}\left\{\widehat{X}\geq m_j\rvert X^*=x, Y=y\right\} dx + \tfrac{1}{2}\int \frac{\mathds{1}\{x\in I_1\}}{|I_1|} \mathbb{P}\left\{\widehat{X}< m_j\rvert X^*=x, Y=y\right\} dx\\
\nonumber= & \int \frac{\mathds{1}\{x\in I_0\}}{|I_0|}\left(\tfrac{1}{2}\mathbb{P}\left\{\widehat{X}\geq m_j\rvert X^*=x, Y=y\right\} + \tfrac{1}{2}\mathbb{P}\left\{\widehat{X}< m_j\rvert X^*=2m_j - x, Y=y\right\}\right)dx\\
\label{eq:simple.test}\geq & \int \frac{\mathds{1}\{x\in I_0\}}{|I_0|}\inf_{\widehat{T}}\left(\tfrac{1}{2}\mathbb{P}\left\{\widehat{T}(r^{(n)})=1\rvert X^*=x, Y=y\right\} + \tfrac{1}{2}\mathbb{P}\left\{\widehat{T}(r^{(n)})=0\rvert X^*=2m_j - x, Y=y\right\}\right)dx.
\end{align}
We have reduced the problem to lower bounding the minimum error probability of the simple test
\[
H_0: X^*=x\;\;\;\text{against}\;\;\; H_1: X^*=2m_j-x.
\]
From~\cite[Eq~(49)]{kailath1967divergence} we have the lower bound
\begin{equation}
\label{eq:two.point.lower}
\inf_{\widehat{T}}\left(\tfrac{1}{2}\mathbb{P}\left\{\widehat{T}(r^{(n)})=1\rvert H_0, Y=y\right\} + \tfrac{1}{2}\mathbb{P}\left\{\widehat{T}(r^{(n)})=0\rvert H_1, Y=y\right\}\right)\geq \frac{\rho^2}{4},
\end{equation}
where $\rho$ is the Bhattacharyya coefficient:
\[
\rho = \sum_{r^{(n)}\in \{0,1\}^n} \sqrt{\mathbb{P}\{r^{(n)}\rvert H_0,Y=y\}\mathbb{P}\{r^{(n)}\rvert H_1,Y=y\}} = \mathbb{E}\left(2^{\Lambda/2}\rvert H_0,Y=y\right)
\]
with
\begin{align*}
\Lambda = & \log \frac{\mathbb{P}\{r^{(n)}\rvert H_1,Y=y\}}{\mathbb{P}\{r^{(n)}\rvert H_0,Y=y\}}\\
= & \log \frac{\prod_{i\leq n}\left(\mathds{1}\{2m_j-x\geq q_i\}p^{r_i}(1-p)^{1-r_i} + \mathds{1}\{2m_j-x< q_i\}p^{1-r_i}(1-p)^{r_i}\right)}{\prod_{i\leq n}\left(\mathds{1}\{x\geq q_i\}p^{r_i}(1-p)^{1-r_i} + \mathds{1}\{x< q_i\}p^{1-r_i} (1-p)^{r_i}\right)}\\
= & \sum_{i=1}^n \mathds{1}\{q_i\in (x, 2m_j-x]\}\left(r_i \log \frac{p}{1-p} + (1-r_i)\log\frac{1-p}{p}\right).
\end{align*}
By Jensen's inequality,
\begin{align*}
2\log \rho\geq \mathbb{E}\left(\Lambda\rvert H_0,Y=y\right)
&= \sum_{i\leq N}\mathbb{P}\left\{q_i\in (x, 2m_j-x]\rvert H_0,Y=y\right\}\left((1-p)\log\frac{p}{1-p} + p\log\frac{1-p}{p}\right) \\
&\ge  -c_1(p)\mathbb{E}(\text{the number of queries in }J_j\rvert H_0,Y=y)
\end{align*}
where $c_1(p) = p\log\frac{p}{1-p}+(1-p)\log\frac{1-p}{p}$ is always nonnegative. We have arrived at
\[
\rho^2\geq 2^{\mathbb{E}\left(\Lambda\rvert H_0,Y=y\right)} \geq 2^{-c_1(p)\mathbb{E}(\text{the number of queries in }J_j\rvert H_0,Y=y)}.
\]
Together with~\eqref{eq:multiple.test},~\eqref{eq:simple.test} and~\eqref{eq:two.point.lower} we have
\begin{align*}
& \mathbb{P}\left\{|\widehat{X}-X^*|>\epsilon/2 \rvert \mathcal{E}_{j,y}\right\}\\
\geq & \frac{|I_0|}{2|J_j|}\int \frac{\mathds{1}\{x\in I_0\}}{|I_0|}2^{-c_1(p)\mathbb{E}(\text{the number of queries in }J_j\rvert X^*=x,Y=y)}dx\\
\geq & \frac{|I_0|}{2|J_j|}2^{-c_1(p)\mathbb{E}(\text{the number of queries in }J_j\rvert X^*\in I_0,Y=y)},
\end{align*}
where the last inequality follows from Jensen's inequality. By symmetry the same inequality holds for $I_1$. Therefore
\begin{align*}
& \mathbb{E}(\text{the number of queries in }J_j\rvert \mathcal{E}_{j,y})\\
\geq & \mathbb{P}\{X^*\in I_0\rvert \mathcal{E}_{j,y}\}\mathbb{E}(\text{the number of queries in }J_j\rvert X^*\in I_0,Y=y)\\
& + \mathbb{P}\{X^*\in I_1\rvert \mathcal{E}_{j,y}\}\mathbb{E}(\text{the number of queries in }J_j\rvert X^*\in I_1,Y=y)\\
\geq & \frac{2|I_0|}{|J_j|}\cdot\frac{1}{c_1(p)}\left(-\log \mathbb{P}\left\{|\widehat{X}-X^*|>\frac{\epsilon}{2}\rvert \mathcal{E}_{j,y}\right\} + \log\frac{|I_0|}{2|J_j|}\right).
\end{align*}
Recall that $|J_j|=\delta/2$ and $|I_0|=|J_j|/2-\epsilon/2$. We have $|I_0|/|J_j|\geq 1/4$ from the assumption $\delta\geq 4\epsilon$. Combine the inequality above with~\eqref{eq:prop.sample.3} to obtain
\begin{align*}
n\geq & -\frac{L}{2c_1(p)}\cdot\frac{\delta}{2}\sum_{j\leq 2/\delta}\int_0^1\log\left(8\mathbb{P}\{|\widehat{X}-X^*|>\epsilon/2\rvert \mathcal{E}_{j,y}\}\right)\\
\geq & -\frac{L}{2c_1(p)}\log\left( \frac{\delta}{2}\sum_{j\leq 2/\delta}\int_0^1 8\mathbb{P}\{|\widehat{X}-X^*|>\epsilon/2\rvert \mathcal{E}_{j,y}\}dy\right)\\
= & -\frac{L}{2c_1(p)}\log\left(8\mathbb{P}\{|\widehat{X}-X^*|>\epsilon/2\}\right) \geq \frac{L}{2c_1(p)}\log\frac{M}{8},
\end{align*}
where the second inequality is due to Jensen's inequality and the last inequality follows from the definition of $(\epsilon,M)$-accuracy.
\end{proof}

\end{appendices}

\bibliographystyle{plain}
\bibliography{bibliography}

\end{document}